\documentclass{article}
\pdfoutput=1

\usepackage[final,nonatbib]{neurips_2022}

\usepackage[utf8]{inputenc} 
\usepackage[T1]{fontenc}    
\usepackage{hyperref}       
\usepackage{url}            
\usepackage{booktabs}       
\usepackage{amsmath}
\usepackage{amsthm}
\usepackage{amssymb}
\usepackage{amsfonts}       
\usepackage{nicefrac}       
\usepackage{microtype}      
\usepackage{xcolor}         
\usepackage{cleveref}
\usepackage{bm}
\usepackage{multirow}
\usepackage{balance}
\usepackage{graphics}
\usepackage{graphicx}
\usepackage{float}
\usepackage{subfigure}
\usepackage{epstopdf}
\usepackage{proof}
\usepackage{algorithmic}
\usepackage{algorithm}
\usepackage{multicol}
\usepackage{bbm}
\usepackage{thm-restate}
\usepackage{caption}
\usepackage{pifont}
\usepackage{titlesec}
\usepackage{titletoc}
\usepackage{lipsum}
\usepackage{atbegshi}
\usepackage{etoolbox}

\newcommand{\xmark}{\ding{55}}

\definecolor{ballblue}{rgb}{0.13, 0.67, 0.8}
\definecolor{lightseagreen}{rgb}{0.13, 0.7, 0.67}
\definecolor{darkturq}{HTML}{00CED1}
\definecolor{org}{HTML}{F8A145}
\definecolor{blu}{HTML}{63ACE5}
\definecolor{c1}{HTML}{41B3A3}
\definecolor{c2}{HTML}{3500D3}
\definecolor{softred}{HTML}{FE8A71}
\definecolor{softblue}{HTML}{63ACE5}

\hypersetup
{
    colorlinks=true,
    linkcolor=blue,
    citecolor=darkturq 	 	
}

\title{Exploring the Algorithm-Dependent Generalization of AUPRC Optimization with List Stability}

\author{\parbox{13cm}
  {\centering
    {\large \quad\quad Peisong Wen$^{1,2}$ \ \ \ \ \ \ \ \ \  Qianqian Xu$^{1}$\thanks{Corresponding authors.} \ \ \ \ \ \ \ \ \ 
    Zhiyong Yang$^{2}$ \ \ \ \ \ \ \ \ \ \ \  \\ Yuan He$^{3}$ \ \ \ \ \ \ \ \ \ \ \ \ \ Qingming Huang$^{1,2,4,5*}$ }\\
    {\normalsize
    $^1$ Key Lab of Intell. Info. Process., Inst. of Comput. Tech., CAS\\
    $^2$ School of Computer Science and Tech., University of Chinese Academy of Sciences\\
    $^3$ Alibaba Group\\
    $^4$ BDKM, University of Chinese Academy of Sciences\\
    $^5$ Peng Cheng Laboratory\\
    }
    {\tt\small \{wenpeisong20z,xuqianqian\}@ict.ac.cn \quad\quad \\ \{yangzhiyong21,qmhuang\}@ucas.ac.cn ~~ heyuan.hy@alibaba-inc.com 
    }
  }
}

\newtheorem{defi}{Definition}

\newtheorem{prop}{Proposition}

\newtheorem{lem}{Lemma}
\newtheorem{rem}{Remark}
\newtheorem{asm}{Assumption}

\newcommand{\first}[1]{\textbf{\color{softred}#1}}
\newcommand{\second}[1]{\textbf{\color{softblue}#1}}
\newcommand{\Eqref}[1]{Eq.~(\ref{#1})}
\newcommand{\Tbref}[1]{Tab.~\ref{#1}}
\newcommand{\Fgref}[1]{Fig.~\ref{#1}}
\newcommand{\Propref}[1]{Prop.~\ref{#1}}
\newcommand{\Thmref}[1]{Thm.~\ref{#1}}
\newcommand{\Lemref}[1]{Lem.~\ref{#1}}
\newcommand{\Defiref}[1]{Def.~\ref{#1}}
\newcommand{\Asmpref}[1]{Asmp.~\ref{#1}}
\newcommand{\Algref}[1]{Alg.~\ref{#1}}
\renewcommand{\cref}[1]{Sec.~\ref{#1}}

\def \xi {\bm{x}_i}

\def \yi {y_i}

\def \prob {\mathbb{P}}

\def \xi {\bm{x}_i}

\def \xpi {\bm{x}^+_{i}}

\def \xni {\bm{x}^-_{i}}

\def \Np {N^{+}}
\def \Nn {N^{-}}
\def \np {n^+}
\def \nn {n^-}

\def \w {\bm{w}}
\def \z {\bm{z}}

\def \zp {\bm{z}^+}
\def \zn {\bm{z}^-}
\def \Expt {\mathbb{E}}

\def \S {\mathcal{S}}
\def \Sp {\mathcal{S}^+}
\def \Sn {\mathcal{S}^-}

\def \eg {\textit{e.g.}}
\def \ie {\textit{i.e.}}
\def \etal {\textit{et al.}}

\begin{document}

\makeatletter
\newcommand{\discardpages}[1]{
  \xdef\discard@pages{#1}
  \AtBeginShipout{
    \renewcommand*{\do}[1]{
      \ifnum\value{page}=##1\relax%
        \AtBeginShipoutDiscard
        \gdef\do####1{}
      \fi%
    }%
    \expandafter\docsvlist\expandafter{\discard@pages}
  }%
}
\newif\ifkeeppage
\newcommand{\keeppages}[1]{
  \xdef\keep@pages{#1}
  \AtBeginShipout{
    \keeppagefalse%
    \renewcommand*{\do}[1]{
      \ifnum\value{page}=##1\relax%
        \keeppagetrue
        \gdef\do####1{}
      \fi%
    }%
    \expandafter\docsvlist\expandafter{\keep@pages}
    \ifkeeppage\else\AtBeginShipoutDiscard\fi
  }%
}
\makeatother

\maketitle

\begin{abstract}
    Stochastic optimization of the Area Under the Precision-Recall Curve (AUPRC) is a crucial problem for machine learning. Although various algorithms have been extensively studied for AUPRC optimization, the generalization is only guaranteed in the multi-query case. In this work, we present the first trial in the single-query generalization of stochastic AUPRC optimization. For sharper generalization bounds, we focus on algorithm-dependent generalization. There are both algorithmic and theoretical obstacles to our destination. From an algorithmic perspective, we notice that the majority of existing stochastic estimators are biased only when the sampling strategy is biased, and is leave-one-out unstable due to the non-decomposability. To address these issues, we propose a sampling-rate-invariant unbiased stochastic estimator with superior stability. On top of this, the AUPRC optimization is formulated as a composition optimization problem, and a stochastic algorithm is proposed to solve this problem. From a theoretical perspective, standard techniques of the algorithm-dependent generalization analysis cannot be directly applied to such a listwise compositional optimization problem. To fill this gap, we extend the model stability from instancewise losses to listwise losses and bridge the corresponding generalization and stability. Additionally, we construct state transition matrices to describe the recurrence of the stability, and simplify calculations by matrix spectrum. Practically, experimental results on three image retrieval datasets on speak to the effectiveness and soundness of our framework.
\end{abstract}

\section{Introduction}
Area Under the Precision-Recall Curve (AUPRC) is a widely used metric in the machine learning community, especially in learning to rank, which effectively measures the trade-off between precision and recall of a ranking model. Compared with threshold-specified metrics like accuracy and recall@k, AUPRC reflects a more comprehensive performance by capturing all possible thresholds. In addition, literature has shown that AUPRC is insensitive toward data distributions \cite{davis2006relationship}, making it adaptable to largely skewed data. Benefiting from these appealing properties, AUPRC has become one of the standard metrics in various applications, \eg, retrieval \cite{qin2010general, revaud2019learning, engilberge2019sodeep, li2021rethinking}, object detection \cite{mohapatra2014efficient, oksuz2020ranking, chen2020ap}, medical diagnosis \cite{ozenne2015precision, kwon2018algorithm}, and recommendation systems \cite{chen2017sampling, wang2019adversarial, bao2019collaborative, tran2019improving, bao2022rethinking}.

Over the past decades, the importance of AUPRC has prompted extensive researches on direct AUPRC optimization. Early work focuses on full-batch optimization \cite{mohapatra2014efficient, metzler2005markov, goadrich2006gleaner}. However, in the era of deep learning, the rapidly growing scale of models and data makes these full-batch algorithms infeasible. Therefore, in recent years, it has raised an increasing favor of the stochastic AUPRC optimization \cite{brown2020smooth, cakir2019deep, henderson2016end, mohapatra2018efficient}. Since AUPRC optimization is a stochastic dependent compositional optimization problem, general convergence rates are infeasible for AUPRC optimization. To fill this gap, \cite{qi2021stochastic, wang2021momentum, wang2022finite} provide AUPRC optimization algorithms with provable convergence. See Appendix {\color{blue} A}
for more on related work. 

Despite the promoting performance of these methods in various scenarios, the generalization of AUPRC optimization algorithms is still an open problem. Some studies \cite{chen2009ranking, song2016training} provide provable generalization for AUPRC optimization in information retrieval. In this scene, a dataset consists of multiple queries, where each query corresponds to a set of positive and negative samples. However, these results require sufficient queries to ensure small generalization errors, but leave the single-query case alone, \ie, \textit{whether the generalization error tends to zero with the length of a single-query increasing is still unclear}. This limits the adaptation scope of these methods.
To fill this gap, in this paper we aim to \emph{\textbf{design a stochastic optimization framework for AUPRC with a provable algorithm-dependent generalization performance in the single-query case.}}

The target is challenging in three aspects:
\textbf{(a)} Most AUPRC stochastic estimators are biased with a biased sampling rate. Moreover, due to the non-decomposability, outputs of existing algorithms
might change a lot with slight changes in the training data, which is called \textit{leave-one-out unstable} in this paper. Such an unstability is harmful to the generalization. 
\textbf{(b)} The standard framework to analyze the algorithm-dependent generalization requires the objective function to be expressed as a sum of instancewise terms, while AUPRC involves a listwise loss.
\textbf{(c)} The stochastic optimization of AUPRC is a two-level compositional optimization problem, which brings more complicated proofs of the stability.

In search of a solution to \textbf{(a)}, we propose a sampling-rate-invariant asymptotically unbiased stochastic estimator based on a reformulation of AUPRC. Notably, to ensure the stability \cite{hardt2016train, lei2020fine, lei2021generalization,lei2021stability} of the estimator, the objective is formulated as a two-level compositional problem by introducing an auxiliary vector for the ranking estimation. Error analysis further supports the feasibility of our method, and inspires us to add a semi-variance regularization term.
To solve this problem, we propose an algorithm with provable convergence that combines stochastic gradient descent (SGD), linear interpolation and exponential moving average.

Facing challenge \textbf{(b)}, we extend instancewise model stability to listwise model stability, and correspondingly put forward the generalization via stability of listwise problems. On top of this, we bridge the generalization of AUPRC and the stability of the proposed optimization algorithm. 

As for challenge \textbf{(c)}, since the variables to be optimized are typically updated alternately in the compositional optimization problem, we propose state transition matrices of these variables, and simplify the calculations of the stability with matrix spectrum.


In a nutshell, the main contributions of this paper are summarized as follows:
\begin{itemize}
    \item Algorithmically, a stochastic learning algorithm is proposed for AUPRC optimization. The core of the proposed algorithm is a stochastic estimator which is sampling-rate-invariant asymptotically unbiased. 
    \item Theoretically, we present the first trial on the algorithm-dependent generalization of stochastic AUPRC optimization. To the best of our knowledge, it is also the first work to analyze the stability of stochastic compositional optimization problems.
    \item Technically, we extend the concept of the stability and generalization guarantee to listwise non-convex losses. Then we simplify the stability analysis of compositional objective by matrix spectrum. These techniques might be instructive for other complicated metrics.
\end{itemize}

\section{Problem Formulation}
\label{sec:setting}
\subsection{Preliminaries on AUPRC}
\label{setting:prel}
\textbf{Notations.} Consider a set of $N$ examples $\S = \{(\xi, \yi)\}_{i=1}^N$ independently drawn from a sample space $\mathcal{D} = \mathcal{X}\times\mathcal{Y}$, where $\mathcal{X}$ is the input space and $\mathcal{Y} = \{-1,1\}$ is the label space. For sake of the presentation, denote the set of positive examples of $\S$ as $\Sp = \{\xpi\}_{i=1}^{\Np}$, and similarly the set of negative examples is denoted as $\Sn = \{\xni\}_{i=1}^{\Nn}$, where $\Np = |\Sp|, \Nn = |\Sn|$. With a slight abuse of notation, we also denote $\S = \Sp \cup \Sn$ if there is no ambiguity. Generally, we assume that the dataset is sufficiently large, such that $\Np / (\Np + \Nn) = \mathbb{P}(y=1) := \pi$. Our target is to learn a score function $h_{\w}: \mathcal{X} \mapsto \mathbb{R}$ with parameters $\w\in \Omega \subseteq \mathbb{R}^{d}$, such that the scores of positive examples are higher than negative examples. Furthermore, when appling the score function to a dataset $\S\in\mathcal{X}^N$, we denote $h_{\w}: \mathcal{X}^N\mapsto\mathbb{R}^N$, where the $k$-th element of $h_{\w}(\S)$ has the top-$k$ values of $\{h_{\w}(\bm{x}) | \bm{x}\in \S\}$. Denote the asymptotic upper bound on complexity as $\mathcal{O}$, and denote asymptotically equivalent as $\asymp$.

In this work, our main interest is to optimize a score function in the view of AUPRC:
\begin{equation}
    \begin{aligned}
        \text{AUPRC}(\w;\mathcal{D}) &= \int_{0}^1 \prob(y=1 | h_{\w}(\bm{x})\geq c) ~d~ \prob(h_{\w}(\bm{x})\geq c | y=1) \\
        &= \int_{0}^1 \frac{\pi TPR(c)}{\pi TPR(c) + (1-\pi)FPR(c)} ~d~ \prob(h_{\w}(\bm{x})\geq c | y=1),
    \end{aligned}
\end{equation}
where $(\bm{x},y)\sim \mathcal{D}$, $c$ refers to a threshold, and $TPR(c) = \prob(h_{\w}(\bm{x})\geq c | y=1), FPR(c) = \prob(h_{\w}(\bm{x})\geq c | y=0)$. For a finite set $\S$, AUPRC is typically approximated by replacing the distribution function $\prob(h_{\w}(\bm{x})\geq c | y=1)$ with its empirical cumulative distribution function \cite{boyd2013area, clemenccon2009nonparametric}: 
\begin{equation}
    \begin{aligned}
        \widehat{\text{AUPRC}}(\w;\S) &= \mathop{\hat{\Expt}}\limits_{\bm{x}^+\sim\Sp}\left[\frac{\pi \widehat{TPR}(h_{\w}(\bm{x}^+))}{\pi \widehat{TPR}(h_{\w}(\bm{x}^+)) + (1-\pi)\widehat{FPR}(h_{\w}(\bm{x}^+))}\right], \\
    \end{aligned}
\end{equation}
where $\widehat{TPR}(c) = \hat{\Expt}_{\bm{x}\sim\Sp}\left[\ell_{0,1}(c - h_{\w}(\bm{x}))\right], \widehat{FPR}(c) = \hat{\Expt}_{\bm{x}\sim\Sn}\left[\ell_{0,1}(c - h_{\w}(\bm{x}))\right]$, $\ell_{0,1}(x) = 1$ if $x\leq 0$ or $\ell_{0,1}(x) = 0$ otherwise. It has been shown that $\widehat{\text{AUPRC}}$ is an unbiased estimator when $\Np / (\Np + \Nn) \rightarrow \pi$ and $N \rightarrow \infty$ \cite{boyd2013area}. With the above estimation, we have the following optimization objective:
\begin{equation}
\label{eq:obj1}
    \begin{aligned}
        \min_{\w} ~~~ \widehat{\text{AUPRC}}^{\downarrow}(\w;\S) = 1 - \widehat{\text{AUPRC}}(\w;\S)
        =\mathop{\hat{\Expt}}\limits_{\bm{x}^+\sim\Sp}\left[
            \sigma\left(
            {\color{blue}\frac{1-\pi}{\pi}} \cdot
            \frac{
                \widehat{FPR}(h_{w}(\bm{x}^+))
            }{
                \widehat{TPR}(h_{w}(\bm{x}^+))
            }
            \right)
        \right],
    \end{aligned}
\end{equation}
where $\sigma(x) = x / (1+x)$ is concave and monotonically increasing. To make it smooth, surrogate losses $\ell_1, \ell_2$ are used to replace $\ell_{0,1}$ in $\widehat{FPR}$ and $\widehat{TPR}$ respectively, yielding the following surrogate objective:
\begin{equation}
\label{eq:obj1}
    \begin{aligned}
        \min_{\w} ~~~ f(\w;\S)
        =\mathop{\hat{\Expt}}\limits_{\bm{x}^+\sim\Sp}\left[
            \sigma\left(
            {\color{blue}\frac{1-\pi}{\pi}} \cdot
            \frac{
                \widehat{FPR}(h_{w}(\bm{x}^+);\ell_1)
            }{
                \widehat{TPR}(h_{w}(\bm{x}^+);\ell_2)
            }
            \right)
        \right],
    \end{aligned}
\end{equation}
where $\widehat{TPR}(c;\ell_2) = \hat{\Expt}_{\bm{x}\sim\Sp}\left[\ell_2(c - h_{\w}(\bm{x}))\right], \widehat{FPR}(c;\ell_1) = \hat{\Expt}_{\bm{x}\sim\Sn}\left[\ell_{1}(c - h_{\w}(\bm{x}))\right]$.
Specifically, when $\Np / (\Np + \Nn) = \pi$, it is equivalent to another commonly used formulation \textbf{Average Precision (AP) Loss}:
\begin{equation}
\label{eq:ap}
    \begin{aligned}
        \widehat{\text{AP}}^{\downarrow}(\w;\S) 
        =\mathop{\hat{\Expt}}\limits_{\bm{x}^+\sim\Sp}\left[
            \sigma\left(
            \frac{
                \sum_{\bm{x}\sim\Sn}\left[\ell_{1}(h_{\w}(\bm{x}^+) - h_{\w}(\bm{x}))\right]
            }{
                \sum_{\bm{x}\sim\Sp}\left[\ell_{2}(h_{\w}(\bm{x}^+) - h_{\w}(\bm{x}))\right]
            }
            \right)
        \right].
    \end{aligned}
\end{equation}

\subsection{Stochastic Learning of AUPRC}
\label{setting:stoc_learning}
Under the stochastic learning framework for instancewise losses, the empirical risk $F(\w;\S)$ is expressed as a sum of instancewise losses: $F(\w;\S) = \frac{1}{N} \sum_{\bm{x}\sim\S} \hat{f}(\w;\bm{x})$,
where $\hat{f}(\w;\bm{x})$ is the stochastic estimator of $F(\w;\S)$. Different from instancewise losses, listwise losses like AUPRC require a batch of samples to calculate the stochastic estimator. Specifically, at each step, a subset of $\S$: $\z = \zp \cup \zn$ is randomly drawn, where $\zp$ consists of $\np$ positive examples and $\zn$ consists of $\nn$ negative examples. Then a stochastic estimator of the loss function, denoted as $\hat{f}(\w;\z)$, is computed with $\z$. Similar to the instancewise case, we consider a variant of the empirical/population AUPRC risks as approximations, which is a sum of stochastic losses w.r.t. all posible $\z$:
\begin{equation}
    F(\w;\S) = \frac{1}{M} \sum_{\z} \hat{f}(\w;\z), ~~~~F(\w) = \Expt_{\S\sim\mathcal{D}}[F(\w;\S)],
\end{equation}
where $M$ is the number of all posible $\z$. Unfortunately, due to the non-decomposability of the empirical AUPRC risk $f(\w;\S)$, it is tackle to determine the approximation errors between $F(\w;\S)$ and $f(\w;\S)$ in general. Nonetheless, in \cref{method:error_analysis} we argue that by selecting proper $\hat{f}(\w;\z)$, $F(\w;\S)$ can be asymptotically unbiased estimator of $f(\w;\S)$, which naturally makes $F(\w)$ an asymptotically unbiased estimator of $1 - \text{AUPRC}$. \textbf{In this case, $\hat{f}$ is said to be an asymptotically unbiased stochastic estimator. Moreover, if the unbiasedness holds under biased sampling rate, it is said to be sampling-rate-invariant asymptotically unbiased.}

\section{Asymptotically Unbiased Stochastic AUPRC Optimization}
\label{sec:method}
In this section, we will present our SGD-style stochastic optimization algorithm of AUPRC. In \cref{method:sur_loss}, we propose surrogate losses to make the objective function differentiable. In \cref{method:alg}, we present details of the proposed stochastic estimator and the corresponding optimization algorithm. Analyses on approximation errors are provided in \cref{method:error_analysis}.

\subsection{Differentiable Surrogate Losses}
\label{method:sur_loss}
Since $\ell_{0,1}$ appears in both the numerator and denominator of \Eqref{eq:obj1}, simply implementing $\ell_1,\ell_2$ with a single function \cite{qin2008query,brown2020smooth,qi2021stochastic} will bring difficulty to analyze the relationship between $\widehat{\text{AUPRC}}^{\downarrow}(\w;\S)$ and $f(\w;\S)$.
This motivates us to choose $\ell_{1} \geq \ell_{0,1}, \ell_{2} \leq \ell_{0,1}$, such that $\widehat{\text{AUPRC}}^{\downarrow}(\w;\S) \leq f(\w;\S)$, thus the original empirical risk could be optimized by minimizing its upper bound $f(\w;\S)$. Concretely, $\ell_1$ and $\ell_2$ are defined as the one-side Huber loss and the one-side sigmoid loss:
\begin{equation}
\label{eq:sur_loss}
    \ell_{1}(x) = \left\{
        \begin{array}{cc}
        \begin{aligned}
            &-2x/\tau_1, & x < 0, \\
            &(1 - x / \tau_1)^2, & 0 \leq x < \tau_1, \\
            & 0, & x \geq \tau_1.
        \end{aligned} 
        \end{array}
    \right.
    ~~~~~~~~~\ell_{2}(x) = \left\{
        \begin{array}{cc}
        \begin{aligned}
            &\frac{\exp(-x / \tau_2) - 1}{\exp(-x / \tau_2) + 1}, & x < 0, \\
            & 0, & x \geq 0.
        \end{aligned} 
        \end{array}
    \right.
\end{equation}
Here $\tau_1, \tau_2 > 0$ are hyperparameters. $\ell_1$ is convex and decreasing, which ensures the gap between positive-negative pairs is effectively optimized. Additionally, compared with the square loss and the exponential loss, $\ell_1$ is more robust to noises. $\ell_2$ is Lipschitz continuous, and $\ell_2\rightarrow \ell_{0,1}$ with $\tau_2 \rightarrow 0$.

\subsection{Stochastic Estimator of AUPRC}
\label{method:alg}
The key to a stochastic learning framework is the design of the stochastic estimator (or the corresponding gradients), \ie, $\hat{f}(\w;\z)$. Existing methods \cite{brown2020smooth,xia2008listwise, cakir2019deep} implement it with $\widehat{\text{AP}}^{\downarrow}(\w;\z)$ (\Eqref{eq:ap}), which might suffer from two problems:
\begin{itemize}
    \item[\textbf{(P1)}] Comparing \Eqref{eq:obj1} and \Eqref{eq:ap}, it can be seen that only when $\np / (\np + \nn) \rightarrow \pi$, $\widehat{\text{AP}}^{\downarrow}$ is an asymptotically unbiased estimator. However, it is hardly satisfied since the sampling strategy is usually biased in practice.
    \item[\textbf{(P2)}] Each term in the summation of $\widehat{\text{AP}}^{\downarrow}$ is related to all instances of a batch, leading to weak leave-one-out stability, \ie, changing one instance might result in a relatively large fluctuation in the stochastic gradient, especially when changing a positive example.
\end{itemize}
To tackle the above problems, we first substitute $\widehat{FPR}(h_{w}(\bm{x}^+);\ell_1)$ with $\mathop{\hat{\Expt}}_{\bm{x}\sim\zn}[\ell_1(h_{\w}(\bm{x}^+) - h_{\w}(\bm{x}))]$, and then introduce an auxiliary vector $\bm{v}\in \mathbb{R}^{\Np}$ to estimate $\widehat{TPR}$. Formally, we propose the following batch-based estimator:
\begin{equation}
\label{eq:estimator}
    \hat{f}(\w;\z) = \hat{f}(\w;\z,\bm{v}) = \mathop{\hat{\Expt}}\limits_{\bm{x}^+\sim\zp}\left[
        \sigma\left(
        {\color{blue}\frac{1-\pi}{\pi}} \cdot
        \frac{
            \mathop{\hat{\Expt}}_{\bm{x}\sim\zn}[\ell_1(h_{\w}(\bm{x}^+) - h_{\w}(\bm{x}))]
        }{
            \mathop{\hat{\Expt}}_{v\sim\bm{v}}[\ell_2(h_{\w}(\bm{x}^+) - v)]
        }
        \right)
    \right].
\end{equation}
Such an estimator enjoys two advantages: in terms of \textbf{P1}, it is asymptotically unbiased regardless of the sampling rate (see \cref{method:error_analysis} for detailed discussions);
as for \textbf{P2}, we use $\bm{v}$ to substitute $h_{\w}(\Sp)$, such that each positive example in a mini-batch only appears in one term. Ideally, it can be considered as using all positive examples in the dataset to estimate $\widehat{TPR}$ instead of that from a mini-batch. With the fact that $\nn \gg \np$, this makes the corresponding algorithm more stable. Moreover, based on the model stability, generalization bounds are available (see \cref{sec:gen}).



\subsection{Analyses on Approximation Errors}
\label{method:error_analysis}
In this subsection, we analyze errors from two approximations in the above algorithm: \textbf{1)} the gap between $F(\w;\S)$ and the true AUPRC loss; \textbf{2)} the gap between the interpolated scores $\phi(h_{\w}(\zp))$ and the true scores $h_{\w}(\Sp)$. Proofs are provided in Appendix {\color{blue} B.1}.

Denote $\pi = \Np / (\Np + \Nn)$ and $\pi_0 = \np / (\np + \nn)$. We would like to show that for all $\w\in \Omega$, $\Expt_{\z}[\hat{f}(\w;\z)]$ is an unbiased estimator when $n \rightarrow \infty$, no matter how $\pi_0$ is chosen, while for $\Expt_{\z}[\widehat{\text{AP}}^{\downarrow}(\w;\z)]$, it holds only when $\pi_0 = \pi$. Since only one model $\w$ is considered, we let $\w_t = \w$ in the update rule of $\bm{v}$ (\Eqref{eq:update_rule_v}), and we have the following proposition:
\begin{restatable}{prop}{EstimatorUnbiasOne}
\label{prop:est_unbis_1}
    Consider updating $\bm{v}$ with \Eqref{eq:update_rule_v} for $T$ steps, then we have
    \begin{equation*}
        \Expt[\bm{v}] = \Expt[\phi(h_{\w}(\zp))] + (1-\beta)^T \left(\bm{v}_1 - \Expt[\phi\left(h_{\w}(\zp)\right)]\right), ~~ Var[\bm{v}] \leq Var[\phi(h_{\w}(\zp))] \cdot \frac{\beta}{2 - \beta}.
    \end{equation*}
\end{restatable}
\begin{rem}
    Two conclusions could be drawn from the above proposition: first, if the linear interpolation is asymptotically unbiased (see next subsection), by choosing a large $T$ or setting $\bm{v}_1 = \Expt[\phi(h_{\w}(\zp))]$, we have $\Expt[\bm{v}] \approx h_{\w}(\Sp)$; second, \textbf{by choosing a smaller $\beta$, $\bm{v}$ is more likely to concentrate on $h_{\w}(\Sp)$}.
\end{rem}

\begin{restatable}{prop}{EstimatorUnbiasTwo}
    \label{prop:est_error}
    Assume the linear interpolation is asymptotically unbiased. Let $\kappa^2_{1} = \mathop{\hat{\Expt}}_{c\sim h_{\w}(\zp)}[Var_{\bm{x}\sim\Sn}[\ell_1\left(c - h_{\w}(\bm{x})\right)]]$, $\kappa^2_{2} = \mathop{\hat{\Expt}}_{c\sim h_{\w}(\zp)}[Var_{v\sim\bm{v}}[\ell_2\left(c - v)\right)]]$. When $\kappa^2_{1} / \nn \rightarrow 0$, $\kappa^2_{2} / \np \rightarrow 0$, then there exists a positive scale $H$, such that
    \begin{equation*}
        \begin{aligned}
            \hat{\Expt}_{\z\subseteq \S}[\hat{f}(\w;\z)] \overset{P}{\rightarrow} \widehat{\text{AUPRC}}^\downarrow(\w;\S), ~~
            \hat{\Expt}_{\z\subseteq \S}\left[\widehat{\text{AP}}^\downarrow(\w;\z)\right] \overset{P}{\rightarrow} \left(1 + (\pi_0-\pi) H\right) \cdot \widehat{\text{AUPRC}}^\downarrow(\w;\S),
        \end{aligned}
    \end{equation*}
    where $\overset{P}{\rightarrow}$ refers to convergence in probability, and $\z\subseteq \S$ refers to subsets described in \cref{setting:stoc_learning}.
\end{restatable}
\begin{rem}
    The above proposition suggests that the proposed batch-based estimator is sampling-rate-invariant asymptotically unbiased, while $\widehat{\text{AP}}^\downarrow$ tends to be larger when the sampling rate of the positive class is greater than the prior, and vice versa. We also provide a non-asymptotic result in Appendix {\color{blue} B.2}.
\end{rem}
Simulation experiments are conducted as complementary to the theory. Following previous work \cite{boyd2013area}, the scores are drawn from three types of distributions, including binormal, bibeta and offset uniform. The results of binormal distribution are visualized in \Fgref{fig:simu_errors}, and detailed descriptions and more results are available in Appendix {\color{blue} B.2}.
These results are consistent with the above remark.

Next we further study the interpolation error. For the sake of presentation, denote $p: [0,1] \mapsto\mathbb{R}$ to be an increasing score function describing $h_{\w}(\Sp)$, where $p(x)$ is the score in the bottom $x$-quantile of $h_{\w}(\Sp)$. Similarly, let $\hat{p}$ to be the interpolation results of $\Expt_{A}[h_{\w}(\zp)]$. Assume that $\Expt_{A}[h_{\w}(\zp)]$ are located in the $(i/\np)$-quantiles of $p$, where $i\in [\np]$, such that $p(i/\np) = \hat{p}(i / \np)$ and all interpolation intervals are with length $1/\np$.
The following proposition provides an upper bound of the approximation error (see \cite{sauer2011numerical} for proof):
\begin{restatable}[\textbf{Linear Interpolation Error}]{prop}{Temp*}
\label{prop:intp_error}
    Let $p,\hat{p}$ be defined as above. Then we have
    \begin{equation*}
        \|p - \hat{p}\|_{\infty} \leq \|p''\|_{\infty} / \left(8(\np)^2\right).
    \end{equation*}
\end{restatable}
\vspace{-2mm}
Similar to the last subsection, simulation results are shown in \Fgref{fig:intp_errors}, which shows the expected errors of linear interpolation are ignorable.

\begin{figure}[t]
    \vspace{-5mm}
    \subfigure[Density functions of scores.]{
        \begin{minipage}[b]{0.22\textwidth}
            \includegraphics[scale=0.32]{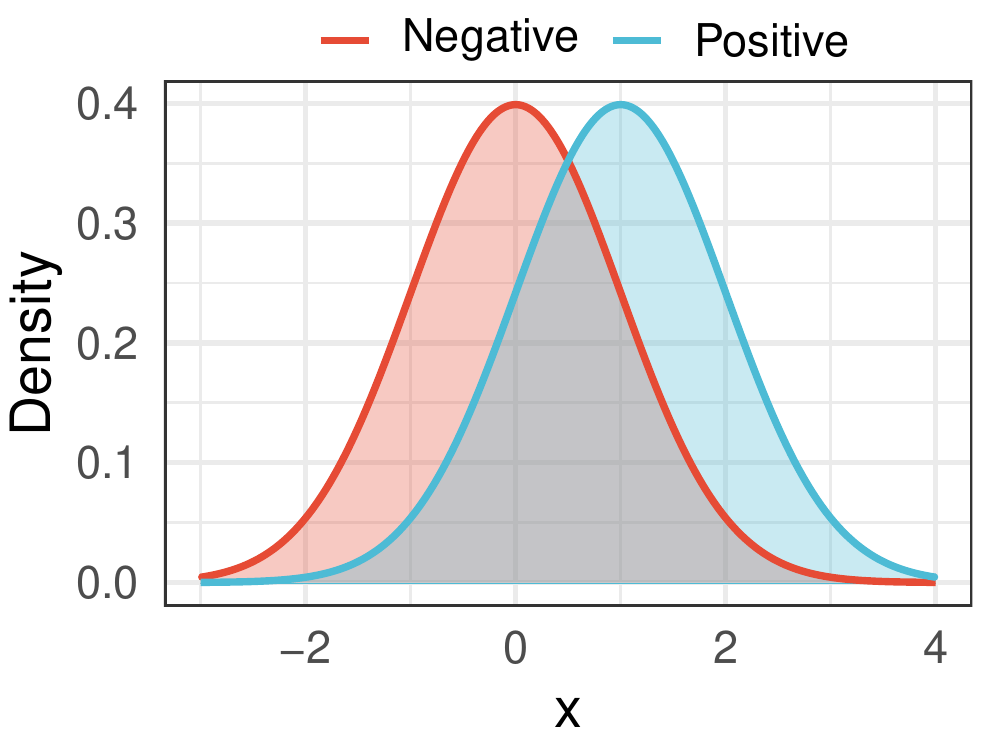}
        \end{minipage}
        \label{fig:simu_density}
	}
    \hspace{2mm}
    \subfigure[Stochastic Estimation Errors with $\pi_0=0.02$ (left) and $\pi_0=0.2$ (right).]{
		\begin{minipage}[b]{0.44\textwidth}
			\includegraphics[scale=0.32]{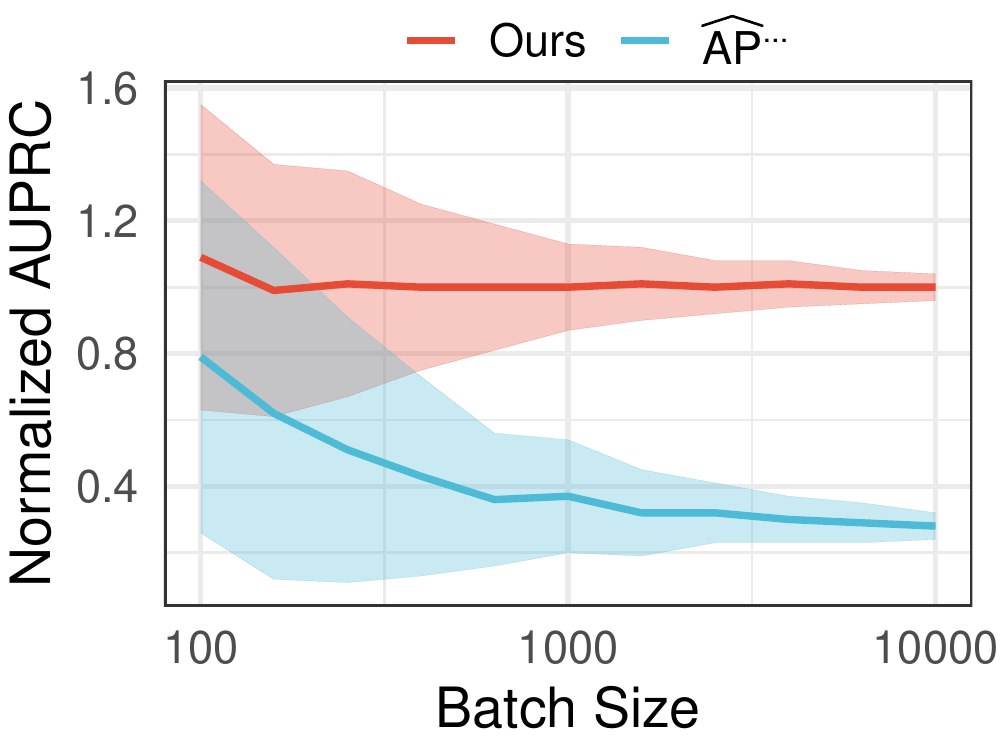}\includegraphics[scale=0.32]{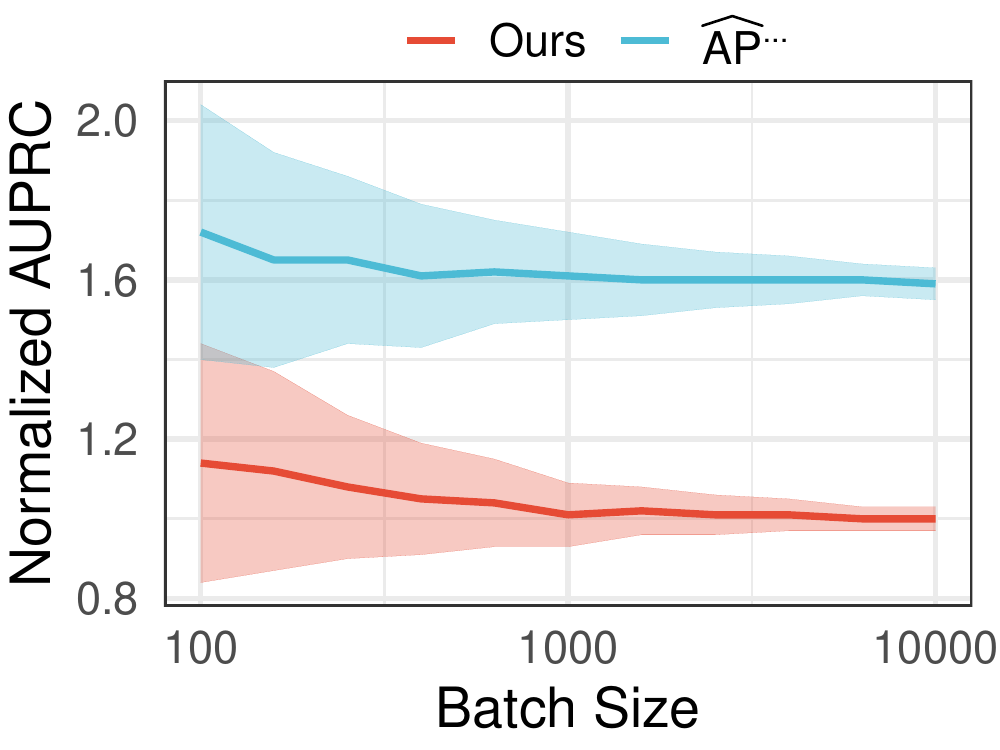}
        \end{minipage}
        \label{fig:est_errors}
    }
    \hspace{2mm}
    \subfigure[Interpolation Errors with $\pi_0=0.03$.]{
		\begin{minipage}[b]{0.22\textwidth}
			\includegraphics[scale=0.32]{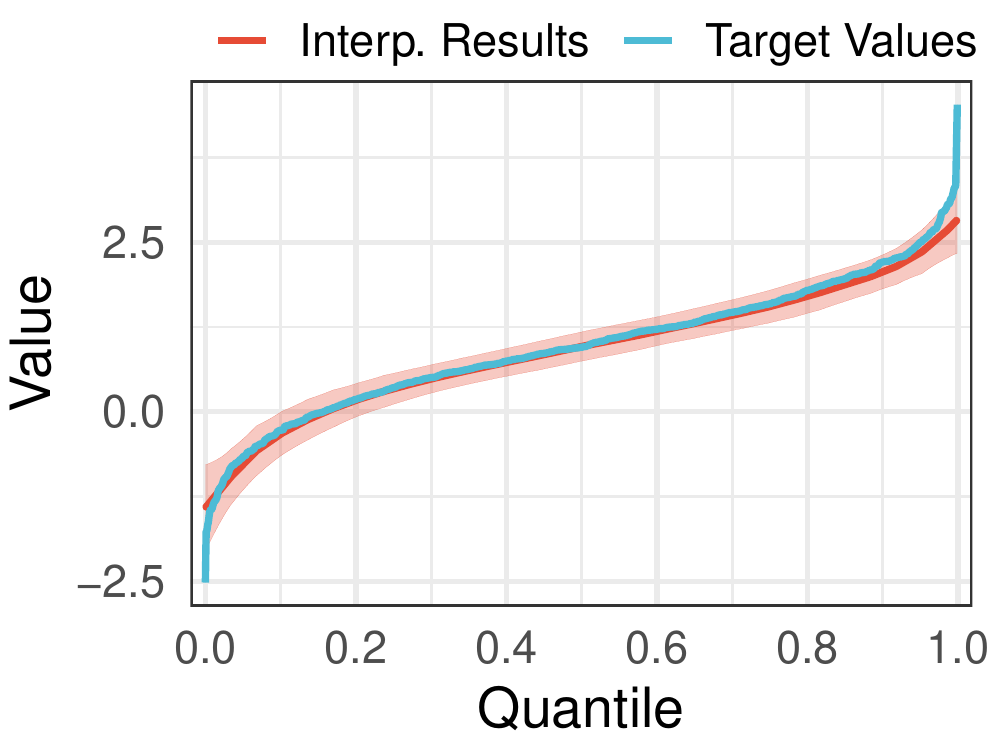}
        \end{minipage}
        \label{fig:intp_errors}
	}
	\caption{Empirical analysis of estimation errors on simulation data.}
	\label{fig:simu_errors}
\end{figure}

\subsection{Optimization Algorithm}

In the rest of this section, we focus on how to optimize $F(\w;\S)$. The main challenge is to design update rules for $\bm{v}$, such that it could efficiently and effectively approximate $h_{\w}(\Sp)$ without full-batch scanning. To overcome the challenge, we propose an algorithm called \textbf{Stochastic Optimization of AUPRC (SOPRC)}, which jointly updates model parameters $\w$ and the auxiliary vector $\bm{v}$. A summary of the detailed process is shown as \Algref{alg:main}. At step $t$, a batch of data is sampled from the training set, and then compute the corresponding scores. Afterward, scores of positive examples are mapped into a $\Np$-dimension vector with linear interpolation $\phi$ as shown in \Algref{alg:interp}. $\bm{v}_{t+1}$ are updated with the interpolated scores in a moving average manner.

Practically, $\np, \nn$ are finite, causing inevitable estimation errors in $f(\w;\z_{i_t},\bm{v}_{t+1})$. Notice that another factor influencing the stochastic estimation errors, \ie, $\kappa^2_1$ and $\kappa^2_2$. To reduce them, it is expected that the variance of positive (negative) scores are small, which motivates us to add a variance regularization term. However, it might force to reduce positive scores that higher than the mean value, which is contrary to our target. Therefore, we propose a \textbf{semi-variance regularization term} \cite{bond2002statistical}:
\begin{equation}
    \label{eq:reg_var}
    \mathcal{L}_{var} = \frac{\lambda_1}{\np} \sum_{\bm{x}\sim\zp \atop h_{\w}(\bm{x}) < \mu^+}(h_{\w}(\bm{x}) - \mu^+)^2 + \frac{\lambda_2}{\nn} \sum_{\bm{x}\sim\zn \atop h_{\w}(\bm{x}) > \mu^-}(h_{\w}(\bm{x}) - \mu^-)^2,
\end{equation}
where $\mu^+ = \frac{1}{\np} \sum_{\bm{x}\sim\zp}h_{\w}(\bm{x})$, $\mu^- = \frac{1}{\nn} \sum_{\bm{x}\sim\zn}h_{\w}(\bm{x})$, $\lambda_1, \lambda_2$ are hyperparameters.
Finally, we compute the gradients of $f(\w;\z_{i_t},\bm{v}_{t+1}) + \mathcal{L}_{var}$, and update parameters $\w$ with gradient descent.

\begin{multicols}{2}
    
    \begin{algorithm}[H]
        \caption{SOPRC}
        \label{alg:main}
        \begin{algorithmic}[1]
            \REQUIRE {Training dataset $\S$, maximum iterations $T$, learning rate $\{\eta_{t}\}_{t=1}^T$ and $\{\beta_{t}\}_{t=1}^T$.}
            \ENSURE {model parameters $\w_{T+1}$.}
            \STATE {Initialize model parameters $\w_1$ and $\bm{v}_1$.}
            \FOR{$t=1$ to $T$}
                \STATE{Sample a subset $\z_{i_t}$ from $\S$.}
                \STATE{Compute $h_{\w_t}(\zp_{i_t})$ and map the results into $\phi(h_{\w_t}(\zp_{i_t}))$ with \Algref{alg:interp}.}
                \STATE{Update $\bm{v}$ with
                \vspace{-1mm} 
                \begin{equation}
                \label{eq:update_rule_v}
                    \begin{aligned}
                        \bm{v}_{t+1} = &(1 - \beta_t) \bm{v}_{t} \\
                        &~~+ \beta_t \phi(h_{\w_t}(\zp_{i_t})).
                    \end{aligned}
                \end{equation}}
                \vspace{-2mm} 
                \STATE{Compute $\mathcal{L}_{var}$ with \Eqref{eq:reg_var}.}
                \STATE{Update the model parameter:
                \begin{equation}
                    \begin{aligned}
                        \w_{t+1} = &\w_{t} - \eta_{t} \cdot \nabla \mathcal{L}_{var} \\
                        &- \eta_{t} \cdot \nabla f(\w_t;\z_{i_t}, \bm{v}_{t+1}).
                    \end{aligned}
                \end{equation}}
            \ENDFOR
        \end{algorithmic}
    \end{algorithm}
    
    \begin{algorithm}[H]
        \caption{Score Interpolation $\phi(\cdot)$}
        \label{alg:interp}
        \begin{algorithmic}[1]
            \REQUIRE {A real value vector $\bm{u}\in\mathbb{R}^n$ where $n < \Np$, range of target values $[b, B]$.}
            \ENSURE {Interpolated vector $\bm{m} = \phi(\bm{u})$.}
            \STATE {Sort $\bm{u}$ in descending order.}
            \STATE {Initialize $\bm{m}$ as $\bm{0}_{\Np}$, let $u_0 = max(2u_1 - u_2, b), u_{n+1} = min(2u_n - 2u_{n-1}, B)$.}
            \FOR{$i=1$ to $n$}
                \FOR{$j=\lceil\frac{\Np(i-1)}{n}\rceil$ to $\left[\frac{\Np \cdot i}{n}\right]$}
                    \STATE{ \quad
                    \vspace{-4mm}
                    \begin{equation*}
                        \vspace{-2mm}    
                        \begin{aligned}
                            m_j += &\left[(i - jn / \Np)u_{i-1}\right. \\
                            & \left. + (1 + jn / \Np - i)u_{i}\right] / 2
                        \end{aligned}
                    \end{equation*}}
                \ENDFOR
                \FOR{$j=\left[\frac{\Np\cdot i}{n}\right]$ to $\lfloor\frac{\Np \cdot (i+1)}{n}\rfloor$}
                    \STATE{\quad
                    \vspace{-4mm}
                    \begin{equation*}
                        \vspace{-2mm}    
                        \begin{aligned}
                            m_j += &\left[(i + 1 - jn / \Np)u_{i-1}\right. \\
                            & \left. + (jn / \Np - i)u_{i}\right] / 2
                        \end{aligned}
                    \end{equation*}    
                    }
                \ENDFOR
            \ENDFOR
        \end{algorithmic}
    \end{algorithm}
    \end{multicols}

    




\section{Generalization of SOPRC via Stability}
\label{sec:gen}
In this section, we turn to study the \textit{excess generalization error} of the proposed algorithm. Formally, following standard settings \cite{bottou2007tradeoffs}, we consider the test error of the model $A(\S)$ trained on the training set $\S$. Our target is to seek an upper bound of the excess error $\Expt_{A,\S}[F(A(\S)) - F(\w^*)]$, where $\w^* \in \arg\min_{\w\in\Omega} \Expt_{A,\S}[F(\w^*)]$. It can be decomposed as:
\begin{equation*}
    \Expt_{\S,A}[F(A(\S)) - F(\w^*)] = 
    \underbrace{\Expt_{\S,A}[F(A(\S)) - F(A(\S);\S)]}_{\textit{Estimation Error}} + \underbrace{\Expt_{\S,A}[F(A(\S);\S) - F(\w^*)]}_{\textit{Optimization Error}}.
\end{equation*}
The estimation error sources from the gap of minimizing the empirical risk instead of the expected risk. In \cref{gen:estab_error}, we provide detailed discussion on the estimation error. The optimization error measures the gap between the minimum empirical risk and the results obtained by the optimization algorithm, which will be studied in \cref{gen:convergence}. Detailed proofs of this section are available in Appendix {\color{blue} C}.
Before the formal presentation, we show the main assumptions:
\begin{asm}[\textbf{Bounded Scores \& Gradient}]
\label{asm:bound_gradient}
    $ |\hat{f}(\w;\cdot)|\leq B, \|\nabla \hat{f}(\w;\cdot)\|_2 \leq G$ for all $\w \in \Omega$.
\end{asm}
\begin{asm}[\textbf{L-Smooth Loss}]
\label{asm:smooth_loss}
    $\|\nabla \hat{f}(\w;\cdot) - \nabla \hat{f}(\tilde{\w};\cdot)\|_2 \leq L\|\w - \tilde{\w}\|_2$ for all $\w,\tilde{\w} \in \Omega$.
\end{asm}
\begin{asm}[\textbf{Lipschitz Continuous Functions}]
\label{asm:lc_sur}
    $|\ell_1(x) - \ell_1(\tilde{x})| \leq L_1|x - \tilde{x}|$, $|\ell_2(x) - \ell_2(\tilde{x})| \leq L_2|x - \tilde{x}|$ for all $x, \tilde{x} \in [-2B,2B]$. $\|\phi(\bm{x}) - \phi(\tilde{\bm{x}})\|_2 \leq C_\phi\|\bm{x} - \tilde{\bm{x}}\|_2$ for all $\bm{x},\tilde{\bm{x}}\in\mathbb{R}^{\Np}$.
\end{asm}

\subsection{Generalization of AUPRC via Model Stability}
\label{gen:estab_error}
The generalization of SGD-style algorithms for instancewise loss has been widely studied with stability measure \cite{lei2020fine, elisseeff2005stability, hardt2016train}.
However, these results could not be directly applied to listwise losses like AUPRC. The main reason is that the estimation of each stochastic gradient requires a list of examples, and the estimation is usually biased.
Nonetheless, to bridge the optimization algorithm and the generalization of AUPRC, we propose a listwise variant of \textit{on-average model stability} \cite{lei2020fine} as follows:
\begin{defi}[\textbf{Listwise On-average Model Stability}]
\label{defi:stab}
    Let $\S = \{(\bm{x}_i, y_i)\}_{i=1}^N$ and $\widetilde{\S} = \{(\widetilde{\bm{x}}_i, y_i)\}_{i=1}^N$ be two sets of examples whose features are drawn independently from $\mathcal{X}$. For any $i=1,\cdots,N$, denote $\S^{(i)} = \{(\bm{x}_1, y_1),\cdots,(\bm{x}_{i-1}, y_{i-1}),(\widetilde{\bm{x}}_i, y_i),(\bm{x}_{i+1},y_{i+1}),\cdots,(\bm{x}_n, y_n)\}$. A stochastic algorithm $A$ is listwise on-average model $(\epsilon^+, \epsilon^-)$-stable if the following condition holds:
    \begin{equation*}
        \begin{aligned}
            \Expt_{\S,\widetilde{\S},A}\left[
                \frac{1}{\Np} \sum_{y_i=1}\left\|\
                    A(\S) - A(\S^{(i)})
                \right\|_2
            \right] \leq \epsilon^+,
            \Expt_{\S,\widetilde{\S},A}\left[
                \frac{1}{\Nn} \sum_{y_i=-1}\left\|\
                    A(\S) - A(\S^{(i)})
                \right\|_2
            \right] \leq \epsilon^-.
        \end{aligned}
    \end{equation*}
\end{defi}




The following theorem shows that the estimation error is bounded by the above-defined stability:
\begin{restatable}[\textbf{Generalization via Model Stability}]{thm}{GenViaStab}
\label{thm:gen_via_stab}
    Let a stochastic algorithm $A$ be listwise on-average model $(\epsilon^+, \epsilon^-)$-stable and \Asmpref{asm:bound_gradient} holds. Then we have
    \begin{equation}
        \Expt_{\S,A}\left[
            F(A(\S)) - F(A(\S);\S)
        \right] \leq G(\np\epsilon^+ + \nn\epsilon^-). 
    \end{equation}
\end{restatable}

With the above theorem, now we only need to focus on the model stability of the proposed algorithm. Notice that in \Algref{alg:main}, both $\w_t$ and $\bm{v}_t$ are updated at each step, thus we have to consider the stability of both simultaneously. The following lemma provides a recurrence for the stability $\w_{t}$ and $\bm{v}_t$.
\begin{restatable}{lem}{ModelStab}
\label{lem:stab_of_sgd}
    Let $\S,\widetilde{\S},\S^{(i)}$ be constructed as \Defiref{defi:stab} and \Asmpref{asm:bound_gradient}, \ref{asm:smooth_loss}, \ref{asm:lc_sur} hold. Let $\{\w_t\}_t$ and $\{\w_t^{(i)}\}_t$ be produced by \Algref{alg:main} with $\S$ and $\S^{(i)}$, respectively. Denote $L = \max\{L_w, L_v / \np, C_\phi B, G/2, B'_\ell\}$,   
        $\bm{m}_t^{(i)} =  \left[
            \begin{array}{ccc}
                    \| \w_{t} - \w_{t}^{(i)} \|_2
                 & 
                    \| \bm{v}_{t} - \bm{v}_{t}^{(i)} \|_2
                 & 1
            \end{array}
        \right]^\top$, $\bm{m}_{t}^{+} = \frac{1}{\Np} \sum_{y_i=1} \Expt_{\S,A}\left[\bm{m}_{t}^{(i)}\right]$, $\bm{m}_{t}^{-} = \frac{1}{\Np} \sum_{y_i=-1} \Expt_{\S,A}\left[\bm{m}_{t}^{(i)}\right]$.
    Then for all $t\in[T]$, by setting $\beta_t \leq 2C_\phi B / \np$, we have 
    \begin{equation}
            \bm{m}_{t+1}^{+} \leq 
            \left(\bm{I}_{3} + \bm{R}^+_{t}\right) \cdot \bm{m}_{t}^{+},~~~~~~ \bm{m}_{t+1}^{-} \leq 
            \left(\bm{I}_{3} + \bm{R}^-_{t}\right) \cdot \bm{m}_{t}^{-},
    \end{equation}
    where $\bm{I}_{3}$ is the $3\times 3$ identity matrix and
    \begin{equation}
        \begin{aligned}
            R^+_{t} &= 
            \left[
                \begin{array}{ccc}
                    2L\eta_t & \frac{L(1-\beta_t) \eta_t}{\Np} & \frac{L\eta_t}{\Np}\\
                    L\beta_t & 0 & \frac{1}{\Np}\\
                    0 & 0 & 0
                \end{array}
            \right],
            R^-_{t} &= 
            \left[
                \begin{array}{ccc}
                    2L\eta_t & \frac{L_v(1 - \beta_t)\eta_t}{\Np} &  \frac{L\eta_t \cdot \np}{\Nn} \\
                    L\beta_t & 0 & 0 \\
                    0 & 0 & 0
                \end{array}
            \right].
        \end{aligned}
    \end{equation}
\end{restatable}

Finally, we utilize the matrix spectrum of $R^+_{t}$ and $R^-_{t}$ to show that the model stability w.r.t. \Algref{alg:main} decreases as the number of training examples increases (see 
Appendix {\color{blue} C.2}
for details):
\begin{restatable}[]{thm}{ColGenConstStep}
\label{thm:stab}
    Let $\lambda = LC_\eta(1+\sqrt{1-\beta^2+\beta})$, and assumptions in \Lemref{lem:stab_of_sgd} hold. By setting $\eta_t \leq \frac{C_\eta}{t}$, $\beta_t = \beta \asymp 1/\np$ and $T \leq \Np$, \Algref{alg:main} is list on-average model stable with 
    \begin{equation}
        \begin{aligned}
            \epsilon^+ = \mathcal{O}\left(
                \frac{\left(T\np\right)^{\frac{\lambda}{\lambda+1}}}{\Np}
            \right),
            \epsilon^- = \mathcal{O}\left(
                \frac{\left(T\nn\right)^{\frac{\lambda}{\lambda+1}}}{\Nn}
            \right).
        \end{aligned}
    \end{equation}
\end{restatable}

\subsection{Convergence of AUPRC Stochastic Optimization}
\label{gen:convergence}
Following previous work \cite{foster2018uniform, karimi2016linear}, we study the optimization error of the proposed algorithm under the \textit{Polyak-\L ojasiewicz (PL)} condition. It has been shown that the PL condition holds for several widely used models including some classes of neural networks \cite{charles2018stability,liu2022loss}.

\begin{asm}[\textbf{Polyak-\L ojasiewicz Condition} \cite{karimi2016linear, lei2021generalization}]
\label{asm:pl_condition}
    Denote $\w^* = \arg\min_{\w\in\Omega} F(\w)$. Assume $F$ satisfy the expectation version of PL condition with parameter $\mu > 0$, \ie,
    \begin{equation}
        \Expt_{\S}[F(\w;\S) - F(\w^*)] \leq \frac{1}{\mu} \Expt_{\S}[\|\nabla F(\w;\S)\|_2^2].
    \end{equation}
\end{asm}
The main difference to the existing convergence analysis on non-convex optimization is that the gradient estimation is biased. Nonetheless, we show that the bias terms from \Algref{alg:main} tend to $0$ with sufficient training data and training time (see 
Appendix {\color{blue} C.3}),
leading to the following convergence:
\begin{restatable}{thm}{ConvergenceFDecLR}
\label{thm:convergence}
    Let \Asmpref{asm:bound_gradient}, \ref{asm:lc_sur}, \ref{asm:pl_condition} hold. By setting $\eta_t = \frac{2t+1}{\mu(t+1)^2}$ and $\beta_t = \beta \asymp 1/\np$, we have 
    \begin{equation}
        \begin{aligned}
            \Expt_{A}[F(\w_{T+1}) - F(\w^*)] = \mathcal{O}\left(\np / T + 1 / \Np\right).
        \end{aligned}
    \end{equation}
\end{restatable}

\begin{restatable}{thm}{FinalGen}
    Let assumptions in \Thmref{thm:stab} and \ref{thm:convergence} hold. By setting $T \asymp (\Np)^{\frac{\lambda+1}{2\lambda+1}}(\np)^{- \frac{1}{2\lambda + 1}}$, we have
\begin{equation}
    \Expt_{\S,A}[F(A(\S)) - F(\w^*)] = \mathcal{O}\left((\Np)^{-\frac{\lambda+1}{2\lambda+1}} \cdot (\np)^{\frac{3\lambda + 1}{2\lambda + 1}} \right) + \mathcal{O}\left((\Nn)^{-\frac{\lambda+1}{2\lambda+1}} \cdot (\nn)^{\frac{3\lambda + 1}{2\lambda + 1}} \right).
\end{equation}
\end{restatable}
\begin{rem}
    Recall that $\lambda = LC_\eta(1+\sqrt{1-\beta^2+\beta})$ and $C_\eta = 4 / \mu$, when $\beta$ is small, we have $\lambda \approx 4L / \mu$. Here $L/\mu$ is a condition number determined by the model and surrogate losses. Notice that $\np \ll \Np, \nn\ll\Nn$, if $\lambda = 1$, the generalization bound is $\mathcal{O}\left((\Np)^{-2/3} \cdot (\np)^{4/3} + (\Nn)^{-2/3} \cdot (\nn)^{4/3} \right)$. As $\lambda$ increases, it increases to $\mathcal{O}\left((\Np)^{-1/2} \cdot (\np)^{3/2} + (\Nn)^{-1/2} \cdot (\nn)^{3/2} \right)$.
\end{rem}

\begin{table}[t]
    \caption{Quantitative results on SOP, iNaturalist, and VehicleID. All methods are trained with training sets. The best and the second best results are highlighted in \first{soft red} and \second{soft blue}, respectively.}
    \setlength\tabcolsep{4.6pt}
    \centering
    \resizebox{\textwidth}{!}{
    \begin{tabular}{l|ccc|ccc|ccc}
      \toprule
      \multirow{3}{*}{Methods} & \multicolumn{3}{c|}{\textbf{Stanford Online Products}} & \multicolumn{3}{c|}{\textbf{iNaturalist}} & \multicolumn{3}{c}{\textbf{PKU VehicleID}}\\
      \cline{2-10} & \small{mAUPRC} & R@1 & R@10 & \small{mAUPRC} & R@1 & R@4 & \small{mAUPRC} & R@1 & R@5 \\
      \midrule
      Contrastive loss \cite{hadsell2006dimensionality} & 57.73 &	77.60 &	89.31 & 27.99 & 54.19 & 71.12 & 67.26 & 87.46 & 94.60 \\
      Triplet loss \cite{hoffer2015deep} & 58.07 & 78.34 & 90.50 & 30.59 & 60.53 & 77.62 & 70.99 & 90.09 & 95.54
      \\
      MS loss \cite{wang2019multi} & 60.10 & 79.64 & 90.38 & 30.28	& 63.39 & 78.50 & 69.15 & 88.82 & 95.06
      \\
      XBM \cite{wang2020cross} & 61.29 & 80.66 & 91.08 & 27.46 & 59.12 & 75.18 & 71.24 & \first{92.78} & 95.83
      \\
      SmoothAP \cite{brown2020smooth} & \second{61.65} & \second{81.13} & \second{92.02} & \second{33.92} & \second{66.13} & \second{80.93} & 72.28 & 91.31 & 96.05
      \\
      DIR \cite{revaud2019learning} & 60.74 & 80.52 & 91.35 & 33.51 & 64.86 & 79.79 & \second{72.72} & 91.38 & \second{96.10}
      \\
      FastAP \cite{cakir2019deep} & 57.10 & 77.30 & 89.61 & 31.02 & 56.64 & 73.57 & 70.82 & 89.42 & 95.38
      \\
      AUROC \cite{gao2015consistency} & 55.80 & 77.32 & 89.64 & 27.24 & 60.88 & 77.76 & 58.12 & 81.73 & 91.92
      \\
      BlackBox \cite{poganvcic2019differentiation} & 59.74 & 79.48 & 90.74 & 29.28 & 56.88 & 74.10 & 70.92 & 90.14 & 95.52
      \\
      \midrule
      Ours & \first{62.75} & \first{81.91} & \first{92.50} & \first{36.16} & \first{68.22} & \first{82.86} & \first{74.92} & \second{92.56} & \first{96.43}
      \\
      \bottomrule
    \end{tabular}
    }
    \label{tab:results}
  \end{table}

\section{Experiments}
To validate the effectiveness of the proposed method, we conduct empirical studies on the image retrieval task, in which data distributions are largely skewed and AUPRC is commonly used as an evaluation metric. More detailed experimental settings are provided in Appendix {\color{blue} D.1}. The source code is available in \url{https://github.com/KID-7391/SOPRC.git}.

\subsection{Implementation Details}
\label{sec:settings}
\textbf{Datasets.} We evaluate the proposed method on three image retrieval benchmarks with various domains and scales, including \textbf{Stanford Online Products (SOP)}\footnote{\url{https://github.com/rksltnl/Deep-Metric-Learning-CVPR16}. Licensed MIT.} \cite{oh2016deep}, \textbf{PKU VehicleID}\footnote{\url{https://www.pkuml.org/resources/pku-vehicleid.html}. Data files \copyright~ Original Authors.} \cite{liu2016deep} and \textbf{iNaturalist}\footnote{\url{https://github.com/visipedia/inat comp/tree/master/2018}. Licensed MIT.} \cite{van2018inaturalist}. We follow the official setting to split a test set from each dataset, and then further split the rest into a training set and a validation set by a ratio of $9:1$. 

\textbf{Network Architecture.} The feature extractor is implemented with ResNet-50 \cite{he2016deep} pretrained on ImageNet \cite{russakovsky2015imagenet}. Following previous work \cite{cakir2019deep, brown2020smooth}, the batch normalization layers are fixed during training, and the output embeddings are mapped to 512-d with a linear projection. Given $L_2$ normalized embeddings of a query image $\bm{e}_q$ and a gallery list $\{\bm{e}_i\}_i$, the scores are represented by the cosine similarity $\bm{e}_q^{\top} \bm{e}_i$ for all $i$.

\textbf{Optimization Strategy.} In the training phase, the input images are resized such that the sizes of the shorter sides are 256. Afterward,  we applied standard data augmentations including random cropping ($224\times 224$) and random flipping ($50\%$). The model parameters are optimized in an end-to-end manner as shown in \Algref{alg:main}, where $\beta = 0.001$ and the weight decay is set to $4\times 10^{-4}$. The default batch size is set to 224, where each mini-batch is randomly sampled such that there are exactly 4 positive examples per category. The learning rates are tuned according to performance on validation sets: for SOP, the learning rate are initialized as $0.01$ and decays by $0.1$ at the $15k$ and $30k$ iterations, $T=50k$; for VehicleID, the learning rate are initialized as $0.001$ and decays by $0.1$ at the $40k$ and $80k$ iterations, $T=100k$; for iNaturalist, the learning rate are initialized as $0.001$ and decays by $0.1$ at the $80k$ and $110k$ iterations, $T=130k$.

\textbf{Competitors.} We compare two types of competitors: \textbf{1) Pairwise Losses}, including \textit{Contrastive Loss} \cite{hadsell2006dimensionality}, \textit{Triplet Loss} \cite{hoffer2015deep}, \textit{Multi-Similarity (MS) Loss} \cite{wang2019multi}, \textit{Cross-Batch Memory (XBM)} \cite{wang2020cross}. These methods construct loss functions with image pairs or triplets. \textbf{2) Ranking-Based Losses}, including \textit{SmoothAP} \cite{brown2020smooth}, \textit{FastAP} \cite{cakir2019deep}, \textit{DIR} \cite{revaud2019learning}, \textit{BlackBox} \cite{jiang2020optimizing}, and \textit{Area Under the ROC Curve Loss (AUROC)} \cite{yang2021learning}. These methods directly optimize the ranking-based metrics.

\textbf{Evaluation Metrics.} In all experiments, we adopt evaluation metrics: \textit{mean AUPRC (mAUPRC)} and \textit{Recall@k}. mAUPRC is also called mean average precision (mAP) in literature, which takes the mean value over the AUPRC of all queries. Recall@k measures the probability that at least one positive example is ranked in the top-k list.


\subsection{Main Results}
We evaluate all methods with \textit{mean AUPRC (mAUPRC)} and \textit{Recall@k}. mAUPRC measures the mean value of the AUPRC over all queries, a.k.a. mean average precision (mAP).
The performance comparisons on test sets are shown in \Tbref{tab:results}. Consequently, we have the following observations: \textbf{1)} In all datasets, the proposed method surpasses all competitors in the view of mAUPRC, especially in the large-scale long-tailed dataset iNaturalist. This validates the advantages of our method in boosting the AUPRC of models. \textbf{2)} Compared to pairwise losses, the AUPRC/AP optimization methods enjoy better performance generally. The main reason is that pairwise losses could only optimize models indirectly by constraining relative scores between positive and negative example pairs, while ignoring the overall ranking. \textbf{3)} Although some pairwise methods like XBM have a satisfying performance on Recall@1, their mAUPRC is relatively low. It is caused by the limitation of Recall@1, \ie, it focuses on the top-1 score while ignoring the ranking of other examples. What's more, this phenomenon shows the inconsistency of Recall@k and AUPRC, revealing the necessity of studying AUPRC optimization. More results are available in Appendix {\color{blue} D.2}.

To qualitatively demonstrate the effect of the proposed method, we also show the mean PR curves and convergence curves in \Fgref{fig:quan}. The left two subfigures demonstrate that the proposed method achieves can effectively improve AUPRC. The right subfigure shows the affect of batch size, from which can be seen that a large batch size leads to better performance. One of the reasons is that a small batch size will amplify the AUPRC stochastic estimation error. Such a problem have been addressed by maintaining inner gradient estimations in a moving average manner \cite{qi2021stochastic,wang2021momentum, yang2022algorithmic}. Unfortunately, when applied to image retrieval problems, it needs to maintain intermediate variables for each sample pair, which will bring high complexity in time and space, thus we leave this problem as further work.

\begin{figure}[t]
    \vspace{-3mm}
    \centering
    \includegraphics[scale=0.203]{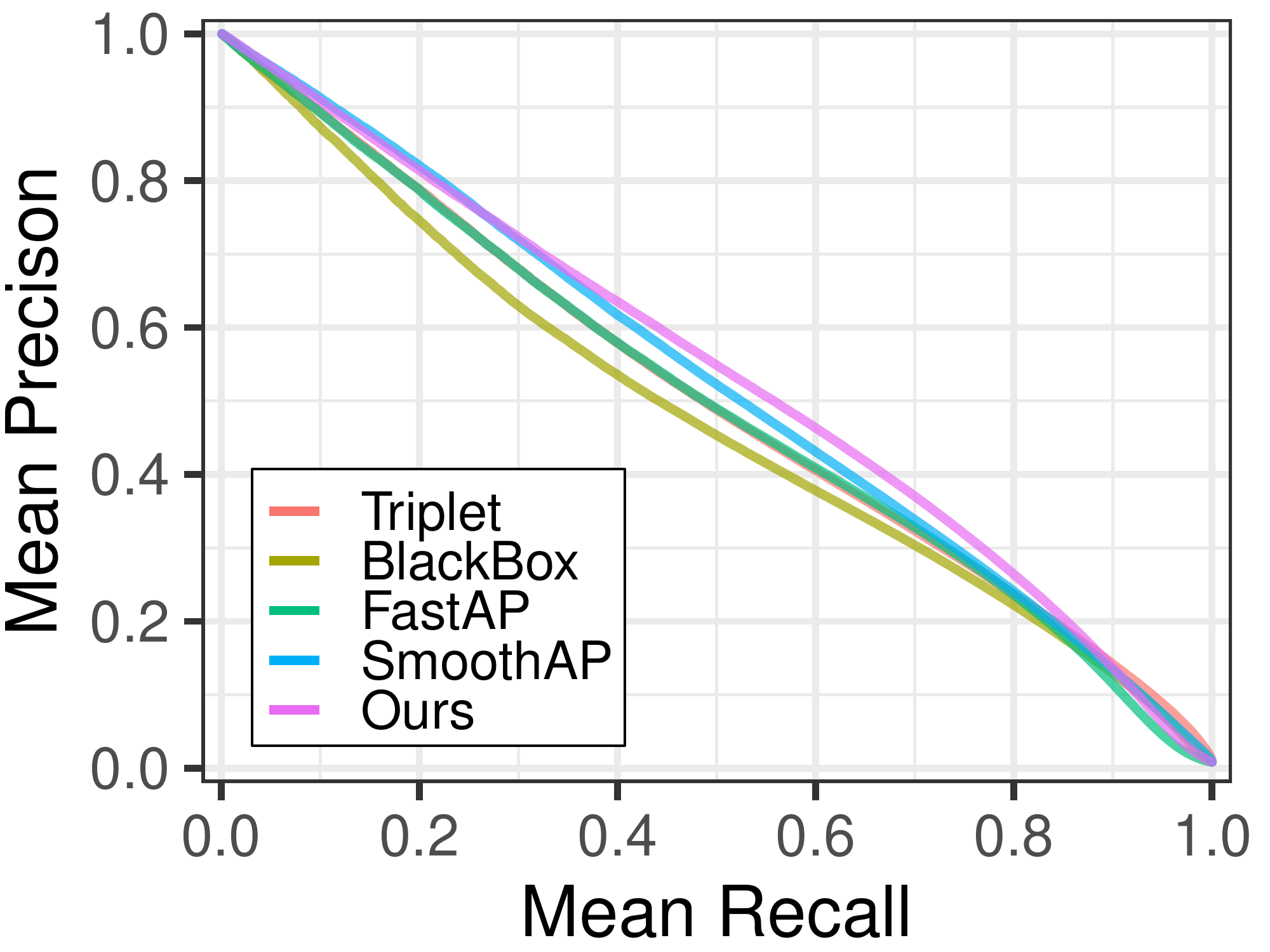}
    \includegraphics[scale=0.48]{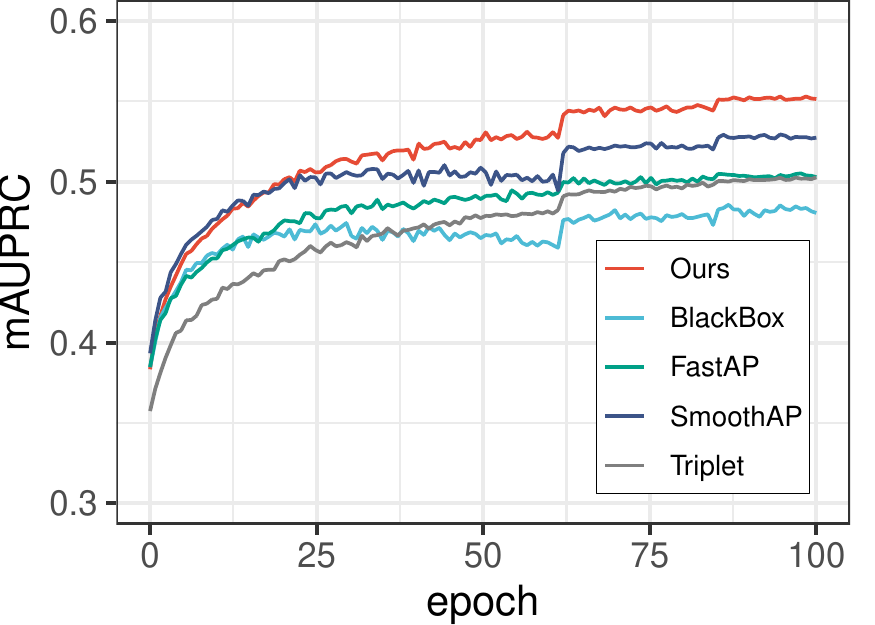}
    \includegraphics[scale=0.48]{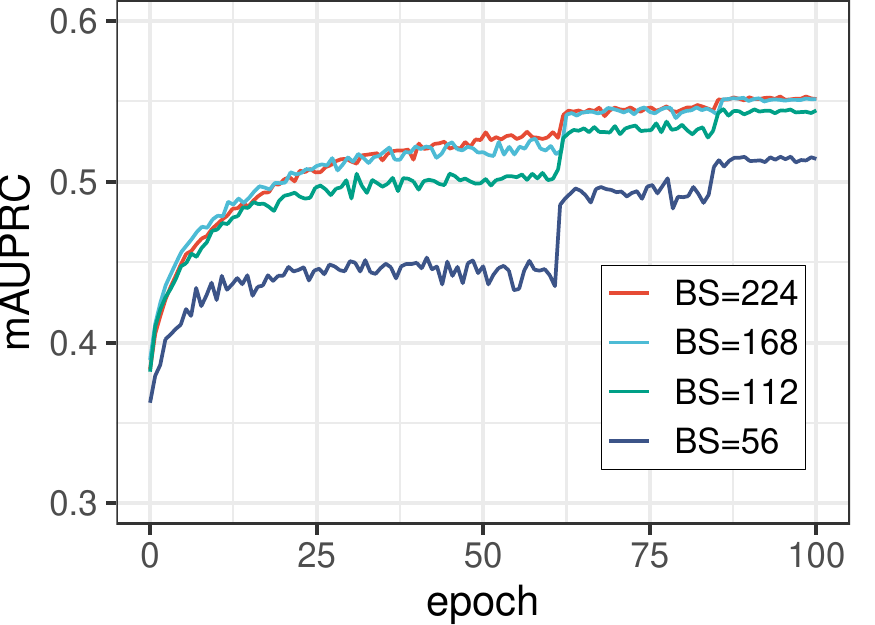}
	\caption{Qualitative results on iNaturalist. Left most: mean PR curves of different methods. Right two: convergence of different methods and batch sizes in terms of mAUPRC in the validation set.}
	\label{fig:quan}
\end{figure}

\subsection{Ablation Studies}
\begin{table}
\caption{Ablation study over different components of our method on iNaturalist.}
\centering
\resizebox{0.93\textwidth}{!}{
\begin{tabular}{ccccc|ccccc}
    \toprule
    No. & Unb. Est. & with $\bm{v}_t$ & with $\mathcal{L}_{var}$ & Opt. & \small{mAUPRC} & R@1 & R@4 & R@16 & R@32 \\
    \midrule
    1 & \xmark & \xmark & \xmark & SGD & 34.58 & 66.35 & 81.04 & 89.80 & 92.72 \\
    2 & \checkmark & \xmark & \xmark & SGD & 35.84 & 67.08 & 81.68 & 90.17 & 92.98 \\
    3 & \checkmark & \checkmark & \xmark & SGD & 35.99 & 67.50 & 82.03 & 90.44 & 93.26 \\
    4 & \checkmark & \checkmark & \checkmark & SGD & \second{36.16} & \second{68.22} & \first{82.86} & \first{91.02} & \first{93.71} \\
    5 & \checkmark & \checkmark & \checkmark & Adam & \first{36.20} & \first{68.48} & \second{82.70} & \second{90.96} & \second{93.63} \\
    \bottomrule
\end{tabular}
}
\label{tab:ablation}
\end{table}

We further investigate the effect of different components of the proposed method. Results are shown in \Tbref{tab:ablation}, and more detailed statements and analyses are as follows.\\
\textbf{Effect of Unbiased Estimator.} To show the performance drop caused by the biased estimator, we replace the prior $\pi$ in \Eqref{eq:estimator} with $\np / (\np + \nn)$. Comparing line 1 and line 2, using the unbiased estimator increases the mAUPRC by 1.3\%, which is consistent with our theoretical results in \cref{method:error_analysis}. Notably, the unbiased estimator is the main source of improvements in terms of mAUPRC. \\
\textbf{Effect of $\bm{v}_t$.} To show the effect of introducing $\bm{v}_t$ to estimate $\phi(\Sp)$, we directly use $\phi(\zp)$ instead in the first two lines. Comparing line 2 and line 3, using $\bm{v}_t$ could bring consistent improvements due to the better generalization ability. \\ 
\textbf{Effect of $\mathcal{L}_{var}$.} We show that shrinking variances could reduce the batch-based estimation errors. Comparing line 3 and line 4, it can be seen that $\mathcal{L}_{var}$ further boosts the proposed method.\\
\textbf{Effect of Optimizer.} Comparing line 4 and line 5, it can be seen that the choice of optimizer only has a slight influence.

\section{Conclusion \& Future Work}
\label{sec:conclusion}
In this paper, we present a stochastic learning framework for AUPRC optimization. To begin with, we propose a stochastic AUPRC optimization algorithm based on an asymptotically unbiased stochastic estimator. By introducing an auxiliary vector to approximate the scores of positive examples, the proposed algorithm is more stable. On top of this, we study algorithm-dependent generalization. First, we propose list model stability to handle listwise losses like AUPRC, and bridge the generalization and the stability. Afterward, we show that the proposed algorithm is stable, leading to an upper bound of the generalization error. Experiments on three benchmarks validate the advantages of the proposed framework. One limitation is the convergence rate is controlled by the scale of the dataset. In the further, we will consider techniques like variance reduction to improve the convergence rate, and jointly consider the corresponding algorithm-dependent generalization.

\section*{Acknowledgments}
    This work was supported in part by the National Key R\&D Program of China under Grant 2018AAA0102000, in part by National Natural Science Foundation of China: U21B2038, 61931008, 6212200758 and 61976202, in part by the Fundamental Research Funds for the Central Universities, in part by Youth Innovation Promotion Association CAS, in part by the Strategic Priority Research Program of Chinese Academy of Sciences, Grant No. XDB28000000, in part by the China National Postdoctoral Program for Innovative Talents under Grant BX2021298, and in part by China Postdoctoral Science Foundation under Grant 2022M713101.




{\small
\balance
\bibliographystyle{plain}
\bibliography{neurips_2022}
}

\newpage
\appendix



\newpage

\section{Related Work}
\label{sec:related_work}
\subsection{AUPRC Optimization}
To measure model performances in largely skewed datasets, Raghavan \etal \cite{raghavan1989critical} use Precision-Recall (PR) curves to describe the trade-off between precision and recall, leading to a metric named Area Under the PR Curve (AUPRC). Benefiting from its insensitivity to data distribution, AUPRC has been widely used in imbalanced scenarios, such as information retrieval \cite{chen2009ranking,qin2008query,metzler2005markov}, recommendation systems \cite{tran2019improving, chen2017sampling} and computer visions \cite{brown2020smooth, chen2020ap, cakir2019deep}. The important application value of AUPRC and the inconsistency with other metrics \cite{davis2006relationship} have raised a wave of research on direct optimization of AUPRC. Early work can be roughly divided into discrete methods and continuous methods. The first technical route utilizes discrete optimization methods to directly optimize the non-differentiable objective, \eg,  Markov random field model \cite{metzler2005markov}, randomized search \cite{goadrich2006gleaner}, dynamic programming \cite{yue2007support}, and error driven method \cite{burges2006learning}. The second technical route seeks for continuous surrogate objective, like convex upper bound on AUPRC for support vector machines (SVMs) \cite{mohapatra2014efficient}. Unfortunately, limited by the high computational complexity, these methods are not suitable for deep learning.

With the increasing application of deep learning in ranking problems, stochastic optimization for AUPRC has attracted the attention of researchers. It is challenging since AUPRC is neither differentiable nor decomposable. Therefore, the mainstream approaches tackle stochastic optimization of AUPRC from two aspects: surrogate loss functions and batch-based estimators. In the first aspect, some methods replace the non-differentiable 0-1 loss with surrogate functions like exponential loss \cite{qin2008query}, sigmoid loss \cite{brown2020smooth}, and linear interpolation function \cite{jiang2020optimizing}. Contrary to approximating the objective with differentiable functions, another route tackle this problem in an error driven style. More specifically, Burges \etal \cite{burges2006learning,burges2010ranknet} propose to decompose the gradients with chain rule, and then replace the differential w.r.t. prediction scores with the difference. Since the differential of scores w.r.t. model parameters are available, the gradients can be obtained in this way. This idea has been extended to solve the imbalance problem in object detection \cite{chen2020ap, chen2019towards, oksuz2020ranking}. However, these methods cannot guarantee the relationship between the surrogate optimization objective and the original AUPRC. In contrast, we propose a differentiable upper bound of AUPRC as a surrogate objective function.

As for the batch-based estimators, the mainstream approaches use a batch of examples to calculate average precision \cite{brown2020smooth,qi2021stochastic,qin2010general}, or approximate precision and recall \cite{revaud2019learning,cakir2019deep,he2018hashing} with the histogram binning technique \cite{NIPS2016_325995af}. However, due to the biased sampling, these estimators are not asymptotically unbiased, leading to biased stochastic gradient estimations. Moreover, since the number of positive examples sampled in a batch is typically limited, these algorithms are less stable, which is not conducive to the generalization. 

Limitations of existing work motivate us to design a more appropriate estimator that (asymptotically) unbiased and stable. In this work, we propose an unbiased estimator and further enhance the stability with an auxiliary set assisting the TPR estimation. 


\subsection{Generalization via Stability of Stochastic Optimization}
The algorithmic stability \cite{rogers1978finite, bousquet2002stability,elisseeff2005stability} is a standard framework for generalization analysis. Besides the original uniform stability \cite{bousquet2002stability}, various types of stability are studied, \eg, expected stability \cite{shalev2010learnability}, hypothesis set stability \cite{foster2019hypothesis}, and on-average model stability \cite{lei2020fine}. Unlike another standard framework built on the Rademacher complexity \cite{bartlett2002rademacher,poggio2002mathematical,valiant1984theory}, generalization via stability takes the optimization algorithm into account. This is equivalent to constraining the hypothesis set to the possible optimization outcomes, which are usually in the neighborhood of the (local) optimum. On the other hand, this framework can naturally consider two factors simultaneously, \ie, optimization errors and estimation errors, allowing jointly consideration of generalization and convergence trade-offs \cite{chen2018stability}. These advantages enable stability to be applied in a wide variety of conditions \cite{li2019generalization,mou2018generalization,charles2018stability}. 

However, existing techniques focus on instancewise or pairwise \cite{lei2021generalization,lei2020sharper} loss functions, while AUPRC is a listwise metric. In addition, each term in the AUPRC are related to all instances, thus the stability is limited by the capacity of a batch. Nonetheless, we propose a listwise variant of on-average model stability \cite{lei2020fine}, and further develop generalization guarantees of our proposed stochastic optimization algorithm for AUPRC. 

\section{Details on Error Analysis}
\subsection{Proofs on the Error Bounds}
\label{app:error_proof}
\EstimatorUnbiasOne*
\begin{proof}
    With the update rule we have
    \begin{equation}
        \bm{v} = \sum_{t=1}^T \beta(1-\beta)^{t-1} \phi\left(h_{\w}(\zp_{i_t})\right) + (1 - \beta)^{T} \bm{v}_1,
    \end{equation}
    thus we have the expectation of $\bm{v}$
    \begin{equation}
        \begin{aligned}
            \Expt[\bm{v}] =& \sum_{t=1}^T \beta(1-\beta)^{t-1} \Expt[\phi\left(h_{\w}(\zp)\right)] + (1-\beta)^T \bm{v}_1 \\
            =& \beta \cdot \frac{1-(1-\beta)^T}{1 - (1-\beta)}\Expt[\phi\left(h_{\w}(\zp)\right)] + (1-\beta)^T \bm{v}_1 \\
            =& \Expt[\phi\left(h_{\w}(\zp)\right)] + (1-\beta)^T \left(\bm{v}_1 - \Expt[\phi\left(h_{\w}(\zp)\right)]\right),
        \end{aligned}
    \end{equation}
    and the variance
    \begin{equation}
        \begin{aligned}
            Var[\bm{v}] =& \sum_{t=1}^T \beta^2(1-\beta)^{2t-2} Var[\phi\left(h_{\w}(\zp)\right)] \\
            \leq& Var[\phi\left(h_{\w}(\zp)\right)] \cdot \frac{\beta^2}{1 - (1-\beta)^2} \\
            =& Var[\phi\left(h_{\w}(\zp)\right)] \cdot \frac{\beta}{2 - \beta}
        \end{aligned}
    \end{equation}
\end{proof}

\EstimatorUnbiasTwo*
\begin{proof}
    With sufficient large $\Np$ and $\Nn$, we consider $\S$ as the population. Given a threshold $c \in \mathbb{R}$, we consider  $\ell_1\left(c - h_{\w}(\bm{x})\right)$ as i.i.d. variables controlled by $\bm{x}$ with mean $\mu_{c,1}$ and variance $\kappa^2_{c,1}$.
    In this way, the surrogate $\widehat{FPR}$ can be viewed as an average of these variables:
    \begin{equation*}
        \begin{aligned}
            X^c_{\nn} = \hat{\Expt}_{\bm{x}\sim\zn}\left[\ell_1\left(c - h_{\w}(\bm{x})\right)\right].
        \end{aligned}
    \end{equation*}
    According to the central limit theorem, when $\nn\rightarrow \infty$ we have
    \begin{equation}
    \label{eq:temp21}
        X^c_{\nn} \rightsquigarrow \mathcal{N}(\mu_{c,1}, \kappa^2_{c,1} / \nn),
    \end{equation}
    where ${\rightsquigarrow}$ refers to convergence in law.
    Similarly, consider $\ell_2\left(c - v)\right)$ as variables with mean $\mu_{c,2}$ and standard deviation $\kappa^2_{c,2}$, and denote
    \begin{equation}
    \label{eq:temp22}
        Y^c_{\np} = \hat{\Expt}_{v\sim\bm{v}}\left[\ell_2\left(c - v\right)\right].
    \end{equation} 
    According to the portmanteau lemma \cite{van2000asymptotic} and \Propref{prop:intp_error}, when $\np \rightarrow \infty$, we have
    \begin{equation}
        Y^c_{\np} \rightsquigarrow \mathcal{N}(\mu_{c,2}, \kappa^2_{c,2} / \np).
    \end{equation}


    Decompose the difference between $\hat{f}(\w;\z)$ and $\widehat{\text{AUPRC}}^\downarrow(\w;\S)$ as 
    \begin{equation}
        \begin{aligned}
            &\mathop{\hat{\Expt}}\limits_{\z\subseteq\S}[\hat{f}(\w;\z)] - \widehat{\text{AUPRC}}^\downarrow(\w;\S) \\
            =& \mathop{\hat{\Expt}}\limits_{\z, c\sim h_{\w}(\zp)}\left[
                \frac{(1-\pi) X^c_{\nn}}{(1-\pi) X^c_{\nn} + \pi Y^c_{\np}}
            \right] - \mathop{\Expt}\limits_{c\sim h_{\w}(\Sp)}\left[
                \frac{(1-\pi) \mu_{c,1}}{(1-\pi) \mu_{c,1} + \pi \mu_{c,2}}
            \right] \\
            =& \underbrace{\mathop{\hat{\Expt}}\limits_{\z, c\sim h_{\w}(\zp)}\left[
                \frac{(1-\pi) X^c_{\nn}}{(1-\pi) {\color{red}X^c_{\nn}} + \pi Y^c_{\np}} - \frac{(1-\pi) X^c_{\nn}}{(1-\pi) {\color{red}\mu_{c,1}} + \pi Y^c_{\np}}
            \right]}_{{\color{blue}(a)}} \\
            &+ \underbrace{\mathop{\hat{\Expt}}\limits_{\z, c\sim h_{\w}(\zp)}\left[\frac{(1-\pi) {\color{red}X^c_{\nn}}}{(1-\pi) \mu_{c,1} + \pi \mu_{c,2}} - \frac{(1-\pi) {\color{red}\mu_{c,1}}}{(1-\pi) \mu_{c,1} + \pi \mu_{c,2}}
            \right]}_{{\color{blue}(b)}} \\
            &+ \underbrace{\mathop{\hat{\Expt}}\limits_{\z, c\sim h_{\w}(\zp)}\left[\frac{(1-\pi) X^c_{\nn}}{(1-\pi) \mu_{c,1} + \pi {\color{red}Y^c_{\np}}} - \frac{(1-\pi) X^c_{\nn}}{(1-\pi) \mu_{c,1} + \pi {\color{red}\mu_{c,2}}}
            \right]}_{{\color{blue}(c)}} \\
            &+ \underbrace{{\color{red}\mathop{\hat{\Expt}}\limits_{\z, c\sim h_{\w}(\zp)}}\left[\frac{(1-\pi) \mu_{c,1}}{(1-\pi) \mu_{c,1} + \pi \mu_{c,2}}
            \right] - {\color{red}\mathop{\Expt}\limits_{c\sim h_{\w}(\Sp)}}\left[
                \frac{(1-\pi) \mu_{c,1}}{(1-\pi) \mu_{c,1} + \pi \mu_{c,2}}
            \right]}_{{\color{blue}(d)}}.
        \end{aligned}
    \end{equation}

    Notice that $\zp$ is randomly sampled from $\Sp$, thus we have ${\color{blue}(d)} = 0$, and we only need to focus on the first three terms. 

    To solve {\color{blue}(a)}, for any threshold $c$, we consider a continuous function $s(;c)$:
    \begin{equation*}
        s(x;c) = \frac{(1-\pi) X^c_{\nn}}{(1-\pi) {\color{red}x} + \pi Y^c_{\np}}.
    \end{equation*}
    With the delta method \cite{van2000asymptotic}, when $\nn \rightarrow \infty$ we have
    \begin{equation}
        s\left(X^c_{\nn};c\right) - s\left(\mu_{c,1};c\right) \rightsquigarrow \mathcal{N}(0, (s'(\mu_{c,1}))^2\kappa^2_{c,1} / \nn).
    \end{equation}

    Notice that the above result holds for all $\z$ and $c$. Consider a random variable $\bm{X}_{\nn} \in \mathbb{R}^{M\np}$ whose elements are $X^c_{\nn}$ w.r.t. all $c$, and a function $Avg: \mathbb{R}^{M\np} \mapsto \mathbb{R}$ outputs the average of the input. Obviously, the gradient of $Avg$ at any point is a scaled unit vector. Therefore, according to the delta method, we have the following convergence:
    \begin{equation}
    \label{eq:temp23}
        {\color{blue}(a)} = \mathop{\hat{\Expt}}\limits_{\z, c\sim h_{\w}(\zp)}\left[s\left(X^c_{\nn};c\right) - s\left(\mu_{c,1};c\right)\right] \rightsquigarrow \mathcal{N}(0, \left\|\Sigma\right\|_1 / (M\np)^2),
    \end{equation}
    where $\Sigma = Var[\bm{X}_{\nn}]$. 
    Denote
    \begin{equation*}
        t^2_{1} = \mathop{\hat{\Expt}}\limits_{\z, c\sim h_{\w}(\zp)}\left[(s'(\mu_{c,1}))^2\kappa^2_{c,1}\right], ~~\kappa^2_{1} = \mathop{\hat{\Expt}}\limits_{\z, c\sim h_{\w}(\zp)}\left[\kappa^2_{c,1}\right], ~~ \kappa^2_{2} = \mathop{\hat{\Expt}}\limits_{\z, c\sim h_{\w}(\zp)}\left[\kappa^2_{c,2}\right]
    \end{equation*}
    then we have 
    \begin{equation}
        (t^2_{1} / \nn) / (M\np) \leq \left\|\Sigma\right\|_1 / (M\np)^2 \leq t^2_{1} / \nn.
    \end{equation}
    According to the mean value theorem, there exists a positive constant $\zeta_a$ such that 
    \begin{equation}
        \left\|\Sigma\right\|_1 / (M\np)^2 = \zeta_a \cdot \kappa^2_{1} / \nn,
    \end{equation}
    and \eqref{eq:temp23} can be rewritten as
    \begin{equation}
        {\color{blue}(a)} = \mathop{\hat{\Expt}}\limits_{\z, c\sim h_{\w}(\zp)}\left[s\left(X^c_{\nn};c\right) - s\left(\mu_{c,1};c\right)\right] \rightsquigarrow \mathcal{N}(0, \zeta_a \cdot \kappa^2_{1} / \nn).
    \end{equation}
    Similarly, there exists $\zeta_b, \zeta_c > 0$, such that
    \begin{equation}
        \begin{aligned}
            {\color{blue}(b)} \rightsquigarrow \mathcal{N}\left(0, ~\zeta_b \cdot \kappa^2_1 / \nn\right), \\
            {\color{blue}(c)} \rightsquigarrow \mathcal{N}\left(0, ~\zeta_c \cdot \kappa^2_2 / \np \right).
        \end{aligned}
    \end{equation}
    To sum up, when $\mathop{\hat{\Expt}}_{c\sim h_{\w}(\zp)}\left[\kappa^2_{c,1}\right] / \nn \rightarrow 0$ and $\mathop{\hat{\Expt}}_{c\sim h_{\w}(\zp)}\left[\kappa^2_{c,2}\right] / \np \rightarrow 0$, we have
    \begin{equation}
        {\color{blue}(a)} + {\color{blue}(b)} + {\color{blue}(c)} \overset{P}{\rightarrow} 0,
    \end{equation}
    where $\overset{P}{\rightarrow}$ refers to convergence in probability.

    As for $\widehat{AP}^\downarrow$, the difference can be decomposed into
    \begin{equation}
        \begin{aligned}
            &\mathop{\hat{\Expt}}\limits_{\z\subseteq\S}\left[\widehat{AP}^\downarrow(\w;\z)\right] - \widehat{\text{AUPRC}}^\downarrow(\w;\S) \\
            =& \mathop{\hat{\Expt}}\limits_{\z, c\sim h_{\w}(\zp)}\left[
                \frac{(1-\pi_0) X^c_{\nn}}{(1-\pi_0) X^c_{\nn} + \pi_0 Y^c_{\np}}
            \right] - \mathop{\Expt}\limits_{c\sim h_{\w}(\Sp)}\left[
                \frac{(1-\pi) \mu_{c,1}}{(1-\pi) \mu_{c,1} + \pi \mu_{c,2}}
            \right] \\
            =& \underbrace{\mathop{\hat{\Expt}}\limits_{\z, c\sim h_{\w}(\zp)}\left[
                \frac{(1-\pi_0) X^c_{\nn}}{(1-\pi_0) X^c_{\nn} + \pi_0 Y^c_{\np}}
            \right] - \mathop{\Expt}\limits_{c\sim h_{\w}(\Sp)}\left[
                \frac{(1-\pi_0) \mu_{c,1}}{(1-\pi_0) \mu_{c,1} + \pi_0 \mu_{c,2}}
            \right]}_{{\color{blue}(e)}} \\
            &+ \underbrace{\mathop{\Expt}\limits_{c\sim h_{\w}(\Sp)}\left[
                \frac{(1-\pi_0) \mu_{c,1}}{(1-\pi_0) \mu_{c,1} + \pi_0 \mu_{c,2}}
            \right] - \mathop{\Expt}\limits_{c\sim h_{\w}(\Sp)}\left[
                \frac{(1-\pi) \mu_{c,1}}{(1-\pi) \mu_{c,1} + \pi \mu_{c,2}}
            \right]}_{{\color{blue}(f)}}.
        \end{aligned}
    \end{equation}
    Similarly,  we have ${\color{blue}(e)} \overset{P}{\rightarrow} 0$. As for ${\color{blue}(f)}$, we have 
    \begin{equation}
        {\color{blue}(f)} = (\pi_0-\pi) \cdot \mathop{\Expt}\limits_{c\sim h_{\w}(\Sp)}\left[
            \frac{\mu_{c,2} / (1-\pi)}{(1-\pi_0) \mu_{c,1} + \pi_0 \mu_{c,2}} \cdot \frac{(1-\pi)\mu_{c,1}}{(1-\pi) \mu_{c,1} + \pi \mu_{c,2}}
        \right].
    \end{equation}
    According to the mean value theorem, there exists 
    \begin{equation}
        H \in \left[\min_{\bm{x}\sim\Sp}\frac{\mu_{c,2} / (1-\pi)}{(1-\pi_0) \mu_{c,1} + \pi_0 \mu_{c,2}}, \max_{\bm{x}\sim\Sp}\frac{\mu_{c,2} / (1-\pi)}{(1-\pi_0) \mu_{c,1} + \pi_0 \mu_{c,2}}\right],
    \end{equation}
    such that
    \begin{equation}
        {\color{blue}(f)} = (\pi_0-\pi) H \cdot \mathop{\Expt}\limits_{c\sim h_{\w}(\Sp)}\left[\frac{(1-\pi)\mu_{c,1}}{(1-\pi) \mu_{c,1} + \pi \mu_{c,2}}
        \right].
    \end{equation}

\end{proof}

\begin{prop}
    For any $0 < \delta < 1$, at least with probability of $1 - \delta$, we have
    \begin{equation*}
        \left|\mathop{\hat{\mathbb{E}}}\limits_{\pmb{z}\subseteq\mathcal{S}}[\hat{f}(\pmb{w};\pmb{z})] - \widehat{\text{AUPRC}}^\downarrow(\pmb{w};\mathcal{S})\right| = \mathcal{O}\left(\sqrt{\frac{\log (6{n^+} / \delta)}{{n^+}}} + 2\sqrt{\frac{\log (6{n^+} / \delta)}{{n^-}}}\right).      
    \end{equation*}
\end{prop}

\begin{proof}
    Denote $X^c_{{n^-}}$, $Y^c_{{n^+}}$, ${\color{blue}(a)}, {\color{blue}(b)}, {\color{blue}(c)}$ as in the proof of \Propref{prop:est_error}. 
Next, we focus on the term $(a)$. Under the assumption of \Propref{prop:est_error}, $X^c_{{n^-}}$ can be viewed as an average of i.i.d. variables, thus according to Hoeffding's inequality, for any $\epsilon > 0$ we have
$$
    \mathbb{P}\left(\left|X^c_{{n^-}} - \mu_{c,1}\right| \geq \epsilon\right) \leq 2\exp\left(-\frac{2{n^-} \epsilon^2}{B_{\ell_1}^2}\right).
$$
Therefore, for any $c$, $0 < \delta < 1$, with probability at least $1 - \delta$,  we have
$$
\begin{aligned}
    &\left|\frac{(1-\pi) X^c_{{n^-}}}{(1-\pi) X^c_{{n^-}} + \pi Y^c_{{n^+}}} - \frac{(1-\pi) X^c_{{n^-}}}{(1-\pi) \mu_{c,1} + \pi Y^c_{{n^+}}}\right| \\\\
    \leq& \left|\frac{(1-\pi)^2 X^c_{{n^-}}}{\left((1-\pi) X^c_{{n^-}} + \pi Y^c_{{n^+}}\right)\left((1-\pi) \mu_{c,1} + \pi Y^c_{{n^+}}\right)}\right|\cdot \left|X^c_{{n^-}} - \mu_{c,1}\right| \\\\
    \leq& \frac{1}{\mu_{c,1}} \cdot \left|X^c_{{n^-}} - \mu_{c,1}\right| \\\\
    \leq& \sqrt{\frac{B_{\ell_1}^2 \log \frac{2}{\delta}}{2\mu_{c,1}^2 {n^-}}}\\\\
    \leq& \sqrt{\frac{B_{\ell_1}^2 \log \frac{2}{\delta}}{2\mu_{1}^2 {n^-}}},
\end{aligned}
$$
where $\mu_1 = \inf_{c}\ \mu_{c,1}$. If we further assume that $X^c_{{n^-}}$ is independent w.r.t. different $c$, then considering all positive $c$, with probability at least $1 - \delta / 3$, we have 
$$
    |{\color{blue}(a)}| \leq \sqrt{\frac{B_{\ell_1}^2 \log \frac{6{n^+}}{\delta}}{2\mu_{1}^2 {n^-}}} = \mathcal{O}\left(\sqrt{\frac{\log (6{n^+} / \delta)}{{n^-}}}\right),
$$
and similarly
$$
    |{\color{blue}(b)}| = \mathcal{O}\left(\frac{\log (6{n^+} / \delta)}{{n^-}}\right), ~~ |{\color{blue}(c)}| = \mathcal{O}\left(\sqrt{\frac{\log (6{n^+} / \delta)}{{n^+}}}\right).
$$
To sum up, with probability at least $1 - \delta$ we have 
$$
    \left|\mathop{\hat{\mathbb{E}}}\limits_{\pmb{z}\subseteq\mathcal{S}}[\hat{f}(\pmb{w};\pmb{z})] - \widehat{\text{AUPRC}}^\downarrow(\pmb{w};\mathcal{S})\right| \leq |{\color{blue}(a)}| + |{\color{blue}(b)}| + |{\color{blue}(c)}| = \mathcal{O}\left(\sqrt{\frac{\log (6{n^+} / \delta)}{{n^+}}} + 2\sqrt{\frac{\log (6{n^+} / \delta)}{{n^-}}}\right).
$$
\end{proof}

\subsection{Details of Simulation Experiments}
\label{app:error_simu}
Following \cite{boyd2013area}, assume that scores are i.i.d. random variables drawn from one of three distributions: binormal, bibeta and uniform offset. The hyperparameters choosen and the corresponding probability density functions are shown in \Fgref{fig:simu_density_full}. For each kind of distribution, we draw $100,000$ score points, where the ratio of positive points $\pi = 0.1$. Afterward, under different sampling-rate $\pi_0 \in \{0.01, 0.02, 0.03, 0.1, 0.2\}$, we repeatly sample several points from the above-mentioned points for $500$ times, and report the corresponding mean value and standard derivation. The results are shown in \Fgref{fig:simu_ap_error_full} and \Fgref{fig:simu_intp_error_full}, which are consistent with our conclusions in the main paper.

\begin{figure}[H]
    \centering
    \includegraphics[scale=0.45]{density_binormal.pdf}
    \includegraphics[scale=0.45]{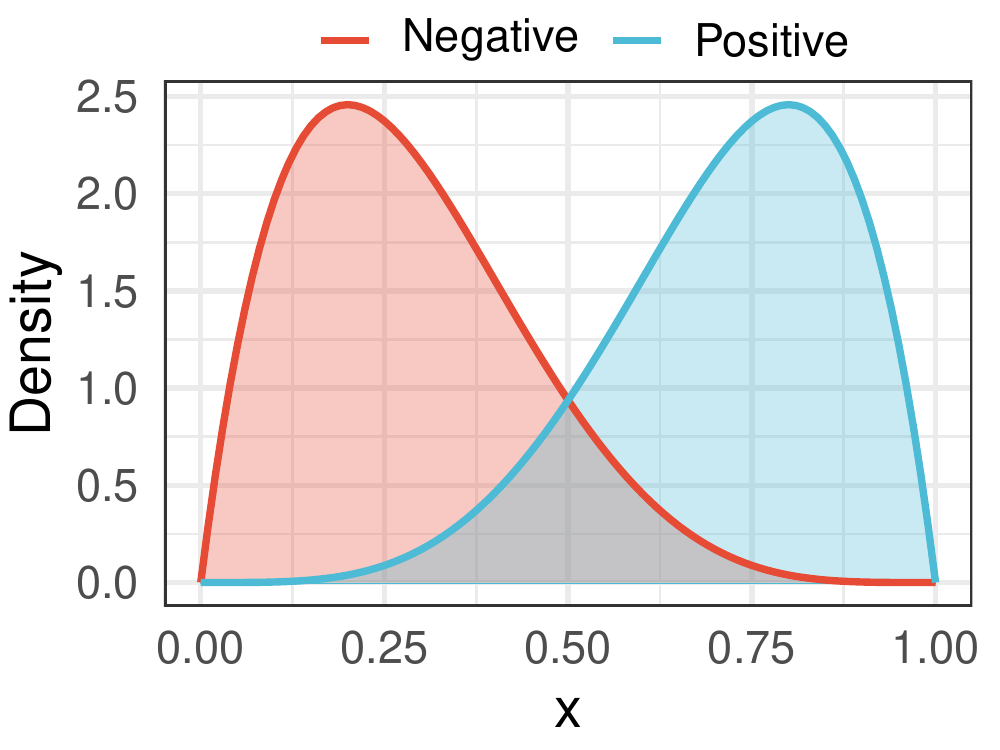}
    \includegraphics[scale=0.45]{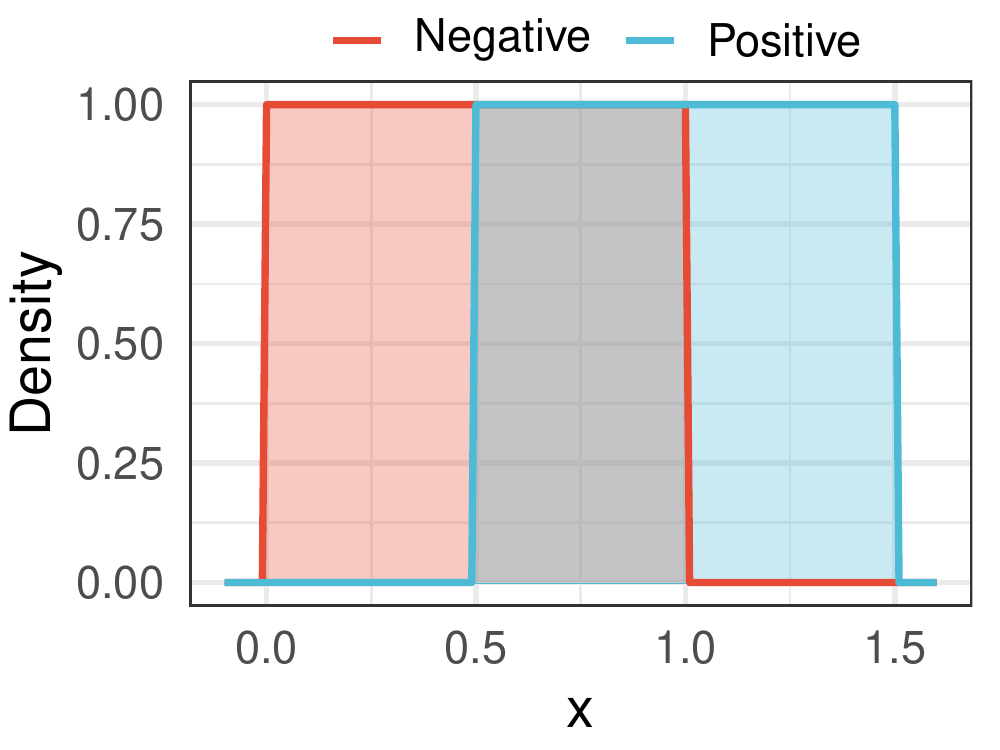}
	\caption{Probability density functions of positive scores $s^+$ and negative scores $s^-$. Left: binormal $\left(s^-\sim \mathcal{N}(0,1), s^+\sim \mathcal{N}(1,1)\right)$. Middle: bibeta $\left(s^-\sim \mathcal{B}(2,5), s^+\sim \mathcal{B}(5,2)\right)$. Right: offset uniform $\left(s^-\sim \mathcal{U}(0,1), s^+\sim \mathcal{U}(0.5,1.5)\right)$.}
	\label{fig:simu_density_full}
\end{figure}

\begin{figure}[h]
    \centering
    \vspace{-6mm}
    \subfigure[$\pi_0 = 0.02$]{
        \begin{minipage}[b]{0.28\textwidth}
            \includegraphics[scale=0.43]{ap_binormal_002.pdf}
            \includegraphics[scale=0.43]{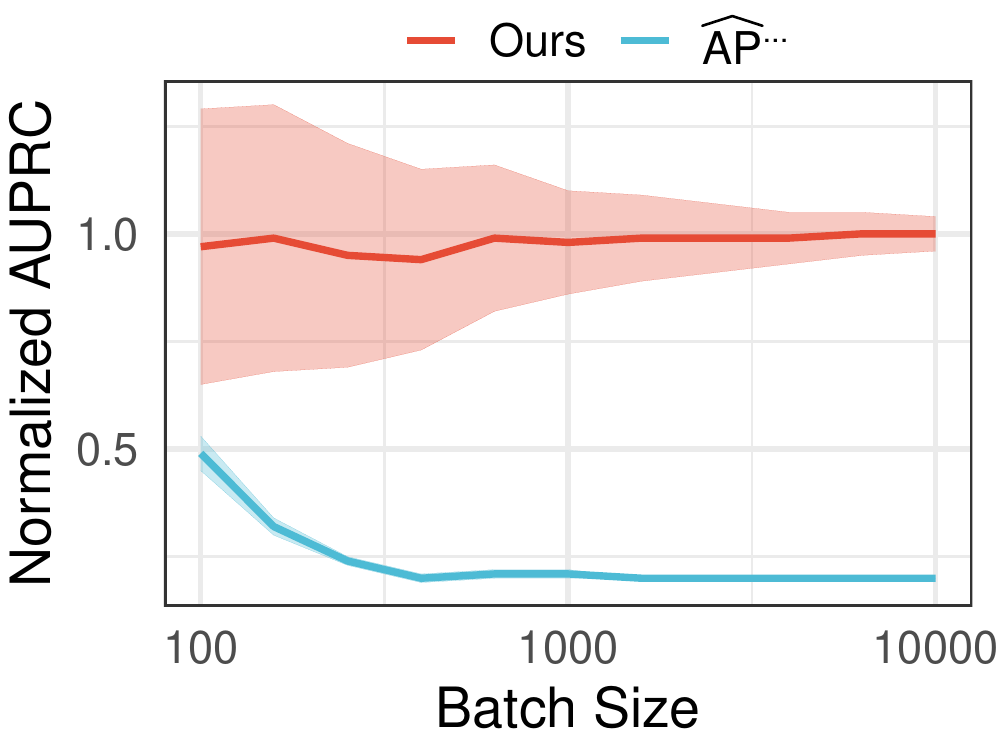}
            \includegraphics[scale=0.43]{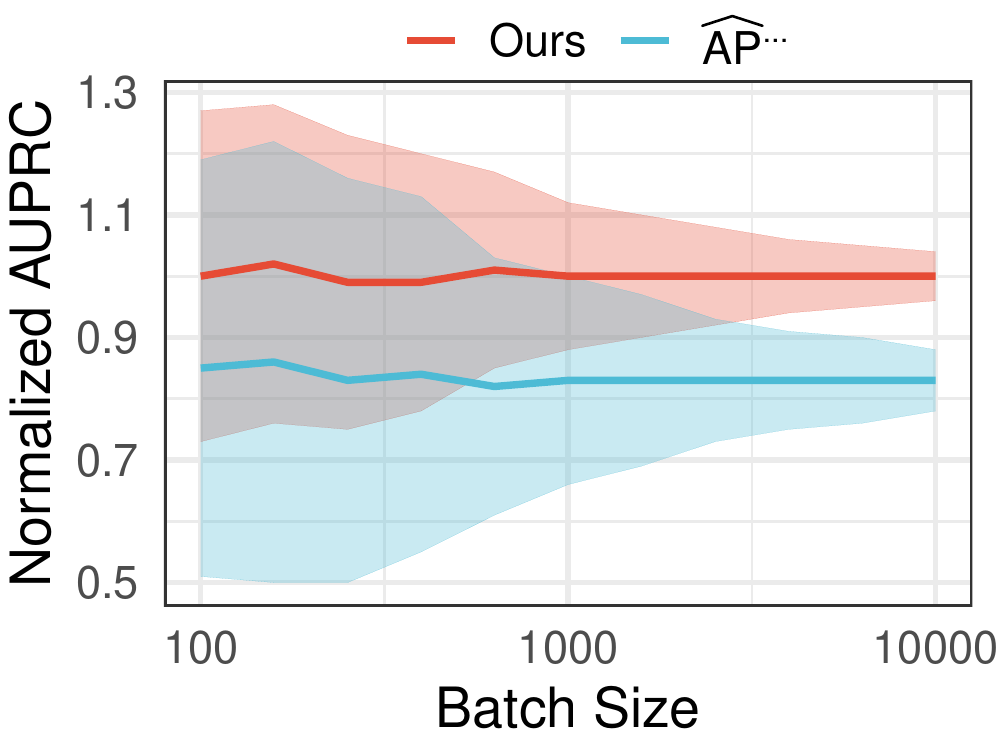}
        \end{minipage}
	}
    \hspace{2mm}
    \subfigure[$\pi_0 = 0.1$]{
        \begin{minipage}[b]{0.28\textwidth}
            \includegraphics[scale=0.43]{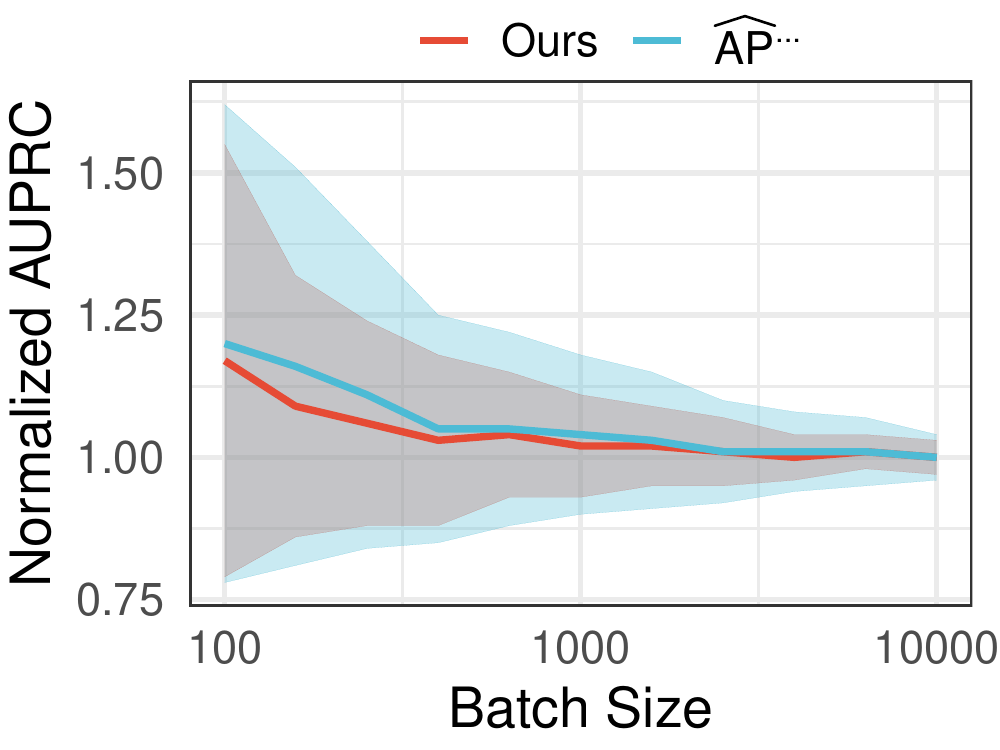}
            \includegraphics[scale=0.43]{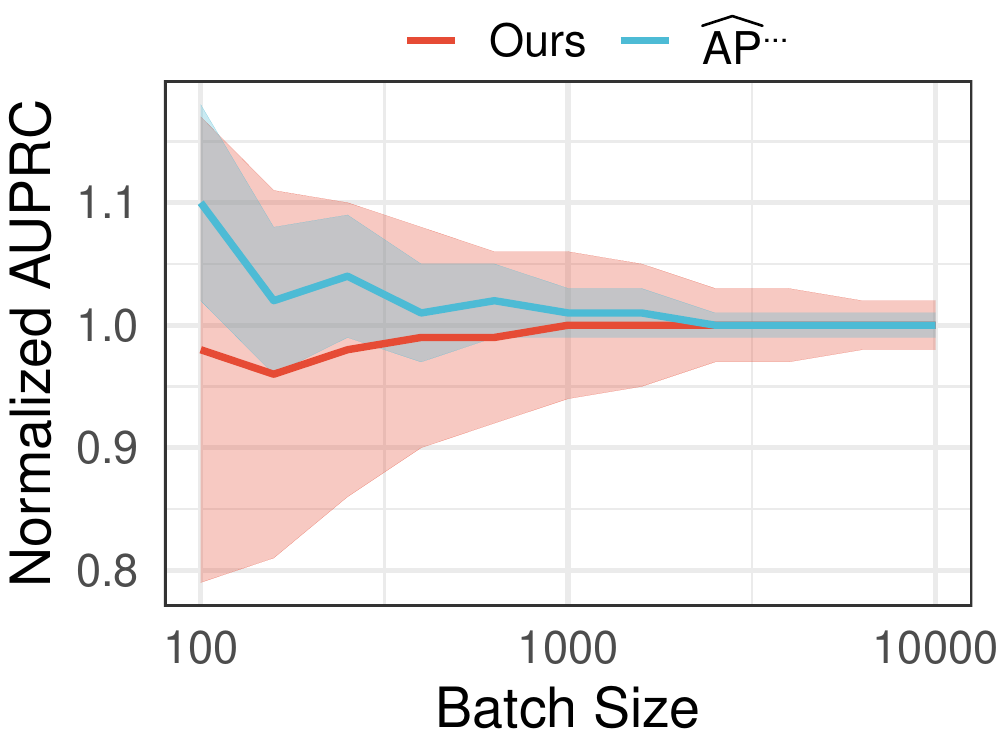}
            \includegraphics[scale=0.43]{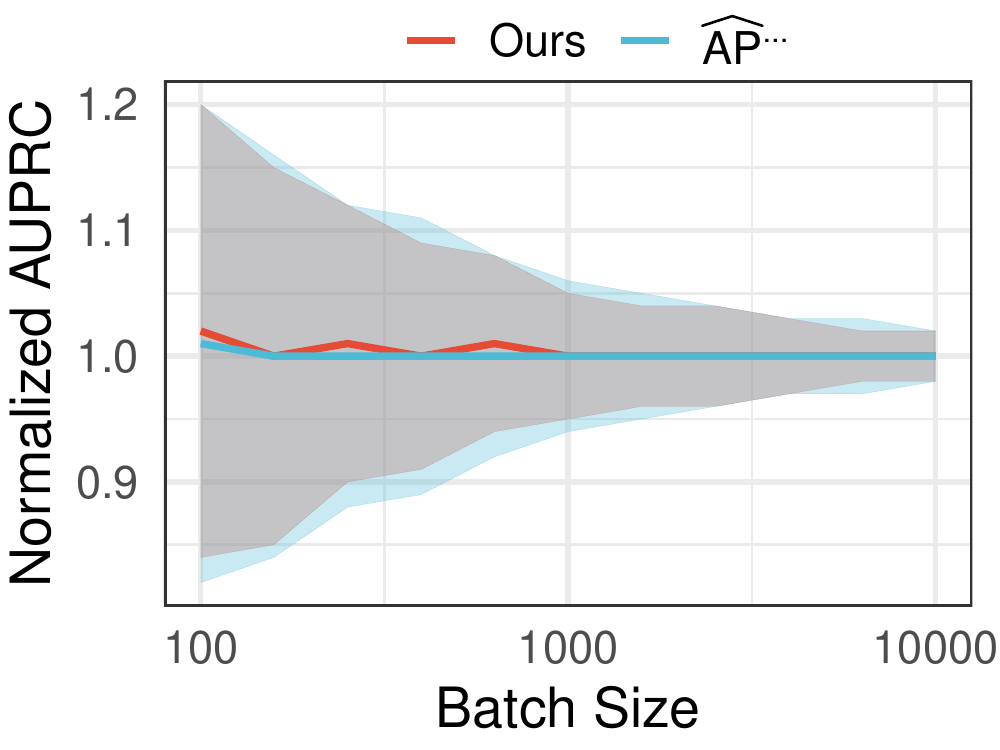}
        \end{minipage}
	}
    \hspace{2mm}
    \subfigure[$\pi_0 = 0.2$]{
        \begin{minipage}[b]{0.28\textwidth}
            \includegraphics[scale=0.43]{ap_binormal_020.pdf}
            \includegraphics[scale=0.43]{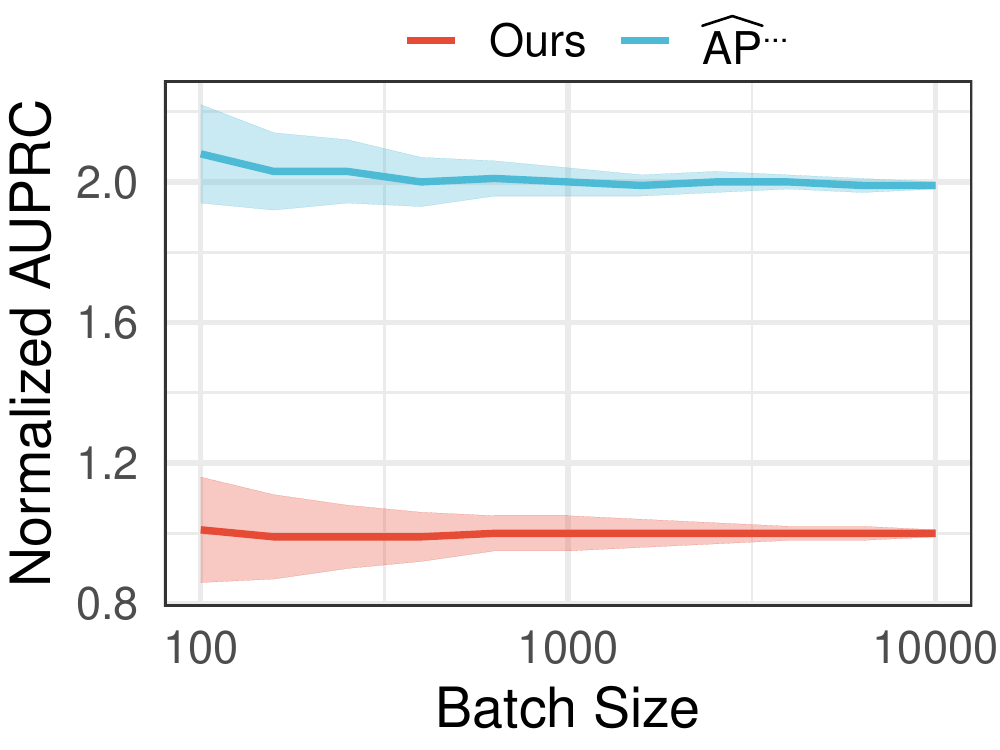}
            \includegraphics[scale=0.43]{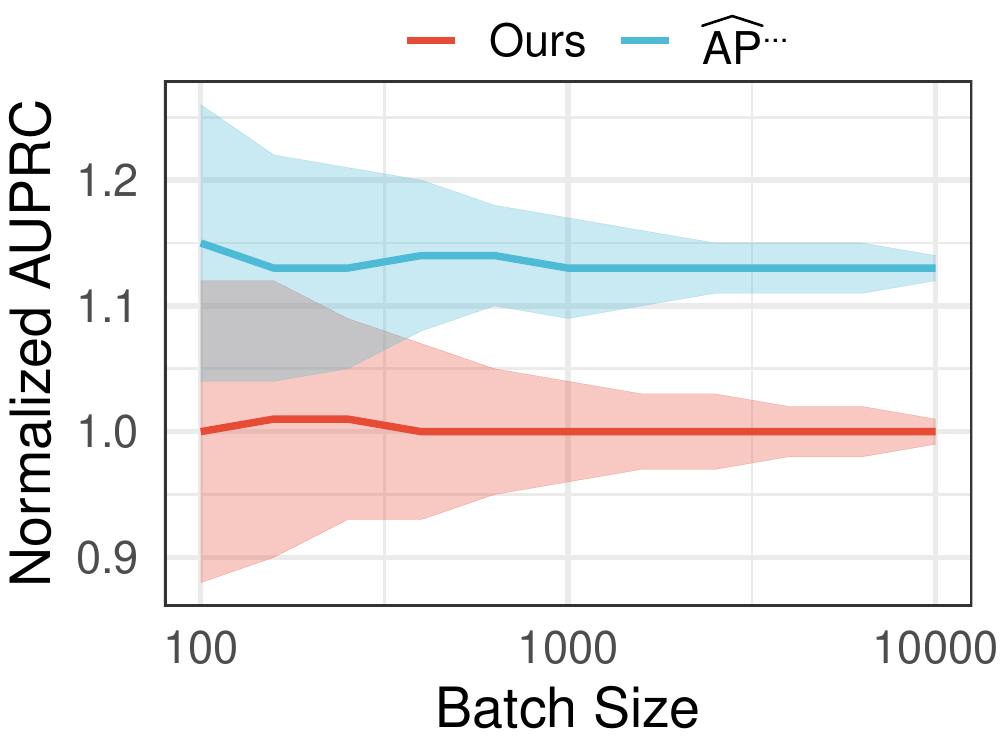}
        \end{minipage}
	}
	\caption{Empirical analysis of estimation errors on simulation data. The score distributions of each row are binormal, bibeta and uniform offset in turn.}
	\label{fig:simu_ap_error_full}
\end{figure}

\begin{figure}[H]
    \centering
    \subfigure[$\pi_0 = 0.01$]{
        \begin{minipage}[b]{0.28\textwidth}
            \includegraphics[scale=0.43]{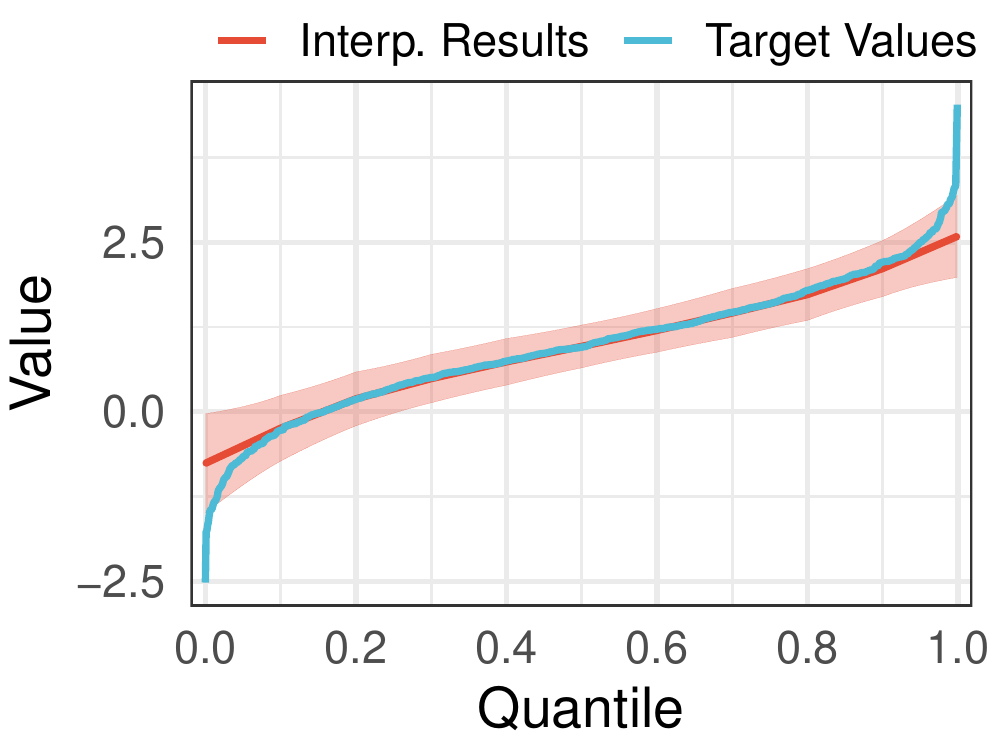}
            \includegraphics[scale=0.43]{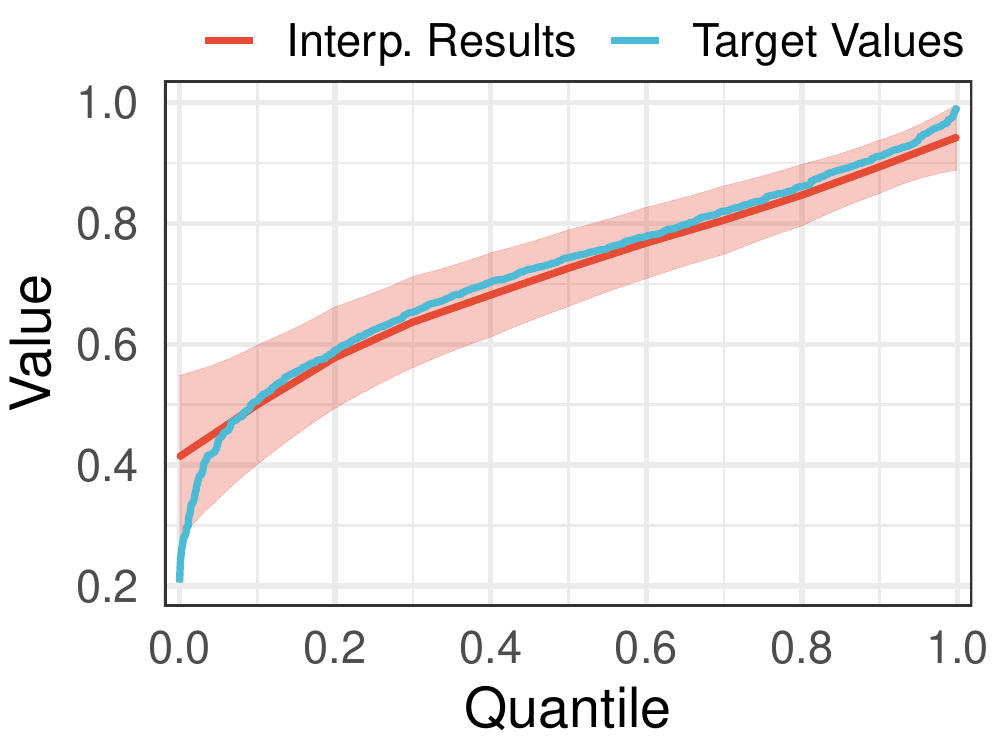}
            \includegraphics[scale=0.43]{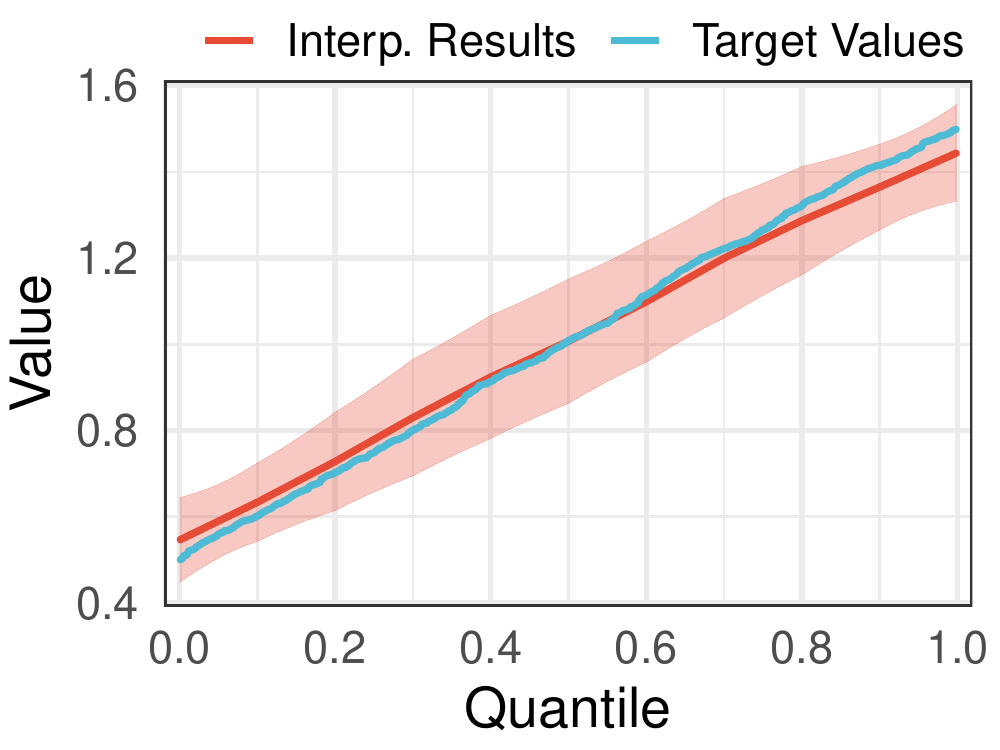}
        \end{minipage}
	}
    \hspace{2mm}
    \subfigure[$\pi_0 = 0.03$]{
        \begin{minipage}[b]{0.28\textwidth}
            \includegraphics[scale=0.43]{intp_binormal_003.pdf}
            \includegraphics[scale=0.43]{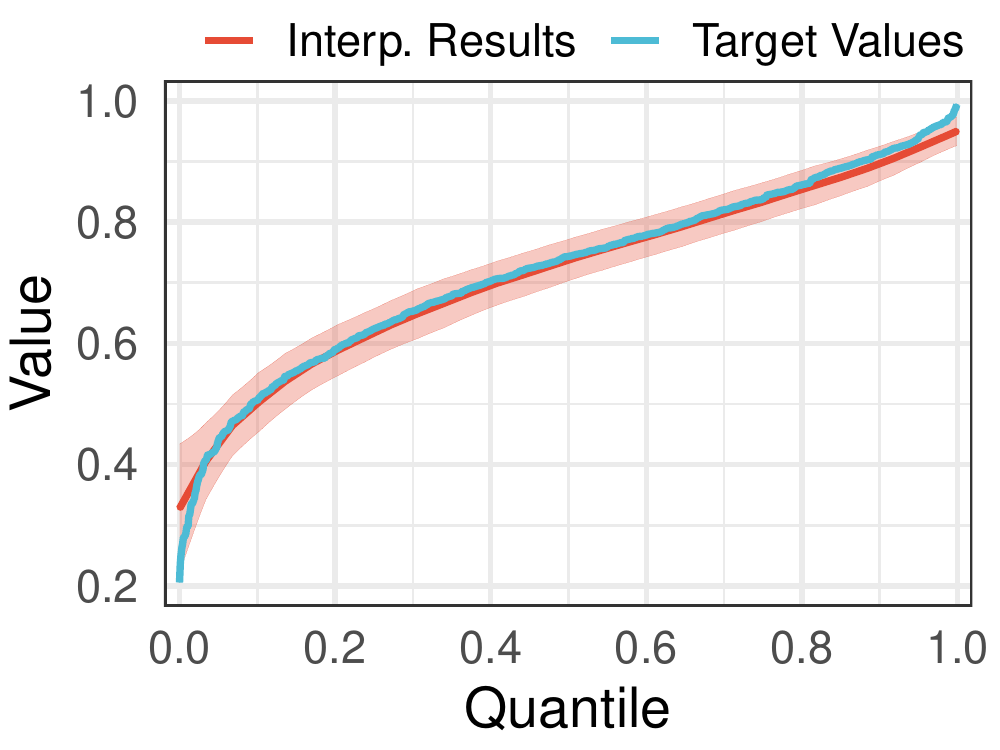}
            \includegraphics[scale=0.43]{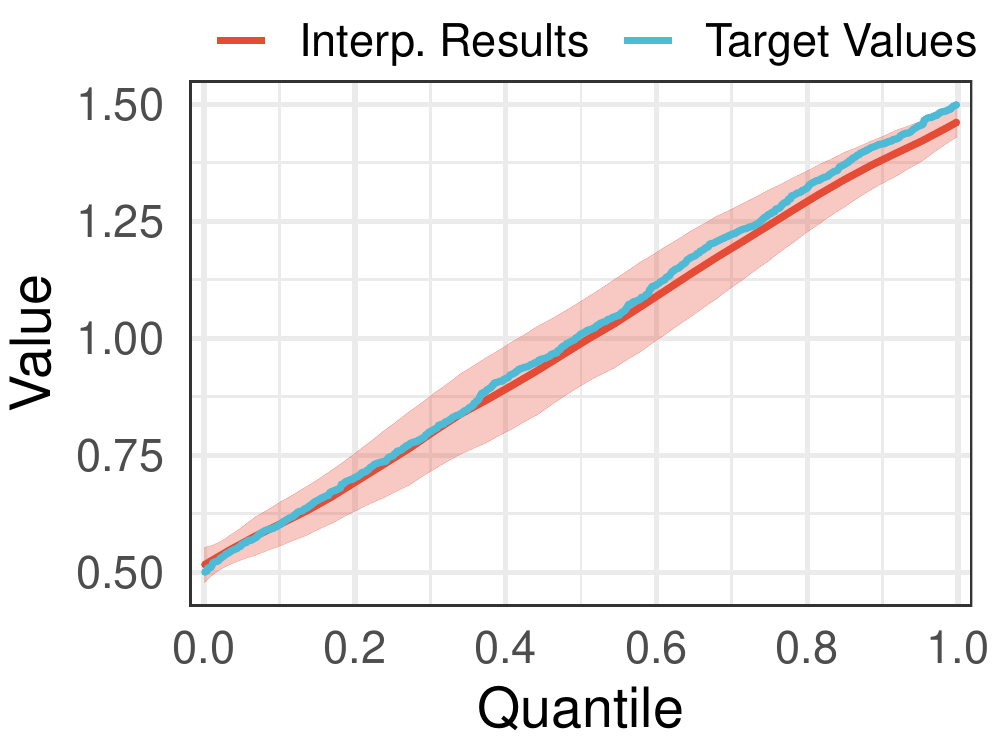}
        \end{minipage}
	}
    \hspace{2mm}
    \subfigure[$\pi_0 = 0.1$]{
        \begin{minipage}[b]{0.28\textwidth}
            \includegraphics[scale=0.43]{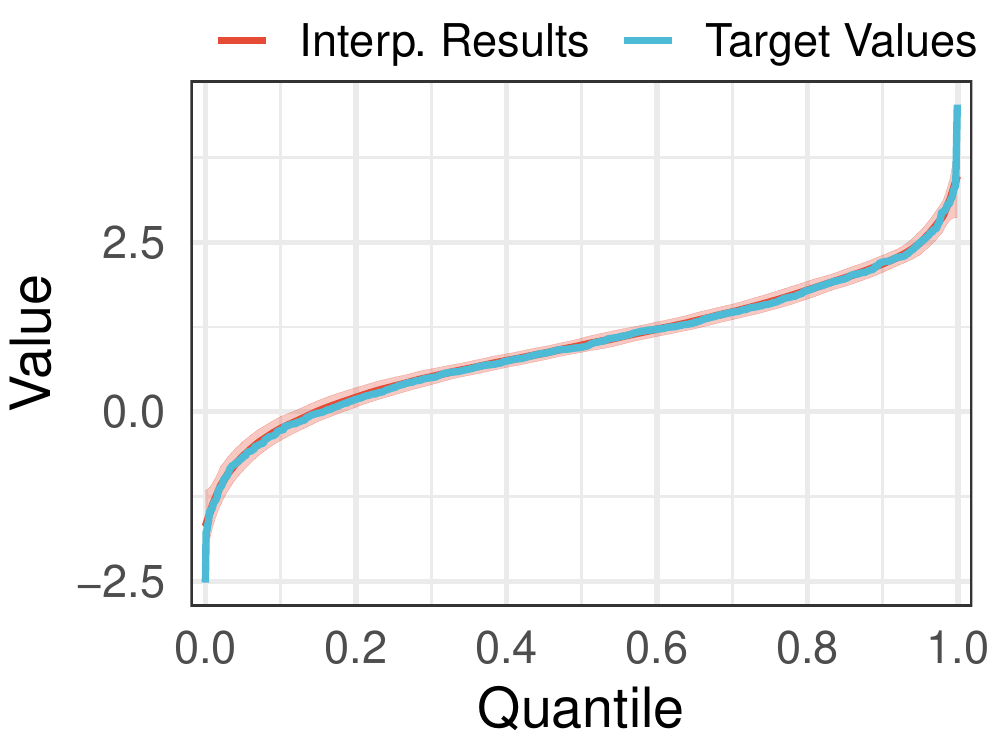}
            \includegraphics[scale=0.43]{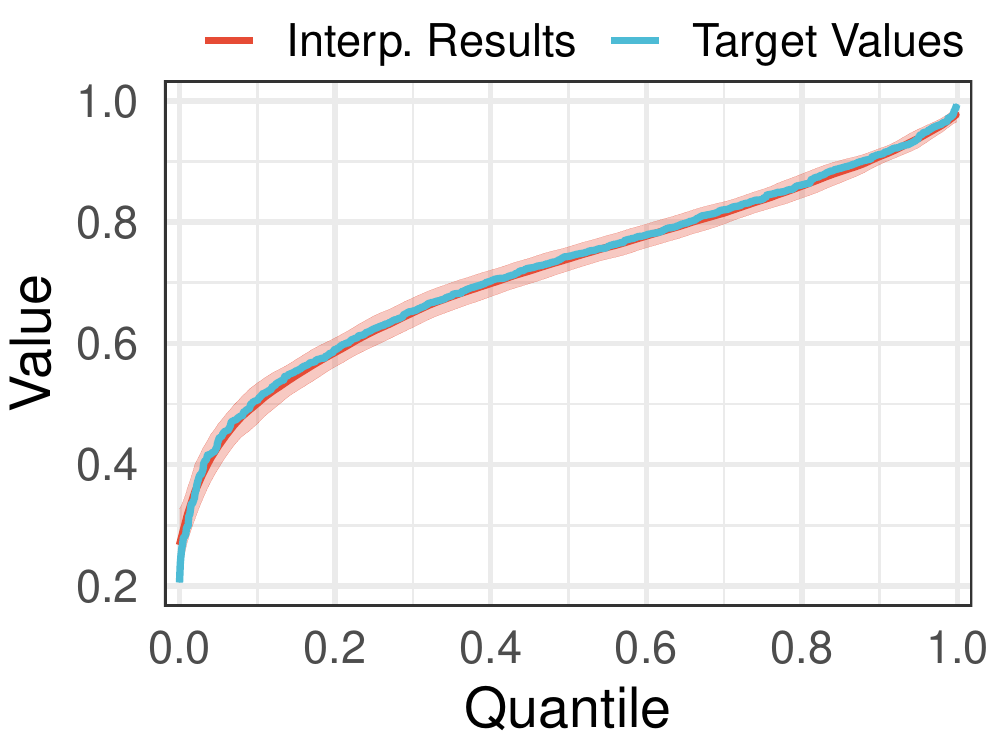}
            \includegraphics[scale=0.43]{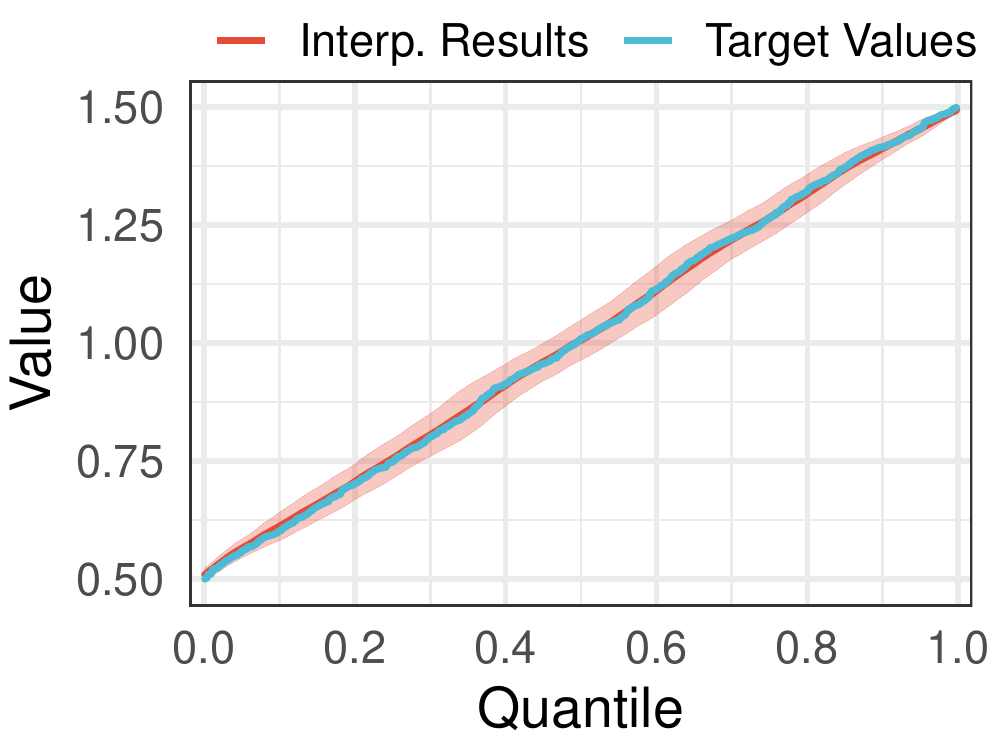}
        \end{minipage}
	}
	\caption{Empirical analysis of interpolation errors on simulation data. The score distributions of each row are binormal, bibeta and uniform offset in turn.}
	\label{fig:simu_intp_error_full}
\end{figure}

\section{Proofs on Generalization via Stability}
\label{app:gen}
\subsection{Generalization by List Model Stability}
\GenViaStab*
\begin{proof}
    Let $\S,\widetilde{\S}$ be defined as in \Defiref{defi:stab}. First of all, according to the symmetry, we have
    \begin{equation}
        \begin{aligned}
            &\Expt_{\S,A}[F(A(\S)) - F(A(\S);\S)] \\ 
            =& \Expt_{\S,\widetilde{\S},A}[F(A(\S);\widetilde{\S}) - F(A(\S);\S)] \\
            =& \Expt_{\S,\widetilde{\S},A}[F(A(\widetilde{\S});\S) - F(A(\S);\S)] \\
            =& \frac{1}{M}\Expt_{\S,\widetilde{\S},A}\left[\sum_{\z\sim\S}\left(\hat{f}(A(\widetilde{\S});\z) - \hat{f}(A(\S);\z)\right)\right] \\
            =& \Expt_{\S,\widetilde{\S},\z,\tilde{\z},A}\left[\hat{f}(A((\S - \z)\cup\tilde{\z});\z) - \hat{f}(A(\S);\z)\right],
        \end{aligned}
    \end{equation}
    where $\z \subseteq \S$ and $\tilde{\z} \subseteq \widetilde{\S}$.
    Then, since the gradients of $\hat{f}$ is upper bounded by $G$, we have
    \begin{equation}
        \begin{aligned}
            &\Expt_{\S,\widetilde{\S},\z,\tilde{\z},A}\left[\left|\hat{f}(A((\S - \z)\cup\tilde{\z});\z) - \hat{f}(A(\S);\z)\right|\right] \\
            \leq& G \cdot \Expt_{\S,\widetilde{\S},\z,\tilde{\z},A}\left[\left\|A((\S - \z)\cup\tilde{\z}) - A(\S)\right\|_2\right] \\
            \leq& G\np \cdot \Expt_{\S,\S^{(i)},A}\left[\left\|A(S^{(i)}) - A(\S)\right\|_2 \big| y_i=1\right] \\
            &~~+ G\nn \cdot \Expt_{\S,\S^{(i)},A}\left[\left\|A(S^{(i)}) - A(\S)\right\|_2 \big| y_i=-1\right]\\
            \leq& \frac{G\np}{\Np} \Expt_{\S,\widehat{\S},A}\left[\sum_{y_i=1}\left\|A(S^{(i)}) - A(\S)\right\|_2\right] + \frac{G\nn}{\Nn}\Expt_{\S,\widetilde{\S},A}\left[\sum_{y_i=-1}\left\|A(S^{(i)}) - A(\S)\right\|_2\right].
        \end{aligned}
    \end{equation}
    The proof is completed by denoting
    \begin{equation}
        \begin{aligned}
            \Expt_{\S,\widetilde{\S},A}\left[\frac{1}{\Np}\sum_{y_i=1}\left\|A(S^{(i)}) - A(\S)\right\|_2\right] = \epsilon^+,\\
            \Expt_{\S,\widetilde{\S},A}\left[\frac{1}{\Nn}\sum_{y_i=-1}\left\|A(S^{(i)}) - A(\S)\right\|_2\right] = \epsilon^-.
        \end{aligned}
    \end{equation}
\end{proof}

\subsection{List Model Stability of SOPRC}
\label{app:stability}
\ModelStab*
\begin{proof}
    For all $i=1,\cdots,N$, let $\{\w_{t}\}_{t}, \{\bm{v}_{t}\}_{t}, \{\z_{i_t}\}_{t}$ and $\{\w'_{t}\}_{t}, \{\bm{v}'_{t}\}_{t}, \{\z'_{i_t}\}_{t}$ be produced by \Algref{alg:main} based on $\S$ and $\S^{(i)}$, respectively. According to the update rules, for all $t = 1,\cdots,T$, we have
    \begin{equation}
    \label{eq:temp3}
        \begin{aligned}        
            \|\w'_{t+1} - \w_{t+1}\|_2 
            &= \|\w'_{t} - \w_{t} - \eta_t\left(\nabla f(\w'_t;\z'_{i_t}, \bm{v}'_{t+1}) - \nabla f(\w_t;\z_{i_t}, \bm{v}_{t+1}) \right)\|_2 \\
            &\leq \|\w'_{t} - \w_{t}\|_2 + \eta_t\left\|\nabla f(\w'_t;\z'_{i_t}, \bm{v}'_{t+1}) - \nabla f(\w_t;\z_{i_t}, \bm{v}_{t+1}) \right\|_2.
        \end{aligned}
    \end{equation}
    We discuss the second term on the right side in the following three cases.

    \textbf{Case 1:} $(\bm{x}_i, y_i) \notin \z_{i_t}$. In this case, we have $\z_{i_t} = \z'_{i_t}$.
    According to the smoothness of $f$ and \Lemref{lem:smooth}, we have
    \begin{equation}
    \label{eq:temp1}
        \begin{aligned}
            &\left\|\nabla f(\w'_t;\z'_{i_t}, \bm{v}'_{t+1}) - \nabla f(\w_t;\z_{i_t}, \bm{v}_{t+1}) \right\|_2 \\
            \leq& \left\|\nabla f(\w'_t;\z_{i_t}, \bm{v}'_{t+1}) - \nabla f(\w_t;\z_{i_t}, \bm{v}'_{t+1})\right\| \\ &~+ \left\|\nabla f(\w_t;\z_{i_t}, \bm{v}'_{t+1}) - \nabla f(\w_t;\z_{i_t}, \bm{v}_{t+1}) \right\|_2 \\
            \leq& L_w \|\w'_t - \w_t\|_2 + L_v \|\bm{v}'_{t+1} - \bm{v}_{t+1}\|_2 / \Np,
        \end{aligned}
    \end{equation}
    and 
    \begin{equation}
    \label{eq:temp2}
        \begin{aligned}
            &\|\bm{v}'_{t+1} - \bm{v}_{t+1}\|_2 \\
            \leq & (1-\beta_{t})\|\bm{v}'_{t} - \bm{v}_{t}\|_2 + \beta_{t} \|\phi\left(h_{\w'_t}(\zp_{i_t})\right) - \phi\left(h_{\w_t}(\zp_{i_t})\right)\|_2 \\
            \leq& (1-\beta_{t})\|\bm{v}'_{t} - \bm{v}_{t}\|_2 + C_{\phi}C_{h}\beta_t\|\w'_t - \w_t\|_2\\
            \leq& (1-\beta_t) \|\bm{v}'_{t} - \bm{v}_{t}\|_2 + \|\w'_t - \w_t\|_2 / \np,
        \end{aligned}
    \end{equation}
    where the last inequation is due to $\beta_{t} \leq \frac{C_{\phi}B}{\np}$.

    Combining \Eqref{eq:temp3}, \Eqref{eq:temp1} and \Eqref{eq:temp2}, we have 
    \begin{equation}
        \begin{aligned}
            \|\w'_{t+1} - \w_{t+1}\|_2 \leq & (1 + L_w\eta_t + L_v\eta_t / \np) \|\w'_t - \w_t\|_2 \\
            &+ L_v (1-\beta_t) \eta_t \|\bm{v}'_{t} - \bm{v}_{t}\|_2 / \Np
        \end{aligned}
    \end{equation}

    \textbf{Case 2:} $(\bm{x}_i, y_i) \in \z_{i_t}$ and $y_i = 1$. In this case, we have $\zn_{i_t} = \bm{z}'^-_{i_t}$, thus
    \begin{equation}
    \label{eq:temp5}
        \begin{aligned}
            &\left\|\nabla f(\w'_t;\z'_{i_t}, \bm{v}'_{t+1}) - \nabla f(\w_t;\z_{i_t}, \bm{v}_{t+1}) \right\|_2 \\
            \leq& G / \np +
            L_w \|\w'_t - \w_t\|_2 + L_v \|\bm{v}'_{t+1} - \bm{v}_{t+1}\|_2 / \Np,
        \end{aligned}
    \end{equation}
    and 
    \begin{equation}
    \label{eq:temp6}
        \begin{aligned}
            &\|\bm{v}'_{t+1} - \bm{v}_{t+1}\|_2 \\
            \leq & (1-\beta_{t})\|\bm{v}'_{t} - \bm{v}_{t}\|_2 + \beta_{t} \|\phi\left(h_{\w'_t}(\bm{z}'^+_{i_t})\right) - \phi\left(h_{\w_t}(\bm{z}'^+_{i_t})\right)\|_2 \\
            &+ \beta_{t} \|\phi\left(h_{\w_t}(\bm{z}'^+_{i_t})\right) - \phi\left(h_{\w_t}(\bm{z}^+_{i_t})\right)\|_2 \\
            \leq& (1-\beta_t)\|\bm{v}'_{t} - \bm{v}_{t}\|_2 + \|\w'_t - \w_t\|_2 / \np + 2 C_\phi B \beta_t \\ 
            \leq& (1-\beta_t)\|\bm{v}'_{t} - \bm{v}_{t}\|_2 + \|\w'_t - \w_t\|_2 / \np + 1 / \np,
        \end{aligned}
    \end{equation}
    where the last inequality is due to $\beta_t \leq 2C_\phi B / \np$.

    By combining \Eqref{eq:temp5} and \Eqref{eq:temp6}, we have
    \begin{equation}
        \begin{aligned}
            \|\w'_{t+1} - \w_{t+1}\|_2 \leq & (1 + L_w\eta_t + L_v\eta_t / \np) \|\w'_t - \w_t\|_2 \\
            &+ L_v(1-\beta_t)\eta_t \|\bm{v}'_{t} - \bm{v}_{t}\|_2 / \Np + (G + L_v / \Np)\eta_t / \np.
        \end{aligned}
    \end{equation}

    \textbf{Case 3:} $(\bm{x}_i, y_i) \in \z_{i_t}$ and $y_i = -1$. In this case, we have $\zp_{i_t} = \bm{z}'^+_{i_t}$, thus
    \begin{equation}
    \label{eq:temp4}
        \begin{aligned}
            &\left\|\nabla f(\w'_t;\z'_{i_t}, \bm{v}'_{t+1}) - \nabla f(\w_t;\z_{i_t}, \bm{v}_{t+1}) \right\|_2 \\
            \leq& 
            \left\|\nabla f(\w'_t;\z'_{i_t}, \bm{v}'_{t+1}) - \nabla f(\w'_t;\z_{i_t}, \bm{v}'_{t+1})\right\|_2 \\
            &~+ \left\|\nabla f(\w'_t;\z_{i_t}, \bm{v}'_{t+1}) - \nabla f(\w_t;\z_{i_t}, \bm{v}'_{t+1})\right\|_2 \\
            &~+ \left\|\nabla f(\w_t;\z_{i_t}, \bm{v}'_{t+1}) - \nabla f(\w_t;\z_{i_t}, \bm{v}_{t+1}) \right\|_2 \\
            \leq& B_\ell' \cdot \np / \nn +
            L_w \|\w'_t - \w_t\|_2 + L_v \|\bm{v}'_{t+1} - \bm{v}_{t+1}\|_2 / \Np,
        \end{aligned}
    \end{equation}
    where $B_\ell' = (1-\pi)B_\ell / \pi$.
    Similar to case 1, \Eqref{eq:temp2} still holds. Combining it with \Eqref{eq:temp4}, we have
    \begin{equation}
        \begin{aligned}
            \|\w'_{t+1} - \w_{t+1}\|_2 \leq & (1 + L_w\eta_t + L_v\eta_t / \np) \|\w'_t - \w_t\|_2 \\
            &+ L_v\eta_t \|\bm{v}'_{t} - \bm{v}_{t}\|_2 / \Np + B_\ell' \cdot \np / \nn
        \end{aligned}
    \end{equation}

    Next, we consider the expectation of $\|\w'_{t+1} - \w_{t+1}\|_2$ and $\|\bm{v}'_{t+1} - \bm{v}_{t+1}\|_2$ taking on $A$. Note that 
    \begin{equation}
        \begin{aligned}
            &\prob(\text{case 1}|y_i=1) = (\Np - \np) / \Np, &\prob(\text{case 3}|y_i=1) = \np / \Np, \\
            &\prob(\text{case 1}|y_i=-1) = (\Nn - \nn) / \Nn, &\prob(\text{case 2}|y_i=-1) = \nn / \Nn,
        \end{aligned}
    \end{equation}
    For $y_i=1$, we have 
    \begin{equation}
        \begin{aligned}
            \Expt_{A}\left[\|\w'_{t+1} - \w_{t+1}\|_2\right] 
            \leq & (1 + L_w\eta_t + L_v\eta_t / \np) \Expt_{A}\left[\|\w'_t - \w_t\|_2\right] \\
            &+ L_v(1-\beta_t)\eta_t \Expt_{A}\left[\|\bm{v}'_{t} - \bm{v}_{t}\|_2\right] / \Np + (G + L_v / \Np) / \Np, \\
            \Expt_{A}\left[\|\bm{v}'_{t+1} - \bm{v}_{t+1}\|_2\right]
            \leq & C_\phi B \beta_t \Expt_{A}\left[\|\w'_t - \w_t\|_2\right] +\Expt_{A}\left[\|\bm{v}'_{t} - \bm{v}_{t}\|_2\right] + 1 / \Np.
        \end{aligned}
    \end{equation}
    By setting $L \geq \max\{L_w, L_v / \np, C_\phi B, G / 2\}$ and $\Np > 2 \np$, the above result can be rewritten as
    \begin{equation}
    \begin{aligned}
    &\left[
        \begin{array}{c}
            \Expt_{A}\left[\|\w'_{t+1} - \w_{t+1}\|_2\right] \\
            \Expt_{A}\left[\|\bm{v}'_{t+1} - \bm{v}_{t+1}\|_2\right] \\
            1
        \end{array}
    \right] \\ 
    \leq &
    \left[
        \begin{array}{ccc}
            1 + 2L\eta_t & L(1-\beta_t) \eta_t / \Np & L\eta_t / \Np \\
            L\beta_t & 1 & 1 / \Np \\
            0 & 0 & 1
        \end{array}
    \right] \left[
        \begin{array}{c}
            \Expt_{A}\left[\|\w'_{t} - \w_{t}\|_2\right] \\
            \Expt_{A}\left[\|\bm{v}'_{t} - \bm{v}_{t}\|_2\right] \\
            1
        \end{array}
    \right]
    \end{aligned}
    \end{equation}

    Similarly, for $y_i=-1$, we have 
    \begin{equation}
        \begin{aligned}
            \Expt_{A}\left[\|\w'_{t+1} - \w_{t+1}\|_2\right] 
            \leq & (1 + L_w\eta_t + L_v\eta_t / \np) \Expt_{A}\left[\|\w'_t - \w_t\|_2\right] \\
            &+ L_v(1-\beta_t)\eta_t \Expt_{A}\left[\|\bm{v}'_{t} - \bm{v}_{t}\|_2\right] / \Np + B'_\ell \np / \Nn, \\
            \Expt_{A}\left[\|\bm{v}'_{t+1} - \bm{v}_{t+1}\|_2\right]
            \leq & C_\phi B \beta_t \Expt_{A}\left[\|\w'_t - \w_t\|_2\right] +\Expt_{A}\left[\|\bm{v}'_{t} - \bm{v}_{t}\|_2\right].
        \end{aligned}
    \end{equation}
    By setting $L \geq \max\{L_w, L_v / \np, C_\phi B, G, B'_\ell\}$, we have 
    \begin{equation}
        \begin{aligned}
        &\left[
            \begin{array}{c}
                \Expt_{A}\left[\|\w'_{t+1} - \w_{t+1}\|_2\right] \\
                \Expt_{A}\left[\|\bm{v}'_{t+1} - \bm{v}_{t+1}\|_2\right] \\
                1
            \end{array}
        \right] \\ 
        \leq &
        \left[
            \begin{array}{ccc}
                1 + 2L\eta_t & L_v(1 - \beta_t)\eta_t / \Nn & L\eta_t \cdot \np / \Nn \\
                L\beta_t & 1 & 0 \\
                0 & 0 & 1
            \end{array}
        \right] \left[
            \begin{array}{c}
                \Expt_{A}\left[\|\w'_{t} - \w_{t}\|_2\right] \\
                \Expt_{A}\left[\|\bm{v}'_{t} - \bm{v}_{t}\|_2\right] \\
                1
            \end{array}
        \right].
        \end{aligned}
        \end{equation}
\end{proof}

\begin{lem}
    Let $\bm{z}$, $\bm{z}' \in \mathcal{Z}^n$ be two datasets that differ by exact one example and \Asmpref{asm:bound_gradient} holds. For all $\w\in\Omega, \bm{v}\in \mathbb{R}^{\Np}$, the gradient difference of w.r.t. $\bm{z}$ and $\bm{z}'$ are upper bounded:

    (a) If $\bm{z}, \bm{z}'$ differ by a negative example, then
    \begin{equation}
        \begin{aligned}
            \|\nabla f(\w;\z'_{i_t}, v) - \nabla f(\w;\z_{i_t}, v)\|_2 \leq \frac{(1-\pi)B_\ell|\zp|}{\pi|\zn|}
        \end{aligned}
    \end{equation}

    (b) If $\bm{z}, \bm{z}'$ differ by a positive example, then
    \begin{equation}
        \begin{aligned}
            \|\nabla f(\w;\z'_{i_t}, v) - \nabla f(\w;\z_{i_t}, v)\|_2 \leq G / |\zp|
        \end{aligned}
    \end{equation}
\end{lem}
\begin{proof}
    The proof of (b) is obvious. For (a), let the different examples in $\bm{z}$ and $\bm{z}'$ be $\tilde{\bm{x}}$ and $\tilde{\bm{x}}'$.

    Denote $c(\bm{x};\bm{v},\ell) = \hat{\Expt}_{v\sim\bm{v}}[\ell(h_{\w}(\bm{x}) - v))]$, and rewrite $f$ as 
    \begin{equation}
        f(\w;\z,\bm{v}) = \hat{\Expt}_{\bm{x}\in\zp}\left[
            \sigma\left(
                \frac{1-\pi}{\pi}\cdot c(\bm{x}; h_{\w}(\zn),\ell_1) / c(\bm{x};v,\ell_2)
            \right)
        \right].
    \end{equation}
    \begin{equation}
        \begin{aligned}
            & \|\nabla f(\w;\z'_{i_t}, v) - \nabla f(\w;\z_{i_t}, v)\|_2 \\
            \leq& \mathop{\hat{\Expt}}\limits_{\bm{x}\in\zp} \left[\|
                \nabla \sigma \left(
                    (1-\pi) / \pi \cdot c(\bm{x}; h_{\w}(\zn),\ell_1) / c(\bm{x}; \bm{v},\ell_2)
                \right)\right. \\
                &- \left.
                \nabla \sigma \left(
                    (1-\pi) / \pi \cdot c(\bm{x}; h_{\w}(\z'^{-}),\ell_1) / c(\bm{x}; \bm{v},\ell_2)
                \right)\|_2
            \right] \\
            \leq& \mathop{\hat{\Expt}}\limits_{\bm{x}\in\zp} \left[\|
                (1-\pi) / \pi \cdot c(\bm{x}; h_{\w}(\zn),\ell_1) / c(\bm{x}; \bm{v},\ell_2)\right.\\
            &- \left.(1-\pi) / \pi \cdot c(\bm{x}; h_{\w}(\z'^{-}),\ell_1) / c(\bm{x}; \bm{v},\ell_2)
            \|_2
            \right] \\
            \leq& \frac{(1-\pi)|\zp|}{\pi} \mathop{\hat{\Expt}}\limits_{\bm{x}\in\zp} \left[\|
                 c(\bm{x}; h_{\w}(\zn),\ell_1) - c(\bm{x}; h_{\w}(\z'^{-}),\ell_1)\|_2
            \right] \\
            \leq& \frac{(1-\pi)B_\ell|\zp|}{\pi|\zn|}.
        \end{aligned}
    \end{equation}
\end{proof}

\begin{restatable}{lem}{BoundTzero}
\label{lem:boundtzero}
    For two datasets $\S, \S^{(i)}$ that differ by the $i$-th element, let $\{\w_t\}_t, \{\bm{v}_t\}_t$ and $\{\w_t'\}_t, \{\bm{v}'_t\}_t$ be produced by \Algref{alg:main} with $\S$ and $\S^{(i)}$, respectively. If $y_i=1$, then for all $t_0\in[\Np / \np]$ we have
    \begin{equation}
        \begin{aligned}
            \Expt_{\S,A}\left[\bm{m}_{T+1}^{(i)}\right] \leq \frac{t_0\np}{\Np} \left[
            \begin{array}{ccc}
                D(\Omega) & B & 1
            \end{array}
            \right]^{\top} + \Expt_{\S,A}\left[\bm{m}_{T+1}^{(i)}~\big|~i\notin I_{t_0}(A) \right]
        \end{aligned},
    \end{equation}
    where $D(\Omega)$ is the diameter of the hypothesis space, $I_t(A):=\{\bm{i}_1, \cdots, \bm{i}_t\}$ is the set of indices selected by $A$ at the first $t$-th iterations.
    Similarly, when $y_i=-1$, for all $t_0\in[\Nn / \nn]$, we have
    \begin{equation}
        \begin{aligned}
            \Expt_{\S,A}\left[\bm{m}_{T+1}^{(i)}\right] \leq \frac{t_0\nn}{\Nn} \left[
            \begin{array}{ccc}
                D(\Omega) & B & 1
            \end{array}
            \right]^{\top} + \Expt_{\S,A}\left[\bm{m}_{T+1}^{(i)}~\big|~i\notin I_{t_0}(A) \right]
        \end{aligned},
    \end{equation}
\end{restatable}
\begin{proof}
    Here we only need to prove the case $y_i=1$. According to the law of total probability, we have
    \begin{equation}
    \label{eq:temp11}
    \begin{aligned}
        \Expt_{\S,A}\left[\bm{m}_{T+1}^{(i)}\right] \leq& \Expt_{\S,A}\left[\bm{m}_{T+1}^{(i)}~\big|~i\in I_{t_0}(A) \right] \prob\left(i\in I_{t_0}(A)\right) \\
        & ~~~~~~+ \Expt_{\S,A}\left[\bm{m}_{t_0}^{(i)}~\big|~i\notin I_{T+1}(A) \right] \prob\left(i\notin I_{t_0}(A)\right).
    \end{aligned}
    \end{equation}
    Notice that 
    \begin{equation}
        \prob\left(i\in I_{t_0}(A)\right) \leq \sum_{t=1}^{t_0} \prob\left(i \in \bm{i}_t\right) = \frac{t_0\np}{\Np},
    \end{equation}
    the proof is complete by bounding the first term on the right side of \Eqref{eq:temp11}.
\end{proof}

\ColGenConstStep*
\begin{proof}
    According to \Lemref{lem:stab_of_sgd} and \Lemref{lem:boundtzero}, we have 
    \begin{equation}
    \label{eq:temp12}
    \begin{aligned}
        \frac{1}{\Np} \sum_{y_i=1} \Expt_{\S,A}\left[\bm{m}^{(i)}_{T+1}\right] \leq& \frac{t_0\np}{\Np} \left[
            \begin{array}{ccc}
                D(\Omega) & B & 1
            \end{array}
            \right]^{\top} \\
        &~~ + \prod_{t=t_0}^{T} (\bm{I}_{3} + \bm{R}^+_{t}) \cdot  \frac{1}{\Np} \sum_{y_i=1} \Expt_{\S,A}\left[\bm{m}^{(i)}_{t_0}\big|i\notin I_{t_0}(A)\right] \\
        =&\frac{t_0\np}{\Np} \left[
            \begin{array}{ccc}
                D(\Omega) & B & 1
            \end{array}
            \right]^{\top} + \prod_{t=t_0}^{T} (\bm{I}_{3} + \bm{R}^+_{t}) \cdot  \left[
                \begin{array}{ccc}
                    0 & 0 & 1
                \end{array}
                \right]^{\top},
    \end{aligned}
    \end{equation}
    where the last equation is due to the fact that for all $t\leq t_0$, $\bm{z}_{i_t} = \bm{z}'_{i_t}$, leading to $\w_t = \w'_t, \bm{v}_t = \bm{v}'_t$.

    Then we focus on $\prod_{t=t_0}^{T} (\bm{I}_{3} + \bm{R}^+_{t})$. Notice that all elements in this product are nonnegative, thus in the remainder of this proof, given two nonnegative matrics $\bm{a}, \bm{b}$ with the same shape,  we say $\bm{a} \leq \bm{b}$ if all elements in $\bm{a}$ are not larger than that in $\bm{b}$.

    Denote
    \begin{equation*}
        \begin{aligned}
            \bm{M} &= LC_\eta\left[
                \begin{array}{ccc}
                    2 & (1-\beta)& 1 / \Np\\
                    \beta & 0 & 1 / \Np \\
                    0 & 0 & 0
                \end{array}
            \right],
        \end{aligned}
    \end{equation*}
    then by setting $\eta_t \leq C_{\eta} / t$ we have
    \begin{equation}
    \label{eq:temp13}
    \begin{aligned}
        &\prod_{t=t_0}^{T} (\bm{I}_{3} + \bm{R}^+_{t})\\
        \leq& \prod_{t=t_0}^{T} \left[
            \begin{array}{ccc}
                1 + 2L\eta_t & L\eta_t / \Np & L\eta_t \\
                L\beta & 1 & 1 / \Np \\
                0 & 0 & 1
            \end{array}
        \right] \\
        =& \prod_{t=t_0}^{T} \left[
            \begin{array}{ccc}
                1 + 2L\eta_t & L(1-\beta)\eta_t& L\eta_t \\
                L\beta & \Np & 1 / \Np \\
                0 & 0 & 1
            \end{array}
        \right]
        \left[
            \begin{array}{ccc}
                1 & 0 & 0 \\
                0 & 1 / \Np & 0 \\
                0 & 0 & 1
            \end{array}
        \right] \\
        =&
        \left[
            \begin{array}{ccc}
                1 & 0 & 0 \\
                0 & \Np & 0 \\
                0 & 0 & 1
            \end{array}
        \right]
        \prod_{t=t_0}^{T} 
        \left(\left[
            \begin{array}{ccc}
                1 & 0 & 0 \\
                0 & 1 / \Np & 0 \\
                0 & 0 & 1
            \end{array}
        \right]
        \left[
            \begin{array}{ccc}
                1 + 2L\eta_t & L(1-\beta)\eta_t& L\eta_t \\
                L\beta & \Np & 1 / \Np \\
                0 & 0 & 1
            \end{array}
        \right]\right)
        \left[
            \begin{array}{ccc}
                1 & 0 & 0 \\
                0 & 1 / \Np & 0 \\
                0 & 0 & 1
            \end{array}
        \right]
            \\
        =& 
        \left[
            \begin{array}{ccc}
                1 & 0 & 0 \\
                0 & \Np & 0 \\
                0 & 0 & 1
            \end{array}
        \right]
        \prod_{t=t_0}^T \left(\bm{I}_3 + 
        \left[
            \begin{array}{ccc}
                2L\eta_t & L(1-\beta)\eta_t& L\eta_t \\
                L\beta / \Np & 0 & 1 / (\Np)^2 \\
                0 & 0 & 0
            \end{array}
        \right]\right)
        \left[
            \begin{array}{ccc}
                1 & 0 & 0 \\
                0 & 1 / \Np & 0 \\
                0 & 0 & 1
            \end{array}
        \right] \\
        \leq& 
        \left[
            \begin{array}{ccc}
                1 & 0 & 0 \\
                0 & \Np & 0 \\
                0 & 0 & 1
            \end{array}
        \right]
        \prod_{t=t_0}^T \left(\bm{I}_3 + \bm{M} / t\right)
        \left[
            \begin{array}{ccc}
                1 & 0 & 0 \\
                0 & 1 / \Np & 0 \\
                0 & 0 & 1
            \end{array}
        \right],
    \end{aligned}
    \end{equation}
    where the last inequation is due to $\Np \geq T$. Notice that $\bm{M} \geq \bm{0}_{3\times 3}$, thus we have
    \begin{equation}
    \label{eq:temp14}
    \begin{aligned}
        \prod_{t=t_0}^T \left(\bm{I}_3 + \bm{M} / t\right)
        \leq& \prod_{t=t_0}^T \exp(\bm{M} / t) \\
        =& \exp\left(\sum_{t=t_0}^T \bm{M} / t \right) \\
        \leq& \exp \left(\bm{M} \log(T/(t_0-1)) \right) \\
        =& \Lambda ~diag\left(\left(\frac{T}{t_0-1}\right)^{\lambda_1}, \left(\frac{T}{t_0-1}\right)^{\lambda_2}, \left(\frac{T}{t_0-1}\right)^{\lambda_3}\right) \Lambda^{-1}, \\ 
    \end{aligned}
    \end{equation}
    where the last eqution is obtained by diagonalizing $\bm{M}$ into $\Lambda~ diag(\lambda_1, \lambda_2, \lambda_3) \Lambda^{-1}$, $\lambda_{i}$ are the eigenvalues of $\bm{M}$ and the columns of $\Lambda$ are the corresponding eigen vectors. 

    Combining \Eqref{eq:temp13} and \Eqref{eq:temp14}, we have
    \begin{equation}
    \label{eq:temp15}
    \begin{aligned}
        &\left[
            \begin{array}{ccc}
                1 & 0 & 0
            \end{array}
        \right]
        \prod_{t=t_0}^{T} (\bm{I}_{3} + \bm{R}^+_{t})
        \left[
            \begin{array}{ccc}
                0 & 0 & 1
            \end{array}
        \right]^\top \\
        \leq& 
        \left[
            \begin{array}{ccc}
                1 & 0 & 0
            \end{array}
        \right]
        \left[
            \begin{array}{ccc}
                1 & 0 & 0 \\
                0 & \Np & 0 \\
                0 & 0 & 1
            \end{array}
        \right]
        \prod_{t=t_0}^T \left(\bm{I}_3 + \bm{M} / t\right)
        \left[
            \begin{array}{ccc}
                1 & 0 & 0 \\
                0 & 1 / \Np & 0 \\
                0 & 0 & 1
            \end{array}
        \right]
        \left[
            \begin{array}{c}
                0\\
                0\\
                1
            \end{array}
        \right]\\
        \leq&
        \left[
            \begin{array}{ccc}
                1 & 0 & 0
            \end{array}
        \right]
        \prod_{t=t_0}^T \left(\bm{I}_3 + \bm{M} / t\right)
        \left[
            \begin{array}{c}
                0\\
                0\\
                1
            \end{array}
        \right]\\
        \leq&
        \left[
            \begin{array}{ccc}
                1 & 0 & 0
            \end{array}
        \right]
        \Lambda~ diag\left(\left(\frac{T}{t_0-1}\right)^{\lambda_1}, \left(\frac{T}{t_0-1}\right)^{\lambda_2}, \left(\frac{T}{t_0-1}\right)^{\lambda_3}\right) \Lambda^{-1}
        \left[
            \begin{array}{c}
                0\\
                0\\
                1
            \end{array}
        \right].
    \end{aligned}
    \end{equation}
    Specifically, we have $\lambda_{1,2} = LC_\eta(1 \pm \sqrt{1+\beta-\beta^2}), \lambda_3 = 0$, and 
    \begin{equation}
        \Lambda = \left[
            \begin{array}{ccc}
                1 & 1 & \frac{1}{\beta \Np} \\
                \frac{\beta(1-\sqrt{1+\beta-\beta^2})}{1-\beta} & \frac{\beta(1+\sqrt{1+\beta-\beta^2})}{1-\beta} & -\frac{2+\beta}{\beta(1-\beta)\Np} \\
                0 & 0 & -1
            \end{array}
        \right],
    \end{equation}
    leading to
    \begin{equation}
        \left[
            \begin{array}{ccc}
                1 & 0 & 0
            \end{array}
        \right] \Lambda = \left[
            \begin{array}{ccc}
                1 & 1 & \frac{1}{\beta \Np} 
            \end{array}
        \right],
    \end{equation}

    \begin{equation}
        \Lambda^{-1}\left[
            \begin{array}{c}
                0 \\ 0 \\ 1
            \end{array}
        \right]^{\top}
        = \left[
            \begin{array}{c}
                \frac{2+2\beta+\beta\sqrt{1+\beta-\beta^2}}{2\beta^2\sqrt{1+\beta-\beta^2} \Np} \\
                \frac{-2-2\beta+\beta\sqrt{1+\beta-\beta^2}}{2\beta^2\sqrt{1+\beta-\beta^2} \Np} \\
                -1
            \end{array}
        \right]^{\top},
    \end{equation}

    \begin{equation}
    \label{eq:temp16}
    \begin{aligned}
        &\left[
            \begin{array}{ccc}
                1 & 0 & 0
            \end{array}
        \right]
        \Lambda~ diag\left(\left(\frac{T}{t_0-1}\right)^{\lambda_1}, \left(\frac{T}{t_0-1}\right)^{\lambda_2}, \left(\frac{T}{t_0-1}\right)^{\lambda_3}\right) \Lambda^{-1} \left[
            \begin{array}{c}
                0\\
                0\\
                1
            \end{array}
        \right] \\
        =& \frac{2+2\beta+\beta\sqrt{1+\beta-\beta^2}}{2\beta^2\sqrt{1+\beta-\beta^2}} \cdot \frac{T^{\lambda_1}}{(t_0-1)^{\lambda_1}\Np} \\
        &~~+ \frac{-2-2\beta+\beta\sqrt{1+\beta-\beta^2}}{2\beta^2\sqrt{1+\beta-\beta^2} \Np} \cdot \frac{T^{\lambda_2}}{(t_0-1)^{\lambda_2}\Np} \\
        &~~- \frac{T^{\lambda_2}}{\beta(t_0-1)^{\lambda_2}\Np} \\
        \leq& \frac{1+2\beta}{\beta^2} \cdot \frac{T^{\lambda_1}}{(t_0-1)^{\lambda_1}\Np} \\
    \end{aligned}
    \end{equation}
    Combining \Eqref{eq:temp12}, \Eqref{eq:temp15}, \Eqref{eq:temp16}, we have 
    \begin{equation}
        \frac{1}{\Np} \sum_{y_i=1} \Expt_{\S,A}\left[\|\w_{T+1} - \w_{T+1}^{(i)}\|_2\right] = \mathcal{O}\left(
            (T / t_0)^{\lambda_1} / \Np
            + \frac{t_0\np}{\Np}
        \right),
    \end{equation}
    by choosing $t_0\asymp T^{\frac{\lambda_1}{\lambda_1 + 1}}\cdot (\np)^{-\frac{1}{\lambda_1 + 1}}$, we have 
    \begin{equation}
        \frac{1}{\Np} \sum_{y_i=1} \Expt_{\S,A}\left[\|\w_{T+1} - \w_{T+1}^{(i)}\|_2\right] = \mathcal{O}\left(
            \frac{T^{\frac{\lambda_1}{\lambda_1+1}}\cdot (\np)^{\frac{\lambda_1}{\lambda_1 + 1}}}{\Np}
        \right).
    \end{equation}

    Similarly, for $y_i = -1$, let 
    \begin{equation*}
        \begin{aligned}
            \bm{M}' &= LC_\eta\left[
                \begin{array}{ccc}
                    2 & (1-\beta)& 1 / \Nn\\
                    \beta & 0 & 1 / \Nn \\
                    0 & 0 & 0
                \end{array}
            \right],
        \end{aligned}
    \end{equation*}
    which shares the same eigenvalues with $\bm{M}$, thus we have
    \begin{equation}
        \frac{1}{\Nn} \sum_{y_i=-1} \Expt_{\S,A}\left[\|\w_{T+1} - \w_{T+1}^{(i)}\|_2\right] = \mathcal{O}\left(
            \frac{T^{\frac{\lambda_1}{\lambda_1+1}}\cdot (\nn)^{\frac{\lambda_1}{\lambda_1 + 1}}}{\Nn}
        \right).
    \end{equation}


\end{proof}

\subsection{Convergence Rate}
\label{app:convergence}
\ConvergenceFDecLR*
\begin{proof}
  First, according to the recurrence of $\bm{v}_{t}$, we have
  \begin{equation}
      \bm{v}_{t+1} = \sum_{\tau=1}^{t} (1-\beta)^{t-\tau} \beta \cdot \phi\left(h_{\w_\tau} (\z_{i_\tau})\right).
  \end{equation}
  For sake of the presentation, denote
  \begin{equation}
    \begin{aligned}
        \bm{u}_{t+1} &= \sum_{\tau=1}^{t} (1-\beta)^{t-\tau} \beta \cdot h_{\w_\tau} (\Sp) \\
        \bm{g}_t &= \Expt_{\S}[\nabla f\left(\w_t;\z_{i_t},\bm{v}_{t+1}\right)]\\
        \delta_1^t &= \nabla F(\w_t) - \Expt_{\S}\left[\nabla f\left(\w_t;\z_{i_t}, h_{\w_t}(\Sp)\right)\right]\\
        \delta_2^t &= \Expt_{\S}\left[\nabla f\left(\w_t;\z_{i_t}, h_{\w_t}(\Sp)\right) - \nabla f\left(\w_t;\z_{i_t},\bm{u}_{t+1}\right)\right]\\
        \delta_3^t &= \Expt_{\S}\left[\nabla f\left(\w_t;\z_{i_t}, \bm{u}_{t+1}\right) - \nabla f\left(\w_t;\z_{i_t},\bm{v}_{t+1}\right)\right].\\
    \end{aligned}
  \end{equation}
  According to the smoothness of $F$, we have 
  \begin{equation}
    \begin{aligned}
      F(\w_{t+1}) \leq& F(\w_{t}) + \langle \nabla F(\w_t),\bm{x}_{t+1} - \bm{x}_t \rangle + \frac{L_w}{2} \|\bm{x}_{t+1} - \bm{x}_t\|_2^2 \\
      =& F(\w_{t}) - \eta_t \langle \nabla F(\w_t), \bm{g}_t \rangle + \frac{L_w\eta_t^2}{2} \| \bm{g}_t\|_2^2 \\
      =& F(\w_{t}) - \eta_t \|\nabla F(\w_t)\|_2^2 + \eta_t \langle \nabla F(\w_t), \delta_1^t + \delta_2^t + \delta_3^t\rangle
      + \frac{L_w\eta_t^2}{2} \|\bm{g}_t\|_2^2.
    \end{aligned}
  \end{equation}
  By taking expectation of $A$ on both sides, we have
  \begin{equation}
    \begin{aligned}
      \Expt_{A}[F(\w_{t+1})] \leq& \Expt_{A}[F(\w_{t})] - \eta_t \Expt_{A}[\|\nabla F(\w_t)\|_2^2] + \frac{L_w\eta_t^2}{2} \Expt_{A}\left[\|\bm{g}_t\|_2^2\right] \\
      &+ \eta_t \Expt_{A}\left[ \langle \nabla F(\w_t), \delta_2^t  + \delta_3^t \rangle \right] \\
      \leq& \Expt_{A}[F(\w_{t})] - \eta_t \Expt_{A}[\|\nabla F(\w_t)\|_2^2] + \frac{L_w\eta_t^2}{2} \Expt_{A}\left[\|\bm{g}_t\|_2^2\right] \\
      &+ \frac{\eta_t}{2} \Expt_{A}\left[ \|\delta_2^t + \delta_3^t\|_2^2 \right] + \frac{\eta_t}{2} \Expt_{A}[\|\nabla F(\w_t)\|_2^2] \\
      \leq& \Expt_{A}[F(\w_{t})] - \frac{\eta_t}{2} \Expt_{A}[\|\nabla F(\w_t)\|_2^2] + \frac{L_wG^2\eta_t^2}{2} \\
      &+ \eta_t \Expt_{A}\left[\|\delta_2^t\|_2^2 + \|\delta_3^t\|_2^2\right].
    \end{aligned}
  \end{equation}
  By using the PL condition and subtracting $F(\w^*)$ from both sides, we get
  \begin{equation}
    \begin{aligned}
      \Expt_{A}[F(\w_{t+1}) - F(\w^*)]
      \leq& \left(1 - \frac{\mu \cdot \eta_t}{2}\right)\Expt_{A}[F(\w_{t}) - F(\w^*)] + \frac{L_wG^2\eta_t^2}{2} \\
      & ~~+ \eta_t \Expt_{A}\left[\|\delta_2^t\|_2^2 + \|\delta_3^t\|_2^2\right]. \\
    \end{aligned}
  \end{equation}
  If we set the learning rate $\eta_t = \frac{2(2t+1)}{\mu(t+1)^2}$, we have $1 - \mu\cdot\eta_t / 2 = t^2 / (t+1)^2$, leading to the follwing result:
  \begin{equation}
    \begin{aligned}
      &(t+1)^2\cdot \Expt_{A}[F(\w_{t+1}) - F(\w^*)] \\
      \leq& T^2\Expt_{A}[F(\w_{t}) - F(\w^*)] + \frac{2(2t+1)^2L_wG^2}{\mu^2(t+1)^2}
      + \frac{2(2t+1)}{\mu} \Expt_{A}\left[\|\delta_2^t\|_2^2 + \|\delta_3^t\|_2^2\right] \\
      \leq& t^2\Expt_{A}[F(\w_{t}) - F(\w^*)] + \frac{8L_wG^2}{\mu^2} + \frac{2(2t+1)}{\mu} \Expt_{A}\left[\|\delta_2^t\|_2^2 + \|\delta_3^t\|_2^2\right].
    \end{aligned}
  \end{equation}
  Applying the above result recursively, we get
  \begin{equation}
  \label{eq:temp18}
  \begin{aligned}
    &(T+1)^2 \cdot \Expt_{A}[F(\w_{t+1}) - F(\w^*)] \\
    \leq& \Expt_{A}[F(\w_{1}) - F(\w^*)] + \frac{8L_wG^2T}{\mu^2} + \sum_{t=2}^T\frac{2(2t+1)}{\mu} \Expt_{A}\left[\|\delta_2^t\|_2^2 + \|\delta_3^t\|_2^2\right] \\
    \leq& \Expt_{A}[F(\w_{1}) - F(\w^*)] + \frac{8L_wG^2T}{\mu^2} + \sum_{t=2}^T\frac{4(t+1)}{\mu} \Expt_{A}\left[\|\delta_2^t\|_2^2 + \|\delta_3^t\|_2^2\right].
  \end{aligned}
  \end{equation}
  Moreover, according to \Lemref{lem:smooth}, we have
  \begin{equation}
  \label{eq:temp19}
  \begin{aligned}
    \sum_{t=2}^T(t+1) \Expt_{A}\left[ \|\delta_2^t\|_2^2 \right]
    \leq& \left(\frac{L_rL_2 C_wG}{\beta\Np}\right)^2 \sum_{t=2}^T \frac{t+1}{t-1} \sum_{\tau=1}^{t-1} \frac{4(2\tau+1)^2}{\mu^2(\tau+1)^4} \\
    \leq& 48 \left(\frac{L_rL_2 C_wG}{\mu\beta\Np}\right)^2 \sum_{t=2}^T \sum_{\tau=1}^{t-1} \frac{1}{(\tau+1)^2} \\
    \leq& 24 \left(\frac{L_rL_2 C_wG}{\mu\beta\Np}\right)^2 (T-1),
  \end{aligned}
  \end{equation}
  and 
  \begin{equation}
  \label{eq:temp20}
  \begin{aligned}
    \sum_{t=2}^T(t+1) \Expt_{A}\left[\|\delta_3^t\|_2^2\right] 
    \leq& \left(\frac{L_r^2}{\Np} + \frac{L_2^2L_r^2\kappa^2}{(\Np)^2}\right) \cdot \frac{(T+4)(T-1)}{2}.
  \end{aligned}
  \end{equation}

  From \Eqref{eq:temp18}, \Eqref{eq:temp19}, \Eqref{eq:temp20}, we have 
  \begin{equation}
    \begin{aligned}
      &\Expt_{A}[F(\w_{t+1}) - F(\w^*)] \\
      \leq& \frac{\Expt_{A}[F(\w_{1}) - F(\w^*)]}{(T+1)^2} + \frac{8L_wG^2\mu + 96(L_rL_2 C_wG)^2 / \beta^2}{\mu^3(T+1)} + \frac{4L_r^2 + 4L_2^2L_r^2\kappa^2}{\mu\Np}.
  \end{aligned}
  \end{equation}

\end{proof}

\begin{lem}
\label{lem:smooth}
    Denote $r(s) = \sigma\left((1-\pi)/ \pi \cdot \mathop{\hat{\Expt}}_{\bm{x}\sim\zn}[\ell_1(h_{\w}(\bm{x}^+) - h_{\w}(\bm{x}))] / s\right)$. Assume $\nabla r(s)$ is a $L_r$-Lipschitz smooth function with regard to $s$, and $\ell_2$ is $L_2$-Lipschitz smooth, then the following facts hold:
    \begin{itemize}
        \item[(a)] $\|\nabla f(\w;\z, \bm{v}_1) - \nabla f(\w;\z, \bm{v}_2)\|_2 \leq \frac{L_rL_2}{\Np} \|\bm{v}_1 - \bm{v}_2\|_2$.
        \item[(b)] $\left\|\delta_2^t\right\|^2_2 \leq \left(\frac{L_rL_2 C_wG}{\beta\Np}\right)^2  \cdot \sum_{\tau=1}^{t-1} \eta_\tau^2 / (t-1)$.
        \item[(c)] Assume $\Expt_{\S}\left[\|h_{\w}(\S) - \phi(h_{\w}(\z))\|_2\right]$ is upper bounded by $\kappa$, then we have $\Expt_{A}\left[\left\|\delta_3^t\right\|^2_2\right] \leq \frac{L_r^2}{\Np} + \frac{L_2^2L_r^2\kappa^2}{(\Np)^2}$.
    \end{itemize}
  \end{lem}
  \begin{proof}
    First, we have $\nabla f(\w;\z, \bm{v}) = \mathop{\hat{\Expt}}\limits_{\z, c\sim h_{\w}(\zp)}\left[r\left(\hat{\Expt}_{v\sim \bm{v}}[\ell_2(h_{\w}(\bm{x}) - v)]\right)\right]$, according to the Lipschitz smoothness of $\nabla r$, we have
    \begin{equation}
    \label{eq:temp7}
        \begin{aligned}
            &\|\nabla f(\w;\z, \bm{v}_1) - \nabla f(\w;\z, \bm{v}_1)\|_2 \\ 
            =& \left\|\mathop{\hat{\Expt}}\limits_{\z, c\sim h_{\w}(\zp)}\left[\nabla r\left(\hat{\Expt}_{v\sim \bm{v}_1}[\ell_2(h_{\w}(\bm{x}) - v)]\right) - \nabla r\left(\hat{\Expt}_{v'\sim \bm{v}_2}[\ell_2(h_{\w}(\bm{x}) - v')]\right)\right] \right\|_2 \\
            \leq& \mathop{\hat{\Expt}}\limits_{\z, c\sim h_{\w}(\zp)}\left[\left\|\nabla r\left(\hat{\Expt}_{v\sim \bm{v}_1}[\ell_2(h_{\w}(\bm{x}) - v)]\right) - \nabla r\left(\hat{\Expt}_{v'\sim \bm{v}_2}[\ell_2(h_{\w}(\bm{x}) - v')]\right)\right\|_2\right] \\
            \leq& L_r \cdot \mathop{\hat{\Expt}}\limits_{\z, c\sim h_{\w}(\zp)}\left[\left|\hat{\Expt}_{v\sim \bm{v}_1}[\ell_2(h_{\w}(\bm{x}) - v)] - \hat{\Expt}_{v'\sim \bm{v}_2}[\ell_2(h_{\w}(\bm{x}) - v')]\right| \right].\\ 
        \end{aligned}
    \end{equation}
    Denote the $i$-th element of $\bm{v}$ as $\bm{v}^{(i)}$, for any $\bm{x}$ we have
    \begin{equation}
    \label{eq:temp8}
        \begin{aligned}
            &\left|\hat{\Expt}_{v\sim \bm{v}_1}[\ell_2(h_{\w}(\bm{x}) - v)] - \hat{\Expt}_{v'\sim \bm{v}_2}[\ell_2(h_{\w}(\bm{x}) - v')]\right|\\
            =& \left|\frac{1}{\Np}\sum_{i=1}^{\Np}\left(\ell_2(h_{\w}(\bm{x}) - \bm{v}_1^{(i)}) - \ell_2(h_{\w}(\bm{x}) - \bm{v}_2^{(i)})\right)\right|\\
            =& \frac{1}{\Np}\sum_{i=1}^{\Np}\left|\ell_2(h_{\w}(\bm{x}) - \bm{v}_1^{(i)}) - \ell_2(h_{\w}(\bm{x}) - \bm{v}_2^{(i)})\right|\\
            \leq& \frac{L_2}{\Np} \cdot \sum_{i=1}^{\Np}\left|\bm{v}_1^{(i)} - \bm{v}_2^{(i)}\right| \\
            =& \frac{L_2}{\Np} \cdot \|\bm{v}_1 - \bm{v}_2\|_1 \\
            \leq& \frac{L_2}{\Np} \cdot \|\bm{v}_1 - \bm{v}_2\|_2,
        \end{aligned}
    \end{equation}
    \Eqref{eq:temp7} and \Eqref{eq:temp8} leads to conclusion \textit{(a)} imediately.

    We then prove conclusion \textit{(b)}. According to the above derivation, we have
    \begin{equation}
    \label{eq:temp9}
        \begin{aligned}
            \|\delta_2^t\|_2^2
            \leq& \left(\frac{L_rL_c}{\Np}\right)^2 \cdot \left\|h_{\w_t}(\Sp) - \bm{u}_{t+1}\right\|_2^2, \\
        \end{aligned}
    \end{equation}
    where the right side could be upper bounded by the following recursion
    \begin{equation}
        \begin{aligned}
            &\left\|h_{\w_t}(\Sp) - \bm{u}_{t+1}\right\|_2^2 \\
            =& (1-\beta)^2 \left\|h_{\w_t}(\Sp) - \bm{u}_{t}\right\|_2^2 \\
            \leq& (1-\beta)^2 \left\|h_{\w_t}(\Sp) - h_{\w_{t-1}}(\Sp) + h_{\w_{t-1}}(\Sp) - \bm{u}_{t}\right\|_2^2 \\
            \overset{{\color{red}(*)}}{\leq}& (1-\beta)^2(1+\beta)\left\|h_{\w_{t-1}}(\Sp) - \bm{u}_{t}\right\|_2^2  \\
             & ~~~~~~~~~~~~~~~~~~~~~~~~ + (1-\beta)^2(1+\beta^{-1})\left\|h_{\w_t}(\Sp) - h_{\w_{t-1}}(\Sp)\right\|_2^2 \\
             \leq& (1-\beta) \left\|h_{\w_{t-1}}(\Sp) - \bm{u}_{t}\right\|_2^2 + (1-\beta)\beta^{-1}\left\|h_{\w_t}(\Sp) - h_{\w_{t-1}}(\Sp)\right\|_2^2 \\
             \leq& (1-\beta)\left\|h_{\w_{t-1}}(\Sp) - \bm{u}_{t}\right\|_2^2 + (1-\beta)\beta^{-1}C_w^2 \left\|\w_t - \w_{t-1}\right\|_2^2 \\
             =& (1-\beta)\left\|h_{\w_{t-1}}(\Sp) - \bm{u}_{t}\right\|_2^2 + (1-\beta)\beta^{-1}C_w^2 \eta_{t-1}^2 \left\|\bm{g}_{t-1}\right\|_2^2.
        \end{aligned}
    \end{equation}
    Step {\color{red}(*)} is due to the Cauchy-Schwarz inequality. By iterating the above recursion, we get
    \begin{equation}
    \label{eq:temp10}
        \begin{aligned}
            \left\|h_{\w_t}(\Sp) - \bm{u}_{t+1}\right\|_2^2 
            \leq& C_w^2 \sum_{\tau=1}^{t-1}(1-\beta)^{t-\tau}\beta^{-1} \eta_\tau^2 \|\bm{g}_{\tau}\|_2^2 \\
            \leq& C_w^2G^2 \sum_{\tau=1}^{t-1}(1-\beta)^{t-\tau}\beta^{-1} \eta_\tau^2 \\
            \leq& \frac{C_w^2G^2}{t-1}\sum_{\tau=1}^{t-1}(1-\beta)^{t-\tau}\beta^{-1} \cdot \sum_{\tau=1}^{t-1} \eta_\tau^2 \\
            \leq& \frac{C_w^2G^2}{\beta^2(t-1)} \cdot \sum_{\tau=1}^{t-1} \eta_\tau^2
        \end{aligned}
    \end{equation}
    The proof of \textit{(b)} is complete due to \Eqref{eq:temp9} and \Eqref{eq:temp10}.

    Now we turn to conclusion \textit{(c)}. First of all, denote 
    \begin{equation}
        \Delta(\S) = \hat{\Expt}_{v\sim h_{\w}(\S)}[\ell_2(h_{\w}(\bm{x}) - v)] - \hat{\Expt}_{v'\sim \phi(h_{\w}(\z))}[\ell_2(h_{\w}(\bm{x}) - v')], 
    \end{equation}
    
    where $\S$ is a dataset with size $N$ and $\z$ is a randomly sampled subset of $\S$, we will show that for any $\w \in \Omega$, $\Expt_{\S}\left[\|\Delta(\S)\|_2^2\right] = \mathcal{O}\left(1 / \Np\right)$. Consider a dataset $\S'$ with at most one example $\bm{x}_0$ differs from $\S$, $\Delta(\S)$ satisifies the bounded differences condition:
    \begin{equation}
    \begin{aligned}
        &\Expt_{A}\left[\sup_{\S,\S'} \left|
            \Delta(\S) - \Delta(\S')
        \right|\right] \\
        =& \sup_{\S,\S'} \left[\left|\Delta(\S) - \Delta(\S')\right|\big | \bm{x}_0 \in \z\right] \cdot \prob\left[\bm{x}_0 \in \z\right] \\
        &+ \sup_{\S,\S'} \left[\left|\Delta(\S) - \Delta(\S')\right|\big | \bm{x}_0 \notin \z\right] \cdot \prob\left[\bm{x}_0 \notin \z\right] \\
        \leq& \frac{1}{\np} \cdot \frac{\np}{\Np} + \frac{1}{\Np} \cdot \frac{\Np - \np}{\Np} \\
        \leq& \frac{2}{\Np}.
    \end{aligned}
    \end{equation}

    According to \Lemref{lem:efron-stein}, we have 
    \begin{equation}
    \label{eq:temp17}
        \begin{aligned}
        \Expt_{A}\left[Var_{\S}(\Delta(\S))\right]
        \leq \frac{1}{4} \cdot \Np \cdot \left(\frac{2}{\Np}\right)^2 = \frac{1}{\Np}.
        \end{aligned}
    \end{equation}
    Furthermore, similar to \Eqref{eq:temp8}, we have
    \begin{equation}
        \Expt_{\S}[\Delta(\S)] \leq \frac{L_2}{\Np}\Expt_{\S}\left[\|h_{\w}(\S) - \phi(h_{\w}(\z))\|_2\right] \leq \frac{L_2\kappa}{\Np},
    \end{equation}
    leading to 
    \begin{equation}
    \label{eq:temp17}
        \begin{aligned}
        \Expt_{A}\left[{\|\delta_3\|^2_2}\right] \leq& L_r \Expt_{\S}\left[|\Delta(\S)|^2_2\right] \\
        \leq& L_r^2 \Expt_{A}\left[Var_{\S}(\Delta(\S))\right] + L_r^2 \Expt_{A}\left[\Expt^2_{\S}[\Delta(\S)]\right] \\
        \leq& \frac{L_r^2}{\Np} + \frac{L_2^2L_r^2\kappa^2}{(\Np)^2}.
        \end{aligned}
    \end{equation}


  \end{proof}

  \begin{lem}[Efron-Stein Inequality \cite{boucheron2013concentration}]
    \label{lem:efron-stein}
    Let $X_1,\cdots,X_n$ be independent random variables and let $Z = f(X_1,\cdots,X_n)$, if $f$ has the bounded differences property with constants $c_1,\cdots,c_n$, then
    \begin{equation}
        Var(Z) \leq \frac{1}{4}\sum_{k=1}^n c_k^2.
    \end{equation}
  \end{lem}

\section{Experiments}
\label{app:experiments}
\subsection{Implementation Details}
\label{app:exp_details}


\textbf{Competitors.} To validate the advantages of the proposed method over the state-of-the art methods in image retrieval, we compare two types of competitors: \textbf{1) Pairwise Losses}, including \textit{Contrastive Loss} \cite{hadsell2006dimensionality}, \textit{Triplet Loss} \cite{hoffer2015deep}, \textit{Multi-Similarity (MS) Loss} \cite{wang2019multi}, \textit{Cross-Batch Memory (XBM)} \cite{wang2020cross}. These methods construct loss functions with image pairs or triplets. \textbf{2) Ranking-Based Losses}, including \textit{SmoothAP} \cite{brown2020smooth}, \textit{FastAP} \cite{cakir2019deep}, \textit{DIR} \cite{revaud2019learning}, \textit{BlackBox} \cite{jiang2020optimizing}, and \textit{Area Under the ROC Curve Loss (AUROC)} \cite{yang2021learning}. These methods directly optimize the ranking-based metrics such as AUPRC or AUROC.
We reimplement all the competitors on the same codebase to ensure that the main experimental setting of competitors is the same as ours, including model structure, data preprocessing and augmentation, learning rate schedule, testing pipeline, etc. The unique hyperparameters of competitors follows the optimal settings of the original paper. Moreover, the optimizer used are slightly different: following previous work, Adam is used to train the competitors, while ours is trained with SGD to ensure the consistence with our theoretical analysis. We also provide the results of ours trained with Adam in \Tbref{tab:ablation}, which shows no significant difference.


\textbf{Environments.} The proposed method and all competitors are implemented with Pytorch 1.8.2 \cite{paszke2019pytorch}. All experiments are conducted on a single NVIDIA RTX 3090 GPU.

\subsection{More Results}
\label{app:exp_result}
\textbf{Quantitative results.} We provide full results with more $k$ for the Recall@k metric in \Tbref{tab:results_full_sop}, \ref{tab:results_full_inat}, \ref{tab:results_full_vehid}. In addition, since the mainstream work on image retrieval does not split a validation but instead trains models with all data except the test data (called \textit{trainval} split here), we further provide evaluation results of some methods trained on the \textit{trainval} split in \Tbref{tab:results_full_sop_trainval}, \ref{tab:results_full_inat_trainval}, \ref{tab:results_full_vehid_trainval}.

\begin{table}[H]
\caption{Quantitative results on SOP. All methods are trained on the \textit{train split}.}
\setlength\tabcolsep{5.6pt}
\centering
\begin{tabular}{l|ccccc}
    \toprule
    Methods & \small{mAUPRC} & R@1 & R@10 & R@100 & R@1000 \\
    \midrule
    Contrastive loss \cite{hadsell2006dimensionality} & 57.73 &	77.60 &	89.31 & 95.54 & 98.65 \\
    Triplet loss \cite{hoffer2015deep} & 58.07 & 78.34 & 90.50 & 96.20 & 98.88 \\
    MS loss \cite{wang2019multi} & 60.10 & 79.64 & 90.38 & 95.93 & 98.72 \\
    XBM \cite{wang2020cross} & 61.29 & 80.66 & 91.08 & 96.04 & 98.55 \\
    SmoothAP \cite{brown2020smooth} & \second{61.65} & \second{81.13} & \second{92.02} & \second{96.69} & \second{98.91} \\
    DIR \cite{revaud2019learning} & 60.74 & 80.52 & 91.35 & 96.46 & 98.80 \\
    FastAP \cite{cakir2019deep} & 57.10 & 77.30 & 89.61 & 95.74 & 98.72 \\
    AUROC \cite{gao2015consistency} & 55.80 & 77.32 & 89.64 & 95.77 & 98.76 \\
    BlackBox \cite{poganvcic2019differentiation} & 59.74 & 79.48 & 90.74 & 96.16 & 98.74 \\
    \midrule
    Ours & \first{62.75} & \first{81.91} & \first{92.50} & \first{96.97} & \first{98.96} \\
    \bottomrule
\end{tabular}
\label{tab:results_full_sop}
\end{table}

\begin{table}[H]
\caption{Quantitative results on SOP. All methods are trained on the \textit{trainval split}.}
\setlength\tabcolsep{5.6pt}
\centering
\begin{tabular}{l|ccccc}
    \toprule
    Methods & \small{mAUPRC} & R@1 & R@10 & R@100 & R@1000 \\
    \midrule
    Triplet loss \cite{hoffer2015deep} & 58.34 & 78.61 & 90.73 & 96.33 & 98.90
    \\
    SmoothAP \cite{brown2020smooth} & \second{61.96} & \second{81.18} & \second{92.10} & \second{96.70} & \second{98.97}
    \\
    DIR \cite{revaud2019learning} & 61.29 & 80.74 & 91.59 & 96.58 & 98.89
    \\
    FastAP \cite{cakir2019deep} & 57.39 & 77.56 & 89.79 & 95.88 & 98.69
    \\
    AUROC \cite{gao2015consistency} & 55.97 & 77.42 & 89.65 & 95.83 & 98.72
    \\
    \midrule
    Ours & \first{63.27} & \first{82.15} & \first{92.69} & \first{97.11} & \first{99.02} \\
    \bottomrule
\end{tabular}
\label{tab:results_full_sop_trainval}
\end{table}

\begin{table}[h]
\caption{Quantitative results on iNaturalist. All methods are trained on the \textit{train split}.}
\setlength\tabcolsep{5.6pt}
\centering
\begin{tabular}{l|ccccc}
    \toprule
    Methods & \small{mAUPRC} & R@1 & R@4 & R@16 & R@32 \\
    \midrule
    Contrastive loss \cite{hadsell2006dimensionality} & 27.99 & 54.19 & 71.12 & 82.77 & 86.94\\
    Triplet loss \cite{hoffer2015deep} & 30.59 & 60.53 & 77.62 & 88.20 & 91.72\\
    MS loss \cite{wang2019multi} & 30.28 & 63.39 & 78.50 & 88.07 & 91.40 \\
    XBM \cite{wang2020cross} & 27.46 & 59.12 & 75.18 & 85.93 & 89.72\\
    SmoothAP \cite{brown2020smooth} & \second{33.92} & \second{66.13} & \second{80.93} & \second{89.71} & \second{92.67}\\
    DIR \cite{revaud2019learning} & 33.51 & 64.86 & 79.79 & 89.07 & 92.20\\
    FastAP \cite{cakir2019deep} & 31.02 & 56.64 & 73.57 & 84.65 & 88.49\\
    AUROC \cite{gao2015consistency} & 27.24 & 60.88 & 77.76 & 88.39 & 91.85\\
    BlackBox \cite{poganvcic2019differentiation} & 29.28 & 56.88 & 74.10 & 85.41 & 89.42\\
    \midrule
    Ours & \first{36.16} & \first{68.22} & \first{82.86} & \first{91.02} & \first{93.71}\\
    \bottomrule
\end{tabular}
\label{tab:results_full_inat}
\end{table}

\begin{table}[h]
\caption{Quantitative results on iNaturalist. All methods are trained on the \textit{trainval split}.}
\setlength\tabcolsep{5.6pt}
\centering
\begin{tabular}{l|ccccc}
    \toprule
    Methods & \small{mAUPRC} & R@1 & R@4 & R@16 & R@32 \\
    \midrule
    Triplet loss \cite{hoffer2015deep} & 31.49 & 61.10 & 77.97 & 88.47 & 91.95
    \\
    SmoothAP \cite{brown2020smooth} & \second{34.89} & \second{66.79} & \second{81.55} & \second{90.13} & \second{92.95}
    \\
    DIR \cite{revaud2019learning} & 34.55 & 65.65 & 80.63 & 89.51 & 92.52
    \\
    FastAP \cite{cakir2019deep} & 32.14 & 57.84 & 74.59 & 85.27 & 88.99
    \\
    AUROC \cite{gao2015consistency} & 26.08 & 58.70 & 76.13 & 87.33 & 91.05
    \\
    \midrule
    Ours & \first{37.31} & \first{69.23} & \first{83.41} & \first{91.29} & \first{93.81}
    \\
    \bottomrule
\end{tabular}
\label{tab:results_full_inat_trainval}
\end{table}

\begin{table*}[h]
\caption{Quantitative results on VehicleID. All methods are trained on the \textit{train split}.}
\setlength\tabcolsep{5.2pt}
\centering
\resizebox{\textwidth}{!}{
\begin{tabular}{l|ccc|ccc|ccc}
    \toprule
    \multirow{2}{*}{Methods} & \multicolumn{3}{c|}{Small} & \multicolumn{3}{c|}{Medium} & \multicolumn{3}{c|}{Large} \\
    \cline{2-10} & \small{mAUPRC} & R@1 & R@5 & \small{mAUPRC} & R@1 & R@5 & \small{mAUPRC} & R@1 & R@5 \\
    \midrule
    Contrastive loss \cite{hadsell2006dimensionality} & 77.74 & 92.08 & 92.08 & 72.33 & 90.15 & 95.40 & 67.26 & 87.46 & 94.60\\
    Triplet loss \cite{hoffer2015deep} & 80.27 & 94.18 & \second{97.54} & 75.58 & 92.10 & 96.20 & 70.99 & 90.09 & 95.54
    \\
    MS loss \cite{wang2019multi} & 78.83 & 93.05 & 97.00 & 74.30 & 91.46 & 96.06 & 69.15 & 88.82 & 95.06
    \\
    XBM \cite{wang2020cross} & 77.45 & \second{94.83} & 96.94 & 73.77 & \second{93.63} & 96.01 & 71.24 & \first{92.78} & 95.83
    \\
    SmoothAP \cite{brown2020smooth} & 80.39 & 94.36 & 97.44 & 76.28 & 92.98 & 96.42 & 72.28 & 91.31 & 96.05
    \\
    DIR \cite{revaud2019learning} & \second{81.17} & 94.64 & 97.43 & \second{76.95} & 92.74 & \second{96.45} & \second{72.72} & 91.38 & \second{96.10}
    \\
    FastAP \cite{cakir2019deep} & 80.67 & 93.67 & 97.10 & 75.73 & 91.58 & 96.29 & 70.82 & 89.42 & 95.38
    \\
    AUROC \cite{gao2015consistency} & 68.27 & 88.31 & 95.26 & 63.41 & 85.57 & 93.42 & 58.12 & 81.73 & 91.92
    \\
    BlackBox \cite{poganvcic2019differentiation} & 80.00 & 93.72 & 97.15 & 75.46 & 91.84 & 96.26 & 70.92 & 90.14 & 95.52
    \\
    \midrule
    Ours & \first{83.11} & \first{95.41} & \first{97.69} & \first{79.00} & \first{93.67} & \first{96.68} & \first{74.95} & \second{92.50} & \first{96.44}
    \\
    \bottomrule
\end{tabular}
}
\label{tab:results_full_vehid}
\end{table*}

\begin{table}[h]
\caption{Quantitative results on VehicleID. All methods are trained on the \textit{trainval split}.}
\setlength\tabcolsep{5.6pt}
\centering
\resizebox{\textwidth}{!}{
\begin{tabular}{l|ccc|ccc|ccc}
    \toprule
    \multirow{2}{*}{Methods} & \multicolumn{3}{c|}{Small} & \multicolumn{3}{c|}{Medium} & \multicolumn{3}{c|}{Large} \\
    \cline{2-10} & \small{mAUPRC} & R@1 & R@5 & \small{mAUPRC} & R@1 & R@5 & \small{mAUPRC} & R@1 & R@5 \\
    \midrule
    Triplet loss \cite{hoffer2015deep} & 80.73 & 94.01 & 97.41 & 75.48 & 92.14 & 96.28 & 70.75 & 89.65 & 95.53
    \\
    SmoothAP \cite{brown2020smooth} & 80.74 & 94.39 & 97.34 & 76.76 & 92.90 & 96.36 & 72.65 & 91.49 & 96.06
    \\
    DIR \cite{revaud2019learning} & \second{81.59} & \second{94.69} & \second{97.43} & \second{77.36} & \second{93.15} & \second{96.46} & \second{73.20} & \second{91.74} & \second{96.14}
    \\
    FastAP \cite{cakir2019deep} & 81.04 & 93.35 & 97.35 & 75.98 & 91.58 & 96.01 & 70.95 & 89.39 & 95.48
    \\
    AUROC \cite{gao2015consistency} & 68.39 & 88.34 & 95.43 & 63.55 & 85.82 & 93.42 & 58.25 & 81.81 & 91.92
    \\
    \midrule
    Ours & \first{83.70} & \first{95.41} & \first{97.72} & \first{79.84} & \first{93.80} & \first{96.77} & \first{75.56} & \first{92.83} & \first{96.48}
    \\
    \bottomrule
\end{tabular}
}

\label{tab:results_full_vehid_trainval}
\end{table}

\textbf{Qualitative results.} We show more mean PR curves in \Fgref{fig:quan_full} and more convergence results in \Fgref{fig:convgence_comp_full}. These qualitative results are consistent with our main conclusions.
\begin{figure}[h]
    \centering
    \includegraphics[scale=0.22]{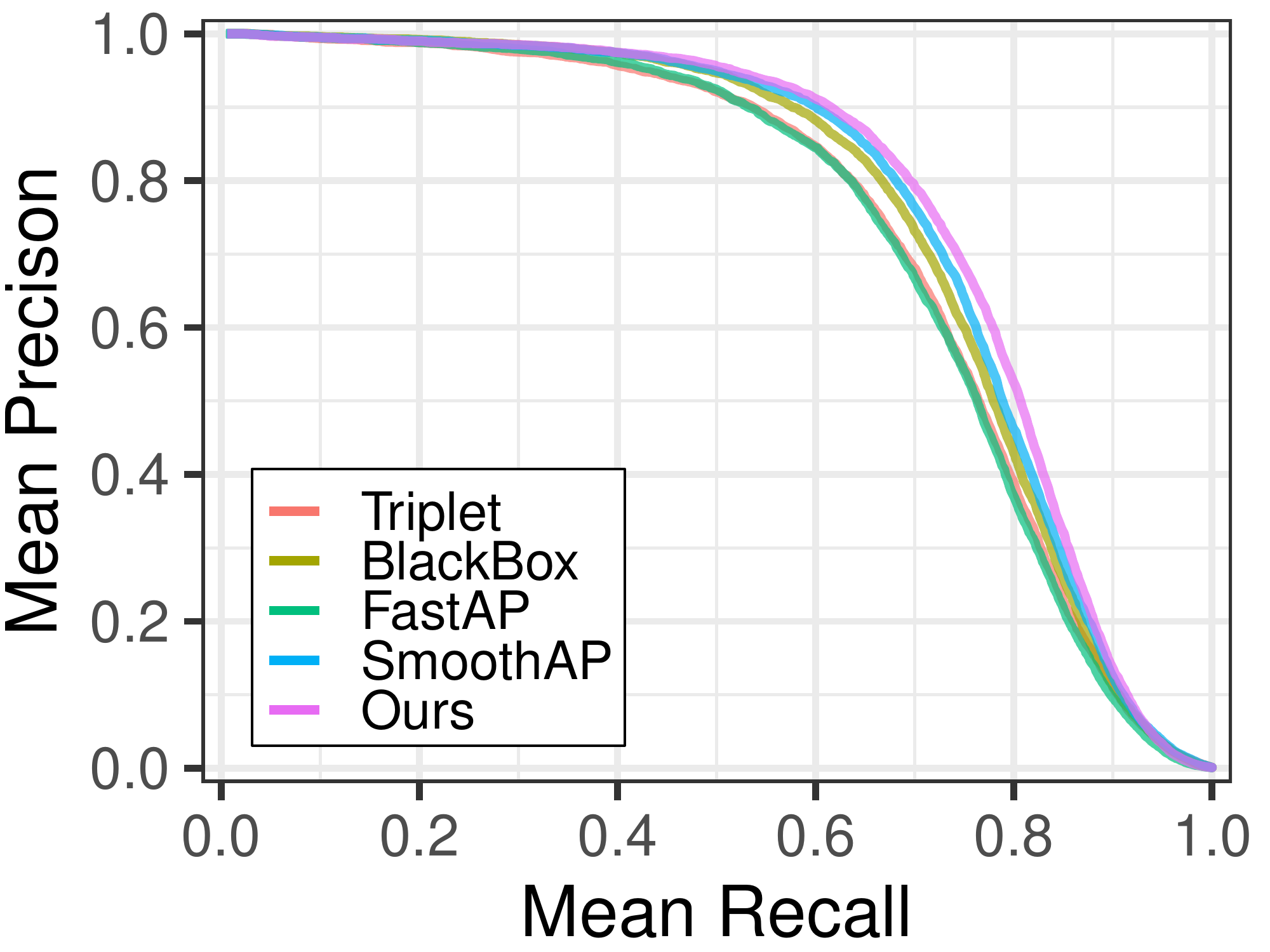}
    \includegraphics[scale=0.22]{pr_iNaturalist.pdf}
    \includegraphics[scale=0.22]{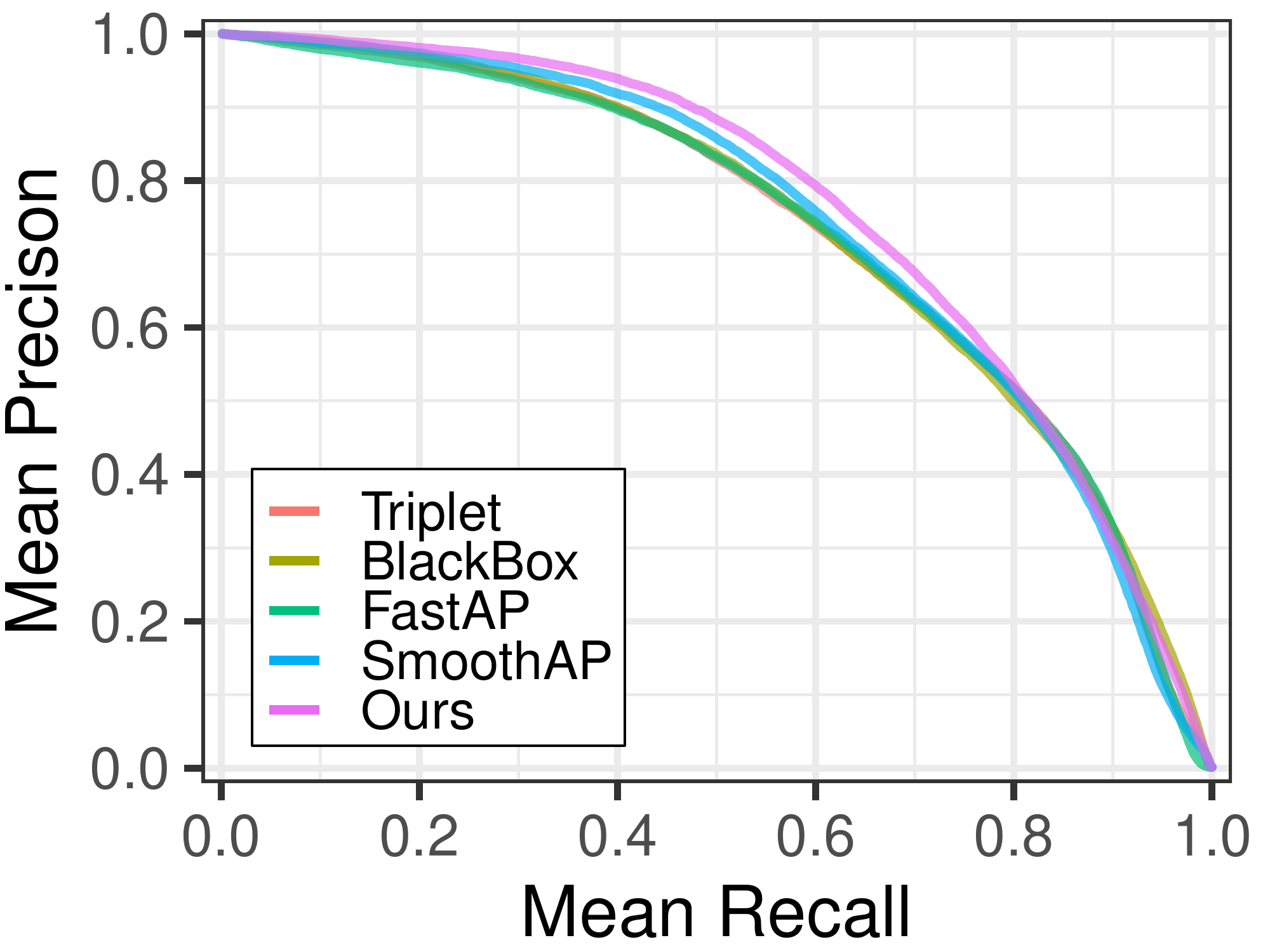}
	\caption{Mean PR curves of different methods on validation sets of SOP (left), iNaturalist (middle), and VehicleID (right).}
	\label{fig:quan_full}
\end{figure}

\begin{figure}[h]
    \centering
    \includegraphics[scale=0.5]{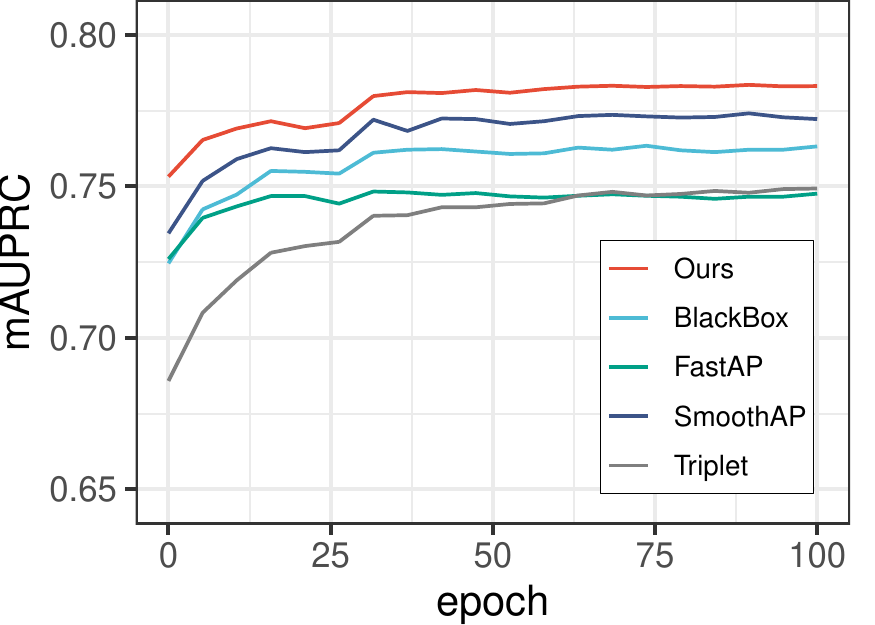}
    \includegraphics[scale=0.5]{convergence_comp.pdf}
    \includegraphics[scale=0.5]{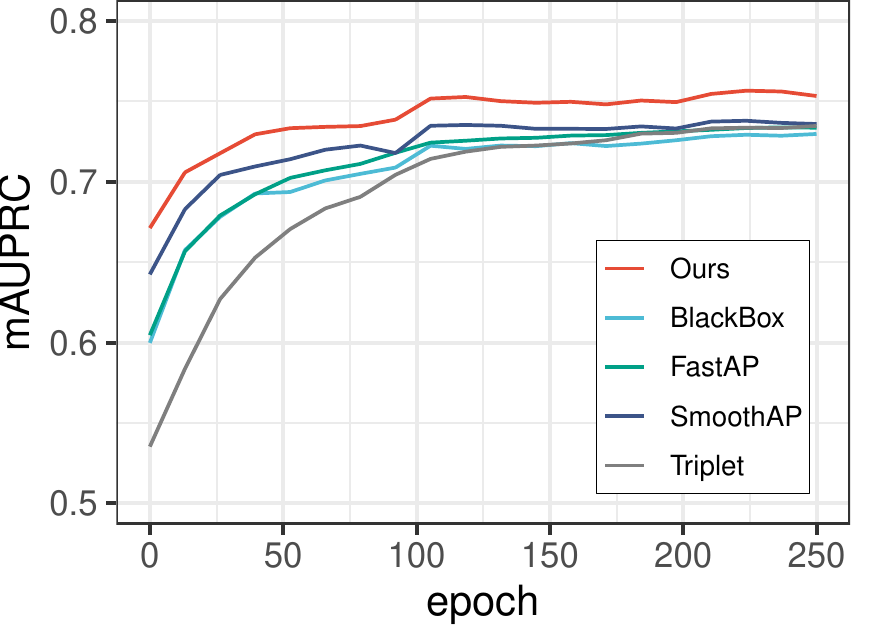}
	\caption{Convergence results of different methods on validation sets of SOP (left), iNaturalist (middle), and VehicleID (right).}
	\label{fig:convgence_comp_full}
\end{figure}

\end{document}